\documentclass[shortAfour,sageh,times]{cls/sagej}
\usepackage{moreverb, url}

\usepackage{amsmath,amsfonts}
\usepackage{algorithmic}
\usepackage{array}
\usepackage{textcomp}
\usepackage{stfloats}
\usepackage{verbatim}
\usepackage{graphicx}
\usepackage[table,usenames,dvipsnames]{xcolor}      
\usepackage[noadjust]{cite}
\usepackage{amsmath,amssymb,amsfonts,amsthm,dsfont,mathtools}
\usepackage{nicematrix}

\usepackage{adjustbox}
\usepackage{graphicx,tabularx,adjustbox,color}
\usepackage{multirow}
\usepackage[font=footnotesize]{caption}
\usepackage[font=footnotesize]{subcaption}

\usepackage[ruled,vlined]{algorithm2e}
\usepackage{stackengine}

\usepackage{dirtytalk}
\allowdisplaybreaks
\renewcommand{\baselinestretch}{1}

\usepackage[colorlinks,bookmarksopen,bookmarksnumbered,citecolor=red,urlcolor=red]{hyperref}

\newcommand{\norm}[1]{\left\lVert#1\right\rVert}

\DeclareMathOperator*{\argmin}{arg\,min}

\newtheorem{theorem}{Theorem}[section]
\newtheorem{proposition}[theorem]{Proposition}
\newtheorem{lemma}[theorem]{Lemma}
\newtheorem{Assumption}[theorem]{Assumption}

\newtheorem{remark}[theorem]{Remark}
\theoremstyle{definition}
\newtheorem{definition}[theorem]{Definition}

\newtheorem*{problem*}{Problem}

\newcommand{\map}[3]{#1:#2 \rightarrow #3}
\newcommand{\longthmtitle}[1]{\mbox{}\textup{\textbf{(#1):}}}

\newcommand\BibTeX{{\rmfamily B\kern-.05em \textsc{i\kern-.025em b}\kern-.08em
T\kern-.1667em\lower.7ex\hbox{E}\kern-.125emX}}

\def\volumeyear{2024}

\newcommand\xqed[1]{%
  \leavevmode\unskip\penalty9999 \hbox{}\nobreak\hfill
  \quad\hbox{#1}}
\newcommand\demo{\xqed{$\bullet$}}

\newcommand{\calB}{{\cal B}}
\newcommand{\calC}{{\cal C}}
\newcommand{\calD}{{\cal D}}
\newcommand{\calE}{{\cal E}}
\newcommand{\calF}{{\cal F}}

\newcommand{\calK}{{\cal K}}
\newcommand{\calL}{{\cal L}}
\newcommand{\calM}{{\cal M}}

\newcommand{\calO}{{\cal O}}
\newcommand{\calP}{{\cal P}}

\newcommand{\calU}{{\cal U}}

\newcommand{\calX}{{\cal X}}

\newcommand{\bff}{\mathbf{f}}
\newcommand{\bfg}{\mathbf{g}}

\newcommand{\bfk}{\mathbf{k}}

\newcommand{\bfn}{\mathbf{n}}

\newcommand{\bfq}{\mathbf{q}}

\newcommand{\bfu}{\mathbf{u}}

\newcommand{\bfx}{\mathbf{x}}
\newcommand{\bfy}{\mathbf{y}}
\newcommand{\bfz}{\mathbf{z}}

\newcommand{\bfeta}{\boldsymbol{\eta}}

\newcommand{\bfxi}{\boldsymbol{\xi}}

\newcommand{\bfF}{\mathbf{F}}

\newcommand{\bfI}{\mathbf{I}}

\newcommand{\bfR}{\mathbf{R}}

\newcommand{\bbE}{\mathbb{E}}

\newcommand{\bbN}{\mathbb{N}}

\newcommand{\bbP}{\mathbb{P}}
\newcommand{\bbQ}{\mathbb{Q}}
\newcommand{\bbR}{\mathbb{R}}

\newcommand{\ubfu}{\underline{\bfu}}

\begin{document}

\runninghead{Long et al.}

\title{Sensor-Based Distributionally Robust Control for Safe Robot Navigation in Dynamic Environments} 




\author{Kehan Long \quad Yinzhuang Yi \quad Zhirui Dai \quad Sylvia Herbert \quad Jorge Cort{\'e}s \quad \newline Nikolay Atanasov}

\affiliation{Contextual Robotics Institute, University of California San Diego, La Jolla, CA 92093, USA.}

\corrauth{Kehan Long, Contextual Robotics Institute, University of California San Diego, La Jolla, CA 92093, USA.}
\email{k3long@ucsd.edu}

\begin{abstract}
We introduce a novel method for mobile robot navigation in dynamic, unknown environments, leveraging onboard sensing and distributionally robust optimization to impose probabilistic safety constraints. Our method introduces a distributionally robust control barrier function (DR-CBF) that directly integrates noisy sensor measurements and state estimates to define safety constraints. This approach is applicable to a wide range of control-affine dynamics, generalizable to robots with complex geometries, and capable of operating at real-time control frequencies. Coupled with a control Lyapunov function (CLF) for path following, the proposed CLF-DR-CBF control synthesis method achieves safe, robust, and efficient navigation in challenging environments. We demonstrate the effectiveness and robustness of our approach for safe autonomous navigation under uncertainty and dynamic obstacles in simulations and real-world experiments with differential-drive robots.
\end{abstract}



\maketitle

\section{Introduction}
\label{sec: intro}

Ensuring real-time, high-frequency robot control with safety guarantees in dynamic and unstructured environments is crucial for the effective deployment of autonomous mobile robots. \citet{potential-field} introduced the seminal artificial potential fields approach to enable collision avoidance for real-time control of mobile robots. This concept has inspired a wealth of research into the joint consideration of path planning and control, including  navigation functions \citep{navigation-function}, dynamic windows \citep{Fox1997TheDW}, and velocity vector fields \citep{de2013navigation}. However, many joint planning and control works typically rely on accurate map representations \citep{voxblox_Oleynikova, herbert2017fastrack, arslan2019sensor, brito2019model, schaefer2021leveraging,kondo2023robust, liu2023radius} updated from onboard sensing to facilitate safe autonomous navigation. In practice, this reliance creates computational bottlenecks, as high-level planning and map updates often occur at lower frequencies and may not adequately account for the uncertainties inherent in high-frequency control \citep{Fox1997TheDW,corke2000elastic,frazzoli2002real,Huang2023_TRO_planning_control}.

Certificate functions have been introduced as powerful tools to assert properties of dynamical systems, such as stability and  safety. Among these, Lyapunov functions \citep{Artstein1983StabilizationWR, SONTAG1989117} guarantee asymptotic stability for dynamical systems, while barrier functions \citep{prajna2004safety} certify forward invariance for desired safe sets. In recent years, control barrier functions (CBFs) have marked a significant advancement in encoding safety constraints for dynamical systems. In conjunction with control Lyapunov functions (CLFs), safe and stable controllers can be synthesized online for control-affine systems via quadratic programming (QP) \citep{cbf}. Due to their computational efficiency and formal guarantees, the CLF-CBF QP framework has become a mainstream approach for synthesizing safe and stable controls for various robot systems \citep{Manavendra_2022_navigate, choi2023constraint, Li_2023_RAL, liu2023realtime}.

However, conventional CBF-based methods commonly assume exact knowledge of robot states, precise system dynamics, and accurate CBF representations. These assumptions rarely hold in practice due to compounded uncertainties from sensor noise, model inaccuracies, and errors in state estimation. Such uncertainties can degrade theoretical safety guarantees and limit practical robustness. Recent approaches have focused on mitigating these issues by estimating CBFs directly from onboard measurements in unknown environments \citep{Long2022RAL, xiao2022differentiable, abdi2023safe, hamdipoor2023safe, Abuaish_2023_radial_NN_ACC}. Nevertheless, these techniques rely heavily on accurate localization and simplified robot geometries, and typically involve computationally intensive environment reconstruction processes from sensory data.

Motivated by these critical limitations, we propose a novel formulation that leverages distributionally robust optimization (DRO) \citep{Esfahani2018DatadrivenDR} to explicitly and efficiently handle multiple realistic uncertainty sources in safe robot navigation. Our distributionally robust control barrier function (DR-CBF) approach directly integrates uncertain sensor measurements and state estimates into safety constraints without requiring precise environmental reconstruction or explicit uncertainty quantification. By using noisy CBF samples within the DRO framework, our formulation significantly enhances the robustness of safe robot navigation in complex, dynamic real-world scenarios. The main \textbf{contributions} of this paper are as follows.

\begin{itemize}

    \item We develop a novel \emph{distributionally robust control barrier function (DR-CBF)} formulation that provides probabilistic safety guarantees by explicitly handling uncertainties in sensor measurements and state estimation;

    \item Coupling the safety constraint with a control Lyapunov function (CLF), we introduce a CLF-DR-CBF quadratic program for synthesizing safe stabilizing controls for general nonlinear control-affine systems; 

    \item Our approach leverages onboard sensor data and state estimates directly as noisy CBF samples, eliminating the need for precise environmental reconstruction and uncertainty quantification;

    \item We validate the safety, efficiency, and robustness of our approach in simulated and real experiments with autonomous differential-drive robots navigating in unknown and dynamic environments, illustrated in Fig.~\ref{fig:jackal_robot}. An open-source implementation of our CLF-DR-CBF controller is available on our project page\footnote{Project page: \href{https://existentialrobotics.org/DRO_Safe_Navigation/}{https://existentialrobotics.org/DRO\_Safe\_Navigation/}}.

\end{itemize}

\begin{figure}[t]
  \centering
  \subcaptionbox{Indoor environment\label{fig:1a}}{\includegraphics[width=0.49\linewidth]{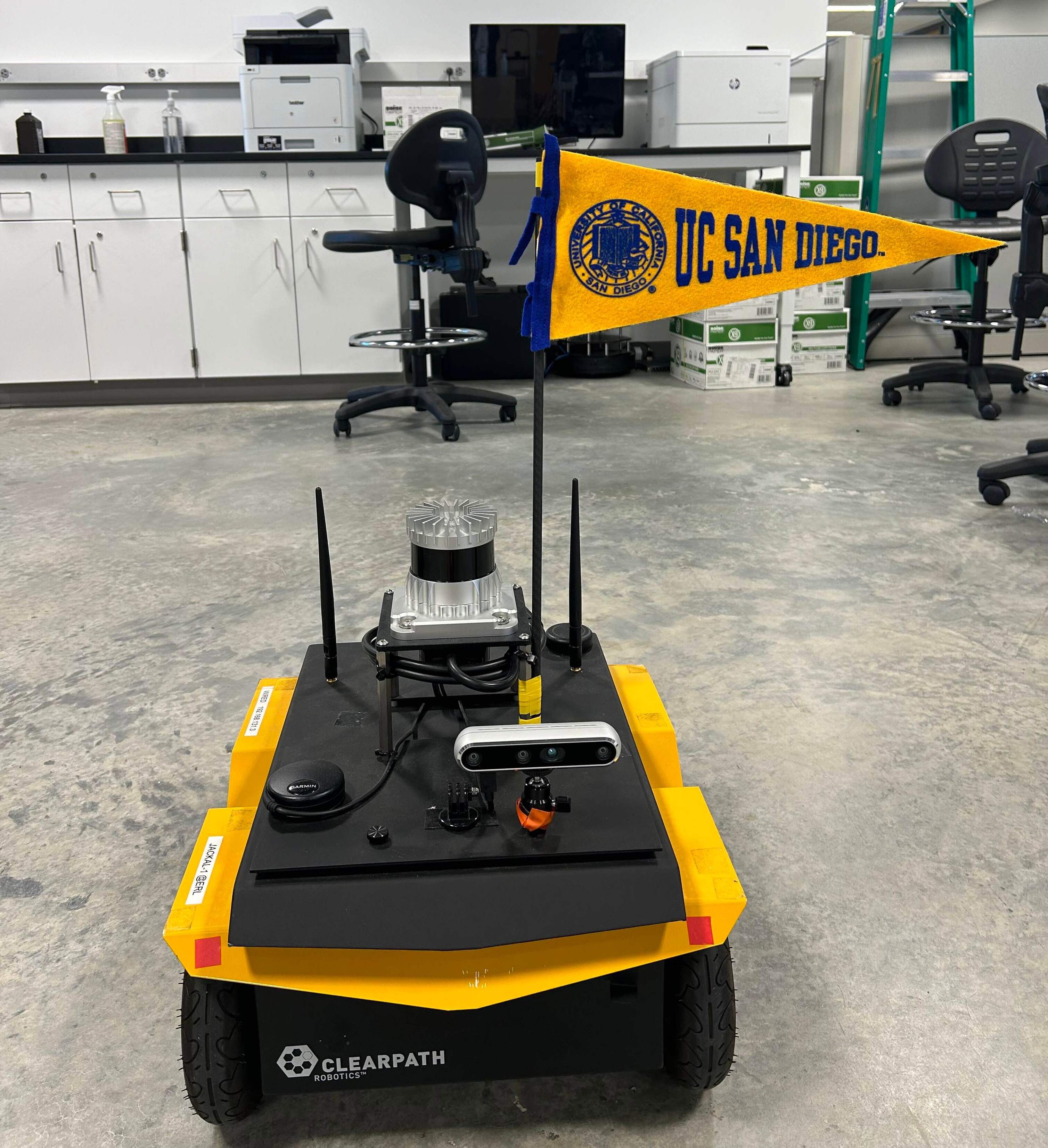}}%
  \hfill%
  \subcaptionbox{Outdoor environment\label{fig:1b}}{\includegraphics[width=0.49\linewidth]{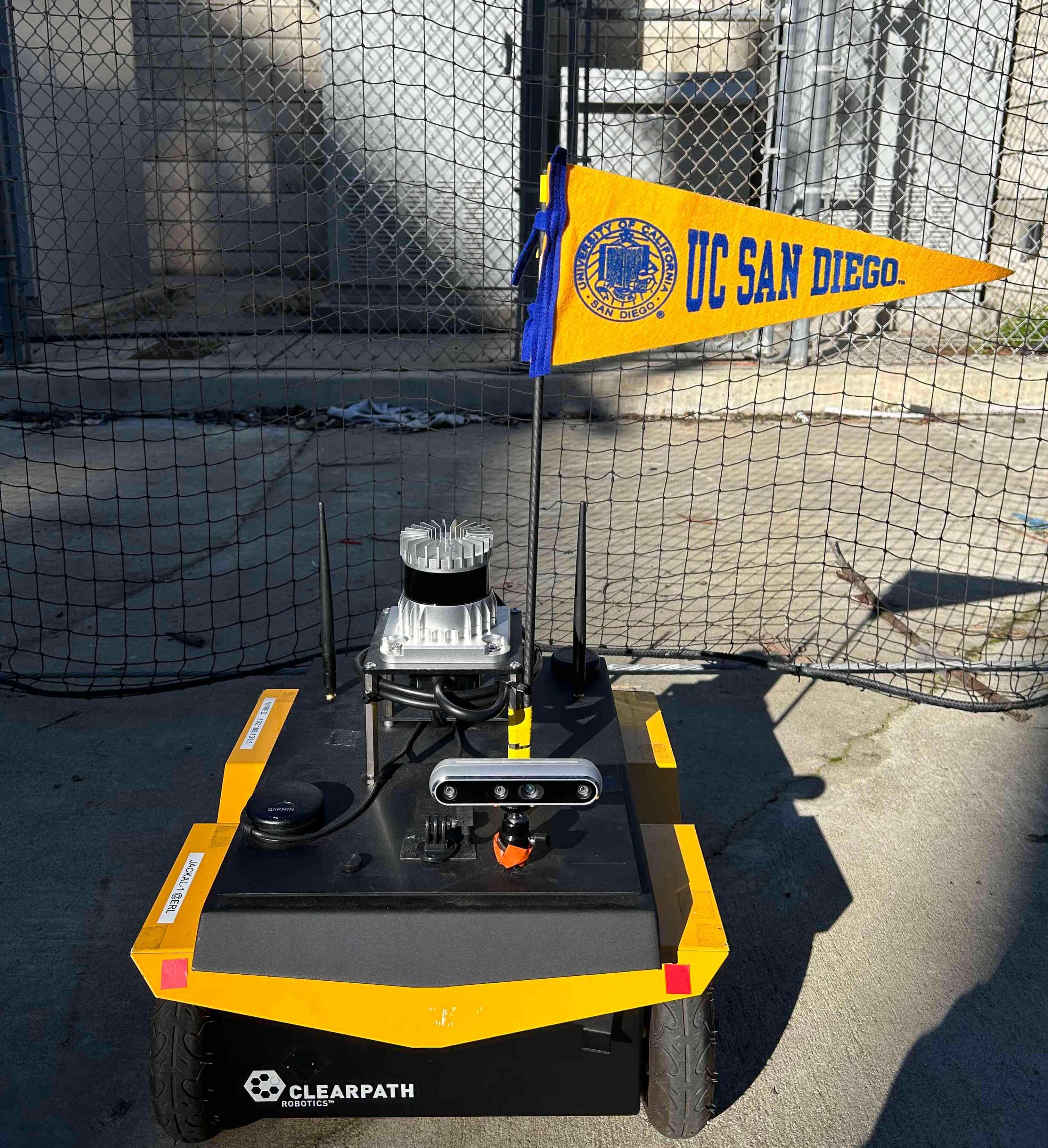}}
  \caption{ClearPath Jackal robot equipped with a LiDAR sensor navigating in unknown environments.}
  \label{fig:jackal_robot}
\end{figure}
\section{Related Work}
\label{sec: related}


This section reviews related work on dynamic obstacle avoidance, distributionally robust optimization, and CLF-CBF techniques for safe stabilizing control.  

\textbf{Dynamic obstacle avoidance.} Robot motion planning algorithms have a rich history, dating back to the 1950s with the introduction of the Dijkstra \citep{DIJKSTRA1959} and $A^*$ \citep{A_star_planning} algorithms for search-based planning. Since then, a substantial amount of research has been dedicated to algorithms for motion planning and trajectory tracking control \citep{lozano_1983_planning, corke2000elastic,frazzoli2002real,Huang2023_TRO_planning_control}. A significant contribution was made by \citet{potential-field}, who introduced artificial potential fields to enable collision avoidance during not only the motion planning stage but also the real-time control of a mobile robot. The formulation was extended to a virtual force field \citep{borenstein1991vector}, facilitating safe navigation in uncertain environments. \citet{navigation-function} developed navigation functions, a particular form of artificial potential functions, which simultaneously ensure collision avoidance and stabilization to a goal configuration. The dynamic window concept was introduced by \citet{Fox1997TheDW} to handle dynamic obstacles by proactively filtering out unsafe control actions. More recently, \citet{bajcsy_2019_cdc_safety_static} leveraged Hamilton-Jacobi reachability for autonomous navigation, providing provable safety guarantees in unknown static environments while enabling real-time updates from sensor measurements and accommodating general nonlinear dynamics and planners. In the domain of safe motion planning for dynamic environments, various methods have been developed to handle uncertainties such as obstacle locations \citep{liu2023radius}, human behavior prediction \citep{schaefer2021leveraging}, and localization errors \citep{summers2018_dr_rrt}. Meanwhile, reinforcement learning techniques have advanced real-time autonomous navigation \citep{pfeiffer2018reinforced, everett_2021_rl, chen2022deep_rl}, often formulating the problem as a partially observable Markov decision process to enable model-free learning of control policies.

Many of these methods require knowledge or estimation of the environment geometry and dynamics, typically in the form of topological \citep{dudek1978robotic} or metric \citep{chatila1985position, borenstein1991vector} map representations. The signed distance function (SDF) has emerged as a particularly valuable tool in this regard \citep{voxblox_Oleynikova, luxin2019fiesta, wu_2023_tro}, offering distance and gradient information for safe control and navigation. However, constructing occupancy or SDF maps at control frequency is challenging in complex and dynamic environments.


\textbf{Distributionally robust optimization.} Distributionally robust optimization (DRO) considers parameter uncertainty in optimization problems, and is particularly effective when only a limited number of uncertainty samples are available. The method compensates for the potential discrepancy between empirical and true uncertainty distributions by utilizing uncertainty descriptors like moment ambiguity sets \citep{Parys2015monent}, Kullback–Leibler ambiguity sets \citep{Jiang2016DatadrivenCC}, and Wasserstein ambiguity sets \citep{Esfahani2018DatadrivenDR, Xie2021OnDR}. The DRO framework has been increasingly utilized for its robust performance guarantees against distributional uncertainty in control \citep{DB-JC-SM:24-tac,PM-KL-NA-JC:23-csl, long2023dro_lf, chriat2023wasserstein} and robotics \citep{ren2022distributionally_ral, coulson2021distributionally, Ryu_2024_icra_mpc_dro}.

\citet{lathrop2021distributionally} developed a Wasserstein safe variant of RRT, offering finite-sample probabilistic safety guarantees.
A DRO approach was introduced by \citet{ren2022distributionally_ral} to enhance policy robustness by iteratively training with adversarial environments generated through a learned generative model. \citet{Long2023_acc_drccp} proposed a DRO formulation for safe stabilizing control under model uncertainty, assuming that nominal safety and stability certificates are provided. \citet{hakobyan2022distributionally} introduced a distributionally robust risk map as a safety specification tool for mobile robots, reformulating optimal control problem over an infinite-dimensional probability distribution space into a tractable semidefinite program. \citet{DB-JC-SM:23-cdc} proposed a distributionally robust coverage control algorithm for a team of robots to optimally deploy in region with an unknown event probability density. A Wasserstein tube MPC is presented in \citet{aolaritei2023wasserstein} for stochastic systems, which utilizes Wasserstein ambiguity sets to construct tubes around nominal trajectories, enhancing robustness and efficiency with limited noise samples available. Compared to existing applications of DRO in robotics and control, our work explicitly incorporates multiple realistic sources of uncertainty, including noisy sensor measurements and state estimation errors, into a unified distributionally robust control barrier function (DR-CBF) formulation. Crucially, we ensure that the resulting safety constraints remain convex, preserving computational tractability and enabling efficient real-time control synthesis. 

%
%



\textbf{Safe stabilizing control.} Quadratic programming that integrates CLF and CBF constraints offers an efficient approach for synthesizing safe stabilizing control inputs for multi-robot navigation \citep{zhang2023neural}, legged locomotion \citep{grandia_2021_legged}, and humanoid operation \citep{khazoom_humanoid_2022}. Despite the efficiency of the CLF-CBF QP in synthesizing safe stabilizing controls, it typically relies on perfect knowledge of system dynamics, state estimates, and barrier function constraints. However, in many robotics applications, various sources of uncertainty can significantly impact performance and reliability. Some recent studies have begun to explore this area, to address uncertainty in system dynamics \citep{dhiman_2023_tac_probabilistic, yousef_2021_tro}, state estimates \citep{Das_2022_cdc_robust, Wang_2023_acc_dob}, and barrier function constraints \citep{Long_learningcbf_ral21, hamdipoor2023safe}. These approaches typically utilize robust and probabilistic models to integrate uncertainty into the QP formulation, leading to convex reformulations that enhance robustness. 

Particularly in safe robot navigation, a robot might need to estimate barrier functions based on its (noisy) observations. \citet{Long_learningcbf_ral21} introduced an incremental online learning approach for estimating the barrier function from LiDAR data and proposed a robust reformulation of the CLF-CBF QP by incorporating estimation errors. \citet{majd_rrt_cbf_iros} integrate time-based rapidly-exploring random trees (RRTs) with CBFs to enable safe navigation in dynamic environments densely populated by pedestrians.  \citet{dawson2022learning} developed an approach to learn observation-space CBFs utilizing distance data and proposed a two-mode hybrid controller to avoid deadlocks in navigation. A reactive planning algorithm was presented in \citet{liu2023realtime} for the safe operation of a bipedal robot with multiple obstacles, utilizing a single differentiable CBF derived from LiDAR point clouds. \citet{abdi2023safe} proposed a method for learning vision-based CBF from RGB-D images with pre-training, enabling safe navigation of autonomous vehicles in unseen environments. By effectively decomposing and predicting the spatial interactions of multiple obstacles, \citet{Yu_sequential_CBF_2023} proposed compositional learning of sequential CBFs, enabling obstacle avoidance in dense dynamic environments. \citet{keyumarsi_LiDAR_CBF} introduced an efficient Gaussian Process-based method for synthesizing CBFs from LiDAR data, showcasing its effectiveness in a turtlebot navigation task. \citet{zhang2024online} introduced an efficient LiDAR-based framework for goal-seeking and exploration of mobile robots in dynamic environments, utilizing minimum bounding ellipses to represent obstacles and Kalman filters to estimate their velocities. 

\textbf{Conformal Prediction.} A related line of research to our work is the application of conformal prediction in robot navigation tasks. 
Conformal prediction \citep{shafer2008tutorial, zhao2024conformal} is a statistical tool for uncertainty quantification that provides valid prediction regions with a user-specified risk tolerance, making it particularly useful for ensuring safety in dynamic environments. Several recent works have explored the integration of conformal prediction into motion planning and control frameworks. \citet{Lindemann_2022_conformal_mpc} used conformal prediction to obtain prediction regions for a model predictive controller. \citet{yang2023safe} employed conformal prediction to quantify state estimation uncertainty and design a robust CBF controller based on the estimated uncertainty. An adaptive conformal prediction algorithm was developed by \citet{dixit2023adaptive} to dynamically quantify prediction uncertainty and plan probabilistically safe paths around dynamic agents.

\section{Preliminaries}
\label{sec: background}
This section introduces our notation and offers a brief review of CLF-CBF QP.

\subsection{Notation}
\label{sec: notation}
The sets of real, non-negative real, and natural numbers are denoted by $\bbR$, $\bbR_{\geq 0}$, and $\bbN$, respectively. For $N \in \bbN$, we write $[N] := \{1,2, \dots N\}$. We denote the distribution and expectation of a random variable $Y$ by $\mathbb{P}$ and $\bbE_{\bbP}(Y)$, respectively. We use $\boldsymbol{0}$ and $\boldsymbol{1}$ to denote the vector with all entries equal to $0$ and $1$, respectively. For a scalar $x$, we define $(x)_+ := \max(x,0)$. We denote by $\bfI_n \in \bbR^{n \times n}$ the identity matrix. 
For a scalar $x$ and $y$, we use $\text{atan2}(y, x)$ to denote the angle between the positive $x$-axis and the point $(x, y)$ in radians.
The interior and boundary of a set $\calC \subset \bbR^n$ are denoted by $\text{Int}(\calC)$ and $\partial \calC$. For a vector $\bfx$, the notation $|\bfx|$ represents its element-wise absolute value, while $\|\bfx \|_1$, $\|\bfx\|$, and $\|\bfx\|_{\infty}$ denote its $L_1$, $L_2$, and $L_{\infty}$ norms, respectively. 
The gradient of a differentiable function $\map{V}{\bbR^n}{\bbR}$ is denoted by $\nabla V$, while its Lie derivative along a vector field $\map{\bff}{\bbR^n}{\bbR^n}$ is $\calL_{\bff} V  = \nabla V^\top \bff$. A continuous function $\alpha: [0,a)\rightarrow [0,\infty )$ is of class $\calK$ if it is strictly increasing and $\alpha(0) = 0$. A continuous function $\alpha: \bbR \rightarrow \mathbb{R}$ is of extended class $\calK_{\infty}$ if it is strictly increasing, $\alpha(0) = 0$, and $\lim_{r \rightarrow \infty} \alpha(r) = \infty$. The special orthogonal group of dimension $p$ is denoted by $\text{SO}(p)$, which is defined as the set of all $p \times p$ orthogonal matrices with determinant equal to 1:
$\text{SO}(p) = \{\bfR \in \bbR^{p \times p} \mid \bfR^\top \bfR = \bfI_p, \det(\bfR) = 1\}$.

\subsection{CLF-CBF Quadratic Program}

Consider a non-linear control-affine system,
\begin{equation}
\label{eq: dynamic}
    \dot{\bfx} = \bff(\bfx) +\bfg(\bfx) \bfu =     [\bff(\bfx) \; \bfg(\bfx)] \begin{bmatrix}
    1 \\
    \bfu
    \end{bmatrix} =: \bfF(\bfx)\ubfu,
\end{equation}
where $\bfx \in \calX \subseteq \mathbb{R}^{n}$ is the  state, $\bfu \in \mathbb{R}^{m}$ is the control input, and $\map{\bff}{\mathbb{R}^{n}}{\mathbb{R}^{n}}$ and $\map{\bfg}{\mathbb{R}^{n}}{\mathbb{R}^{n \times m}}$ are locally Lipschitz continuous functions.  

The notion of a control Lyapunov function (CLF)~\citep{Artstein1983StabilizationWR, SONTAG1989117} plays a key role in certifying the stabilizability of control-affine systems.

\begin{definition}
A continuously differentiable function $\map{V}{\bbR^n}{\bbR}$ is a \emph{control Lyapunov function} (CLF) on $\calX$ for system \eqref{eq: dynamic} if $V(\bfx)>0$, $\forall \bfx \in \calX \setminus \{\boldsymbol{0}\}$, $V(\boldsymbol{0}) = 0$, and
\begin{equation}\label{eq: clf}
    \inf_{\bfu \in \bbR^m} \text{CLC}(\bfx,\bfu) \leq 0, \quad \forall \bfx \in \calX,
\end{equation}
where $\text{CLC}(\bfx,\bfu) := \mathcal{L}_{\bff} V(\bfx) + \mathcal{L}_{\bfg} V(\bfx)\bfu + \alpha_V( V(\bfx))$
is the \emph{control Lyapunov constraint} (CLC) defined for some class $\calK$ function $\alpha_V$.
\end{definition}
%
%


To facilitate safe control synthesis, we consider a time-varying set $\calC(t)$ defined as the zero superlevel set of a continuously differentiable function $h: \calX \times \bbR_{\geq 0} \rightarrow \bbR$:
\begin{equation}
\label{eq: safe_set}  
    \calC(t) := \{\bfx \in \calX: h(\bfx, t) \geq 0 \}.
\end{equation}
Safety of the system \eqref{eq: dynamic} can then be ensured by keeping the state $\bfx$ within the safe set $\calC(t)$.

\begin{definition}
\label{def: tv_cbf}
A continuously differentiable function $h: \mathbb{R}^n \times \bbR_{\geq 0} \rightarrow {\mathbb{R}}$ is a \emph{time-varying control barrier function} (TV-CBF) on $\mathcal{X} \subseteq \mathbb{R}^n$ for \eqref{eq: dynamic} if there exists an extended class $\mathcal{K}_{\infty}$ function $\alpha_h$ with:
\begin{equation}\label{eq:tv_cbf}
    \sup_{\bfu\in \mathcal{U}} \text{CBC}(\bfx,\bfu, t) \geq 0, \quad \forall \; (\bfx,t) \in \calX \times \bbR_{\geq 0},
\end{equation}
where the \emph{control barrier constraint (CBC)} is:
\begin{align}
\label{eq:tvcbc_define}
    &\text{CBC}(\bfx,\bfu, t) := \dot{h}(\bfx, t) + \alpha_h(h(\bfx,t)) \\
    & = \mathcal{L}_{\bff} h(\bfx, t) + \mathcal{L}_{\bfg} h(\bfx, t)\bfu + \frac{\partial h(\bfx,t)}{\partial t} + \alpha_h(h(\bfx,t)). \notag
\end{align}
\end{definition}

Definition~\ref{def: tv_cbf} allows us to consider the set of control values $K_{\text{CBF}}(\bfx, t) := \left \{  \bfu \in \bbR^m: \text{CBC}(\bfx,\bfu, t) \geq 0 \right \}$ that render the set $\calC(t)$
forward invariant. 

\begin{definition}\label{def:tvfi}
Let $t_0$ be a fixed initial time. A time-varying set $\calC(t)$ is said to be \emph{forward invariant} under control law $\bfu : [t_0, \infty) \rightarrow \mathbb{R}^m$ if, for any initial state $\bfx_0 \in \calC(t_0)$, there exists a unique maximal solution $\bfx : [t_0, t_1) \rightarrow \mathbb{R}^n$ to the system dynamics in~\eqref{eq: dynamic} with $\bfx(t_0) = \bfx_0$, such that $\bfx(t) \in \calC(t)$ for all $t \in [t_0, t_1)$.
\end{definition}
%
%

Suppose we are given a baseline feedback controller $\bfu = \bfk(\bfx)$ and we aim to ensure the safety and stability of the control-affine system \eqref{eq: dynamic}. By observing that both the stability and safety constraints in \eqref{eq: clf}, \eqref{eq:tv_cbf} are affine in the control input $\bfu$, a quadratic program \citep{ames2016control} can be formulated to synthesize a safe stabilizing controller:
\begin{equation}
\begin{aligned}
\label{eq: clf_cbf_qp}
    (\bfu(\bfx, t), \delta) \in &\argmin_{\bfu \in \bbR^m,\delta \in \bbR} \| \bfu - \bfk(\bfx)\|^2 + \lambda \delta^2,  \\
    \mathrm{s.t.} \, \,  &\text{CLC}(\bfx,\bfu) \leq \delta,  \text{CBC}(\bfx,\bfu, t) \geq 0,
\end{aligned}
\end{equation}
where $\delta$ denotes a slack variable that relaxes the CLF constraint to ensure feasibility of the QP, controlled by the scaling factor $\lambda > 0$. As discussed in~\citet{ames2016control,PM-AA-JC:23-scl}, if the CBF $h$ has relative degree of $1$ with respect to the system dynamics in \eqref{eq: dynamic}, the controller in \eqref{eq: clf_cbf_qp} is Lipschitz continuous in $\bfx$ and piecewise continuous in $t$. This guarantees unique solutions for the closed-loop system, ensuring the safe set $\calC(t)$ remains forward invariant. However, a few practical limitations should be noted: (a) the use of the slack variable $\delta$ relaxes the strict satisfaction of the CLF stability constraint, creating a trade-off between safety and stability that depends on the design of the safe set; (b) in numerical implementations, the CBF-QP must be solved at a fixed rate, whereas the theoretical guarantees are valid for executions in continuous time; and (c) while the CBF-QP is generally feasible, it may become infeasible if input constraints limit the available control actions. Researchers have proposed methods to address these challenges, including employing Input-Constrained Control Barrier Functions (ICCBFs) \citep{Agrawal_2021_cdc}, using neural network-based CBFs to account for input saturation \citep{liu2023safe}, and tuning of the class $\mathcal{K}_\infty$ functions \citep{hardik_2022_rate_cbf}.
\section{Problem Formulation}\label{sec:problem}

We consider a mobile robot that relies on noisy range measurements to traverse an unknown dynamic environment towards a desired goal. The robot's motion is governed by control-affine dynamics as in \eqref{eq: dynamic}. Let $\phi: \calX \rightarrow \mathbb{R}^p$ be the mapping that projects the robot's state $\bfx$ to its position vector $\phi(\bfx) \in \mathbb{R}^p$. We denote the robot orientation as $\bfR(\bfx) \in \text{SO}(p)$.

Let \( \mathcal{B}_0 \subset \mathbb{R}^p \) represent the robot's shape in its local coordinate frame, i.e., when the robot is at the origin with no rotation or velocity. To describe the shape of the robot in a general state \( \mathbf{x} \), we introduce the following linear transformation:
\begin{equation}
\label{eq: robot_shape_transformation}
\mathcal{B}(\mathbf{x}) = A(\mathbf{x}) \mathcal{B}_0 + \mathbf{b}(\mathbf{x}),
\end{equation}
where \( A: \mathbb{R}^n \to \mathbb{R}^{p \times p} \) maps the robot’s state \( \mathbf{x} \) to an invertible transformation matrix (e.g., rotation), and \( \mathbf{b}: \mathbb{R}^n \to \mathbb{R}^p \) defines the translation, both being continuously differentiable.

%
%

We use a \emph{signed distance function} (SDF) to describe the robot'a shape. Let \( \mathcal{S} \subset \mathbb{R}^p \) be a closed set. The signed distance function \( d: \mathbb{R}^p \rightarrow \mathbb{R} \) computes the signed distance from a point \( \mathbf{q} \in \mathbb{R}^p \) to the set boundary \( \partial \mathcal{S} \):
\begin{equation}
\label{eq:general_sdf}
    d_{\mathcal{S}}(\mathbf{q}) =
    \begin{cases}
        -\min\limits_{\mathbf{q}^* \in \partial \mathcal{S}} \| \mathbf{q} - \mathbf{q}^* \|, & \text{if } \mathbf{q} \in \mathcal{S}, \\
        \phantom{-}\min\limits_{\mathbf{q}^* \in \partial \mathcal{S}} \| \mathbf{q} - \mathbf{q}^* \|, & \text{if } \mathbf{q} \notin \mathcal{S}.
    \end{cases}
\end{equation}
This SDF provides a measure of how far a point \( \mathbf{q} \) is from the set \( \mathcal{S} \). The function takes negative values for points inside \( \mathcal{S} \), positive values for points outside, and is zero on the boundary \( \partial \mathcal{S} \). 

When the set \( \mathcal{S} \) represents the robot's body \( \mathcal{B}(\mathbf{x}) \), we call its signed distance function the \emph{robot SDF}. To express the SDF of the robot's body in terms of its reference shape \( \mathcal{B}_0 \), we apply the linear transformation in \eqref{eq: robot_shape_transformation}:
\begin{equation}
\label{eq: robot_sdf_transformation}
    d(\mathcal{B}(\mathbf{x}), \mathbf{q}) = d_{\mathcal{B}_0}\left( A(\mathbf{x})^{-1} (\mathbf{q} - \mathbf{b}(\mathbf{x})) \right),
\end{equation}
where \( d_{\mathcal{B}_0}(\cdot) \) is the SDF of the reference shape \( \mathcal{B}_0 \). Closed-form robot SDF expressions are available for simple geometries (e.g., circles, spheres) and can be approximated using feed-forward neural networks for more complex geometries \citep{deepsdf, koptev_neural_jsdf_2022}.

The robot is equipped with a range sensor (e.g., LiDAR) mounted at a fixed position and orientation in the robot's body frame.  The sensor emits multiple rays, each corresponding to a direction in the sensor's field of view, and generates distance measurements along these rays. The measurements are subject to additive noise, modeled as $\bfeta(\phi(\bfx), \bfR(\bfx)) = \bar{\bfeta}(\phi(\bfx), \bfR(\bfx)) + \bfn$, where $\bar{\bfeta}(\phi(\bfx), \bfR(\bfx))$ represents the true distances, and $\bfn$ is the noise vector, assumed to be bounded and independent for each ray. The resulting noisy measurements are denoted as $\bfeta(\phi(\bfx), \bfR(\bfx)) = [\eta_1, \ldots, \eta_K]^\top \in [\eta_{\text{min}}, \eta_{\text{max}}]^K$, where $K$ denotes the number of rays per sensor observation, and $\eta_{\text{min}}, \eta_{\text{max}}$ denote the sensor's minimum and maximum range, respectively.

The environment includes both static and dynamic obstacles. Let the obstacle set at the initial time \( t = 0 \) be represented as \( \mathcal{O}_0 \subset \mathbb{R}^p \). At any time \( t \), the obstacle space \( \mathcal{O}(t) \subset \mathbb{R}^p \) is obtained by applying a smooth transformation to the initial obstacle set:
\begin{equation}
\mathcal{O}(t) = A_{\mathcal{O}}(t) \mathcal{O}_0 + \mathbf{b}_{\mathcal{O}}(t),
\label{eq: obstacle_set_transformation}
\end{equation}
where \( A_{\mathcal{O}}: \mathbb{R} \to \mathbb{R}^{p \times p} \) is a time-dependent transformation function mapping time \( t \) to an unknown transformation matrix, and \( \mathbf{b}_{\mathcal{O}}: \mathbb{R} \to \mathbb{R}^p \) is a time-dependent translation function. To ensure smooth changes in the environment, we assume that both \( A_{\mathcal{O}} \) and \( \mathbf{b}_{\mathcal{O}} \) are continuously differentiable, preventing obstacles from experiencing sudden jumps, expansions, or contractions. Additionally, obstacles are assumed to have finite thickness, bounded velocities and accelerations, and smooth boundaries to ensure reliable detection.

At time $t$, the boundary of the obstacle space \( \partial \mathcal{O}(t) \) is estimated from the noisy sensor measurements \( \bfeta \). The free space, where the robot can operate safely, is the open set \( \mathcal{F}(t) = \mathbb{R}^p \setminus \mathcal{O}(t) \).

\begin{problem*}
\label{problem:navigation}
Consider a mobile robot with dynamics given by \eqref{eq: dynamic}, equipped with a range sensor, operating in an unknown dynamic environment. Design a control policy that drives the robot safely and efficiently to a desired goal position $\mathbf{q}_G \in \mathbb{R}^p$. The policy must ensure that the robot's body remains within the free space, i.e., $\calB(\bfx(t)) \subset \calF(t)$ for all $t \geq 0$, while accounting for uncertainties in sensing and state estimation. 
\end{problem*}

\section{System Overview}
\label{sec: overview}

\begin{figure}[t]
  \centering
  \includegraphics[width=0.98\linewidth]{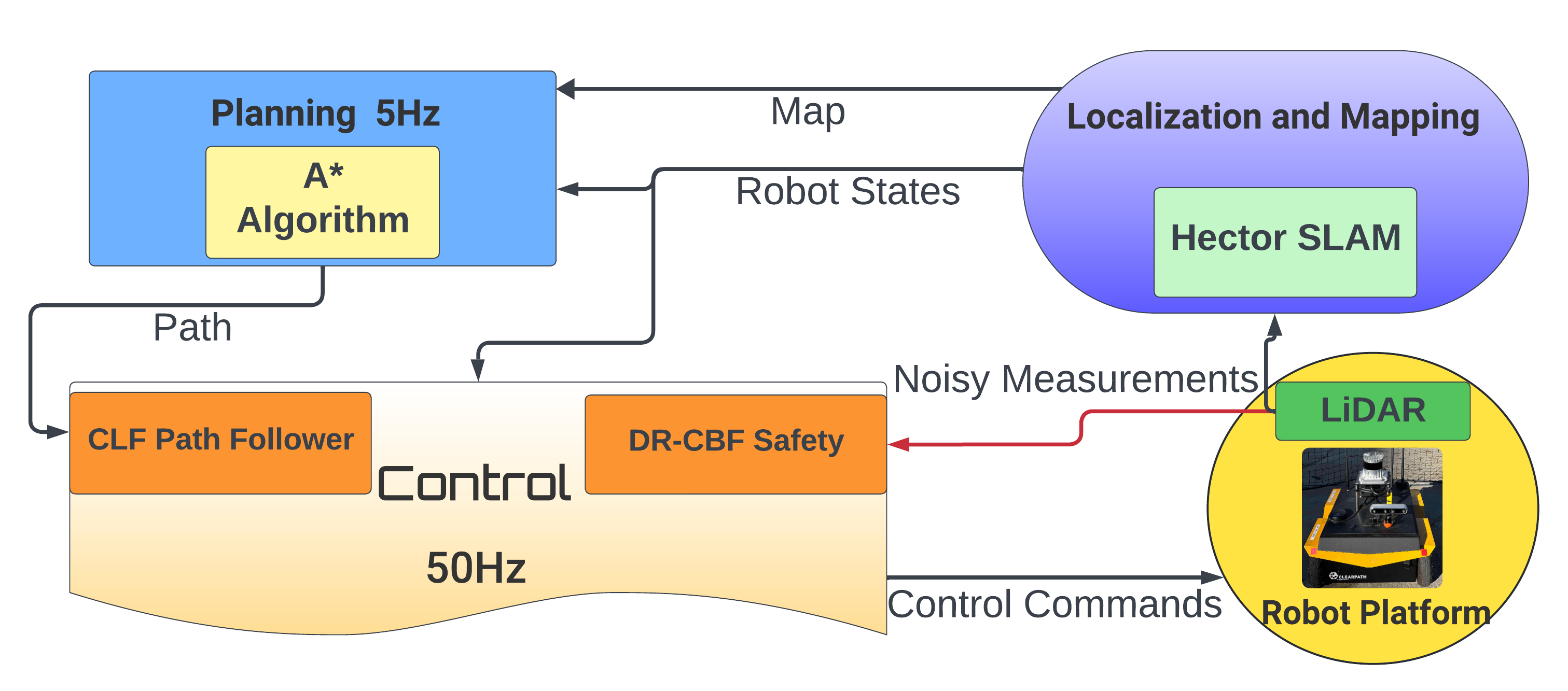}\\
  \caption{Overview of our approach for safe robot navigation in unknown dynamic environments. The system consists of three main components: (1) localization and mapping, (2) path planning, and (3) control. The contribution of this work lies in the control component, where a novel distributionally robust control barrier function is used to ensure safety in real-time, directly utilizing sensor data, and a control Lyapunov function is used to navigate cluttered and dynamic environments.}\label{fig: system_overview}
\end{figure}

This section describes the methods we use for localization and mapping, path planning, and control to enable safe robot navigation, Fig.~\ref{fig: system_overview} presents an overview of the robot autonomy components. Our contribution is the control method, which includes a distributionally robust time-varying control barrier function to guarantee safety in dynamic environments, presented in Sec.~\ref{sec: dro_cbf_navigation}, and a control Lyapunov function for stable path-following control introduced in Sec.~\ref{sec: clf_based_path_following}. 


\textbf{Localization and Mapping.} The robot is equipped with a LiDAR scanner and uses the Hector SLAM algorithm \citep{hector_slam_2011} to estimate its pose and build an occupancy map of the environment. We emphasize that the efficiency and accuracy of the constructed occupancy map may not be sufficient to ensure safe navigation in a dynamic environment. We use the map for high-level path planning and employ a low-level controller to guarantee safe tracking using CLF and CBF techniques.

\textbf{Path Planning.} We use the $A^*$ planning algorithm \citep{A_star_planning} to generate a path $\gamma: [0,1] \mapsto \mathbb{R}^p$ from the robot's current position $\phi(\bfx) \in \mathbb{R}^p$ to the goal $\bfq_G \in \mathbb{R}^p$. The path $\gamma(s)$ is parametrized by a scalar $s \in [0,1]$ such that $\gamma(0) = \phi(\bfx)$ and $\gamma(1) = \bfq_G$. As the robot navigates through the environment, the map is continuously updated by the SLAM algorithm and the path is continuously replanned to adapt to changes in the map.


\textbf{Control.} Our control approach simultaneously guarantees the robot's tracking of the planned path $\gamma$ and its safety from collisions in the dynamically changing environment. In Sec.~\ref{sec: dro_cbf_navigation}, we develop a distributionally robust CBF that uses noisy distance measurements and estimated robot states to directly to guarantee safety with respect to dynamic obstacles. In Sec.~\ref{sec: clf_based_path_following}, we develop a CLF to track the path $\gamma$ with stability guarantees.

\section{Distributionally Robust Safe Control}
\label{sec: dro_cbf_navigation}


In this section, we present a distributionally robust control barrier function (DR-CBF) formulation that enables real-time safety guarantees in cluttered and dynamic environments by utilizing sensor data directly. This formulation is applicable to general control-affine systems \eqref{eq: dynamic}, as introduced in Sec.~\ref{sec:problem}.


We consider a CBF \(h(\mathbf{x}, t)\) with super zero-level set $
\mathcal{C}(t) = \{\mathbf{x} \in \mathcal{X} \mid h(\mathbf{x}, t) \geq 0 \}$ that satisfies 
\begin{equation*}
\mathcal{C}(t) \subseteq \{\mathbf{x} \in \mathcal{X} \mid \mathcal{B}(\mathbf{x}(t)) \subseteq \mathcal{F}(t)\},
\end{equation*} 
where \(\mathcal{B}(\mathbf{x}(t))\) represents the robot’s body at state \(\mathbf{x}(t)\), and \(\mathcal{F}(t)\) is the free space at time \(t\). This establishes a connection between the CBF \(h\) and the environment geometry. To develop the DR-CBF formulation, we make the following assumption on the unknown CBF.


\begin{Assumption}
\label{ass: cbf_properties}
The CBF $h$ has a uniform relative degree of 1 with respect to the system dynamics \eqref{eq: dynamic}, i.e., the time derivative of $h(\bfx, t)$ along \eqref{eq: dynamic} depends explicitly on the control input $\bfu$. 
\end{Assumption}

Under Assumption \ref{ass: cbf_properties}, we can write the control barrier constraint associated with $h(\bfx, t)$ as:
\begin{align}
\label{eq: cbc_rewrite_w}
&\text{CBC}(\bfx,\bfu, t) = \\
&[\nabla_{\bfx} h(\bfx, t)]^\top \bfF(\bfx)\ubfu + \frac{\partial h(\bfx,t)}{\partial t} + \alpha_h(h(\bfx, t)) = \notag \\
&\underbrace{[\nabla_{\bfx} h(\bfx, t)^\top \bfF(\bfx), \; \; \alpha_h(h(\bfx, t)), \; \; \frac{\partial h(\bfx,t)}{\partial t}]}_{\bfxi^\top(\mathbf{x}, t)} \begin{bmatrix} \ubfu\\1\\1
\end{bmatrix} \geq 0, \notag 
\end{align}
where we define an uncertainty vector $\bfxi(\mathbf{x}, t) \in \mathbb{R}^{m+3}$, containing the elements of the time derivative of the barrier function, for each $(\bfx, t) \in \calX \times \bbR$.
%
Since the environment is unknown and both the sensor measurements and the robot state estimates are noisy, $\bfxi$ cannot be determined exactly.

%
%

Our formulation addresses uncertainties arising from both the barrier function $h$ and the state estimations $\mathbf{x}$. Unlike existing robust or probabilistic CBF approaches, which often focus on uncertainties in system dynamics \citep{clark_2019_acc_robust_cbf,Long2022RAL,dhiman_2023_tac_probabilistic, Wang_2023_acc_dob, breeden2023robust, das2024robust}, few works tackle the challenges posed by state estimation errors. This stems from the nonlinearity of the dynamics model $\bfF$ and the barrier function $h$, which makes it difficult to propagate state estimation errors in the control barrier constraint (CBC) in \eqref{eq: cbc_rewrite_w}.

Our formulation addresses this challenge by leveraging the power of distributionally robust optimization. Instead of explicitly propagating state estimation errors through the system dynamics and barrier functions, we assume access to $M$ state samples $\{\mathbf{x}_j\}_{j=1}^M$ from a state estimation algorithm. These samples can be obtained, for example, from the Gaussian distributions provided by a Kalman filter \citep{kalman1960new}, particles generated by a particle filter \citep{djuric2003particle}, or a graph-based localization algorithm \citep{grisetti2010tutorial}. The state samples, combined with estimates of $h$ (e.g., obtained directly from the distance measurements $\bfeta$), are directly used to construct samples of the uncertainty vector $\{\boldsymbol{\xi}_i\}_{i=1}^N$, as detailed in Sec.~\ref{sec: unicycle_dr_cbf_samples}. By managing uncertainties in $\bfxi$ using distributionally robust optimization, we ensure the satisfaction of constraint \eqref{eq: cbc_rewrite_w} without requiring explicit error propagation.

Before introducing the control synthesis formulation, we review the preliminaries of chance constraints and distributionally robust optimization.


\subsection{Chance Constraints and Distributionally Robust Optimization}
\label{sec: prelim_cvar_ccp}

Consider a random vector $\boldsymbol{\xi}$ with (unknown) distribution $\mathbb{P}^*$ supported on the set $\Xi \subseteq \bbR^k$. Let $G :\bbR^m \times \Xi \to \bbR$ define an inequality constraint $G(\bfu,\boldsymbol{\xi}) \leq 0$ (e.g., the CBC in \eqref{eq: cbc_rewrite_w}). Consider, then, the chance-constrained program,
\begin{equation}
\begin{aligned}
\label{eq: ccp}
    &\min_{\bfu \in \bbR^m} c(\bfu),  \\
    \mathrm{s.t.} \, \, & \mathbb{P}^*(G(\bfu, \bfxi) \leq 0) \geq 1 - \epsilon, 
\end{aligned}
\end{equation}
where $c: \bbR^m \mapsto \bbR$ is a convex objective function (e.g., the objective function in~\eqref{eq: clf_cbf_qp}) and $\epsilon \in (0,1)$ denotes a user-specified risk tolerance. Generally, the chance constraint in~\eqref{eq: ccp} leads to a non-convex feasible set. To address this, \citet{Nemirovski2006ConvexAO} propose a convex conditional value-at-risk (CVaR) approximation of the original chance constraint. 

Value-at-risk (VaR) at confidence level $1 - \epsilon$ for $\epsilon \in (0,1)$ is defined as $\text{VaR}_{1-\epsilon}^{\mathbb{P}_q}(Q) := \inf_{s \in \bbR}\{s \; | \; \bbP_q(Q \leq s) \geq 1 - \epsilon\}$ for a random variable $Q$ with distribution $\bbP_q$.  As VaR does not provide information about the right tail of the distribution and leads to intractable optimization in general, one can employ CVaR instead, defined as $\text{CVaR}_{1-\epsilon}^{\mathbb{P}_q}(Q) = \bbE_{\mathbb{P}_q} [ Q \; | \; Q \geq \text{VaR}_{1-\epsilon}^{\mathbb{P}_q}(Q)] $. The resulting constraint
\begin{align}\label{eq: cvar_ccp}
\text{CVaR}_{1-\epsilon}^{\mathbb{P}^*}(G(\bfu,\boldsymbol{\xi}))\leq0
\end{align}
creates a convex feasible set, which is a subset of the feasible set in the original chance-constrained problem~\eqref{eq: ccp}. Additionally, CVaR can be written as the following convex program~\citep{Rockafellar00optimizationof}:
\begin{equation}
\label{eq: cvar_opti_def}   
    \text{CVaR}_{1-\epsilon}^{\mathbb{P}^*}(G(\bfu,\bfxi)) := \inf_{s \in \mathbb{R}}[\epsilon^{-1}\mathbb{E}_{\mathbb{P}^*}[(G(\bfu,\bfxi)+s)_+]-s].
\end{equation}

The formulations in \eqref{eq: ccp} and \eqref{eq: cvar_ccp} require knowledge of $\bbP^*$ to be utilized. However, in many robotics applications, usually only samples of the uncertainty $\bfxi$ are available (e.g., obtained from LiDAR distance measurements). This motivates us to consider distributionally robust formulations~\citep{Esfahani2018DatadrivenDR, Xie2021OnDR}. 

Assuming finitely many samples $\{\bfxi_i\}_{i \in [N]}$ from the true distribution of $\bbP^*$ are available, we first describe a way of constructing an ambiguity set of distributions that agree with the empirical distribution. Let $\calP_p(\Xi) \subseteq \calP(\Xi)$ be the set of Borel probability measures with finite $p$-th moment with $p \geq 1$. The $p$-Wasserstein distance between two probability measures $\mu$, $\nu$ in $\calP_p(\Xi)$ is defined as:
%
\begin{equation}
\label{eq: wasserstein_def}
    W_{p}(\mu,\nu) := \left(\inf_{\beta \in \bbQ(\mu,\nu)} \left[ \int_{\Xi \times \Xi} \eta(\boldsymbol{\xi},\boldsymbol{\xi}')^p \text{d}\beta(\boldsymbol{\xi},\boldsymbol{\xi}') 
    \right] \right)^{\frac{1}{p}}, 
\end{equation}
where $\bbQ(\mu,\nu)$ denotes the collection of all measures on $\Xi \times \Xi$ with marginals $\mu$  and $\nu$ on the first and second factors, and $\eta$ denotes the metric in the space $\Xi$. Throughout the paper, we take $\eta(\boldsymbol{\xi},\boldsymbol{\xi}') = \|\boldsymbol{\xi} - \boldsymbol{\xi}'\|_1$ and consider the ambiguity set corresponding to the $1$-Wasserstein distance. We denote by $\mathbb{P}_N :=  \frac{1}{N}\sum_{i=1}^N \delta_{\bfxi_i}$ the discrete empirical distribution of the available samples $\{\bfxi_i\}_{i \in [N]}$, and define an ambiguity set, $\calM_{N}^{r} := \{\mu \in \calP_p(\Xi) \; | \; W_p(\mu,\mathbb{P}_{N} ) \leq r\}$, as a ball of distributions with radius $r$ centered at $\mathbb{P}_N$.

\begin{remark}\longthmtitle{Choice of Wasserstein ball radius}\label{rem:choice-r}
{\rm
There is a connection between the sample size $N$ and the Wasserstein radius $r$ for constructing the ambiguity set $\calM_N^r$. A distribution $\mathbb{P}$ is light-tailed if there exists an exponent $\rho$ such that $A := \mathbb{E}_{\mathbb{P}}[\exp{\norm{\boldsymbol{\xi}}^\rho}]=\int_{\Xi}\exp{\norm{\boldsymbol{\xi}}^\rho}\mathbb{P}(d\boldsymbol{\xi})<\infty$.
If the true distribution $\mathbb{P}^*$ is light-tailed, the choice of $r=r_N(\bar{\epsilon}) $ given in~\citet[Theorem 3.5]{Esfahani2018DatadrivenDR} is
    \begin{align}
    \label{eq: wasserstein_r_guarantee}
        r_N(\bar{\epsilon}) = \begin{cases}
            (\frac{\log(c_1\bar{\epsilon}^{-1})}{c_2 N})^{\frac{1}{\max\{k,2\}}} \quad &\text{if} \ N\geq \frac{\log(c_1\bar{\epsilon}^{-1})}{c_2}, \\
            (\frac{\log(c_1\bar{\epsilon}^{-1})}{c_2 N})^{\frac{1}{\rho}} \quad &\text{else},
        \end{cases}
    \end{align}
    where $c_1, c_2$ are 
    positive constants that depend on $\rho, A$ and $k$,
    ensures that the ambiguity ball $\calM_N^{r_{N}(\bar{\epsilon})}$ contains $\mathbb{P}^*$ with probability at least $1-\bar{\epsilon}$. 
    } \demo
\end{remark}

\begin{figure}[t]
  \centering
  \includegraphics[width=0.9\linewidth]{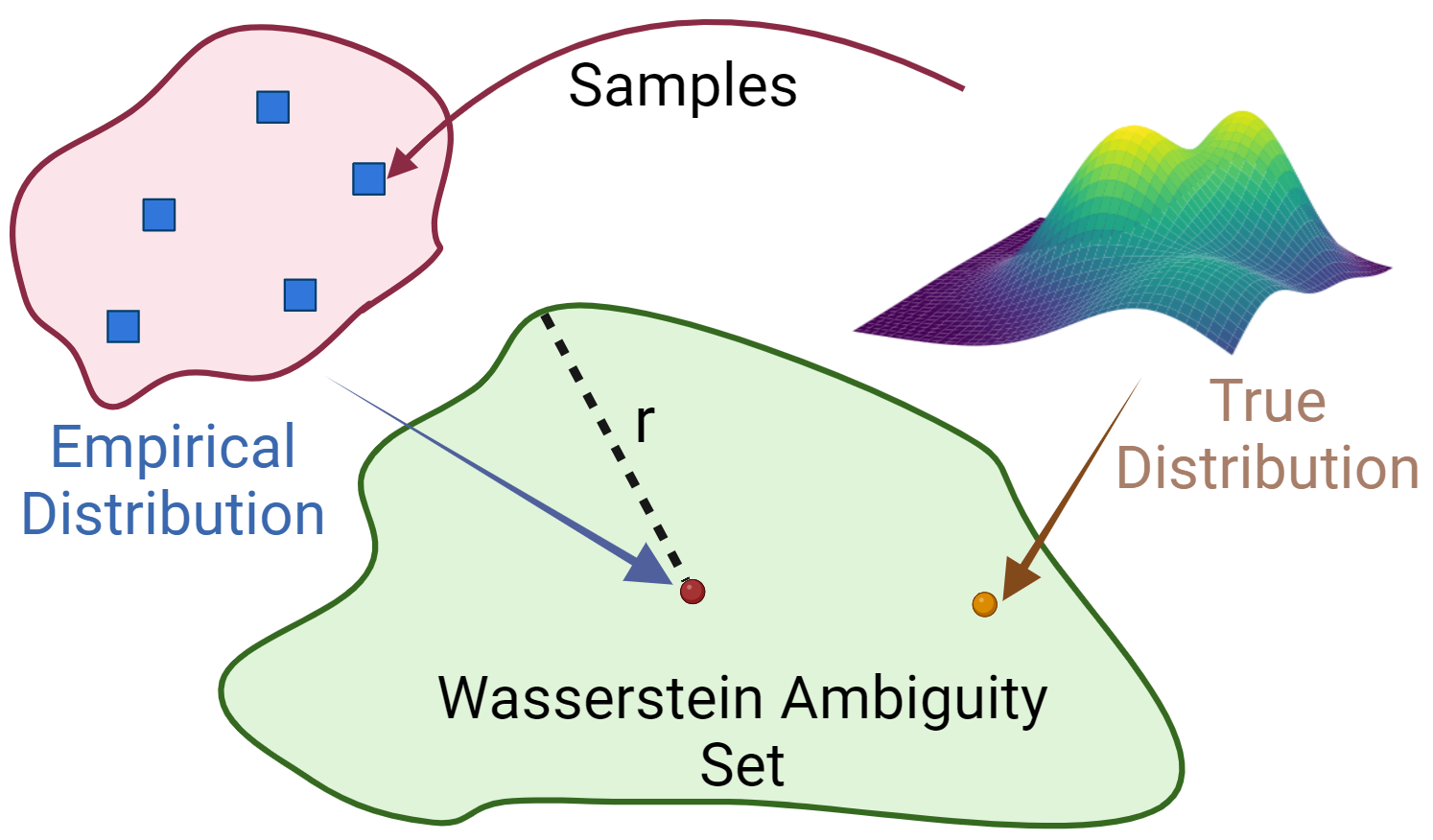}\\
  \caption{Wasserstein ambiguity set illustration. The figure shows the relationship between the samples, empirical distribution, true distribution, and the Wasserstein ambiguity set. The blue squares represent the available samples from the true distribution (yellow dot), which form the empirical distribution (red dot). The Wasserstein ambiguity set (green region) is constructed as a ball of distributions centered at the empirical distribution, with a radius $r$ that depends on the sample size and the desired confidence level. The ambiguity set aims to contain the true distribution with high probability.}
  \label{fig: dro_illustrate}
  \vspace{-2ex}
\end{figure}

Fig.~\ref{fig: dro_illustrate} provides an illustration of the Wasserstein ambiguity set and its relation to the samples, the empirical distribution, and the true distribution. 

\subsection{Distributionally Robust Safety Constraint}\label{sec: dr_safety_formualte}

Consistently with our exposition of DRO in the previous section, we make the following assumption. 

\begin{Assumption}\label{ass: cbf_samples}
At each $(\bfx, t) \in \calX \times \bbR$, $N$ samples of the vector $\bfxi$ in \eqref{eq: cbc_rewrite_w} can be obtained, denoted by $\{\bfxi_i\}_{i \in [N]}$. 
\end{Assumption}

The samples $\{\bfxi_i\}_{i \in [N]}$ can be obtained using sensor measurements and state estimation (we discuss this in detail in Sec.~\ref{sec: unicycle_dr_cbf_samples}). In many robotic systems, sensing and state estimation may operate at lower frequencies than the control loop. For example, LiDAR measurements or SLAM-based localization may provide updates at about $10$ Hz, while the control loop may require computations at $50$ Hz. Our DR-CBF formulation addresses this challenge by incorporating samples of 
$\bfxi$ derived from potentially delayed or uncertain sensor data and state estimations. By accounting for both the asynchrony and uncertainty inherent in sensing and estimation, our approach ensures probabilistic safe performance under realistic conditions.

Inspired by the CLF-CBF QP formulation in~\eqref{eq: clf_cbf_qp}, we consider the following distributionally robust formulation to ensure safety with high probability:
\begin{subequations}
\label{eq: clc_cbc_cvar_approximate}
\begin{align}
    (\bfu(\bfx, t), \delta) &= \argmin_{\bfu \in \bbR^m,\delta \in \bbR} \| \bfu - \bfk(\bfx)\|^2 + \lambda \delta^2, \notag \\
    \text{s.t.} \quad &\text{CLC}(\bfx,\bfu) \leq \delta, \label{eq: clc_constraint} \\
    &\inf_{\mathbb{P} \in \calM_{N}^{r}}\mathbb{P}(\text{CBC}(\bfx, \bfu, \bfxi)) \geq 0 ) \geq 1 - \epsilon, \label{eq: cvar_cbc_constraint}
\end{align}
\end{subequations}
where $\calM_N^r$ denotes the ambiguity set with radius $r$ around the empirical distribution $\bbP_N$. The explicit time dependency of $\bfu$ on $t$ stems from the random vector $\bfxi(\bfx,t)$ in the CBF constraint. The formulation in \eqref{eq: clc_cbc_cvar_approximate} addresses the inherent uncertainty in the safety constraint without assuming a specific probabilistic model for $\bfxi$. The Wasserstein radius $r$ defines the acceptable deviation of the true distribution of $\bfxi$ from the empirical distribution $\mathbb{P}_N$. 


If a controller $\bfu^*(\bfx, t)$ satisfies \eqref{eq: cvar_cbc_constraint}, the following result ensures that the closed-loop system satisfies a chance constraint under the true distribution.

\begin{lemma}\longthmtitle{Chance-constraint satisfaction under the true distribution}\label{lemma: dr_safety_guarantee}
Assume the distribution $\bbP^*$ of $\bfxi$ is light-tailed and the Wasserstein radius $r_N(\bar{\epsilon})$ is set according to~\eqref{eq: wasserstein_r_guarantee}. If the controller $\bfu^*(\bfx, t)$ satisfies \eqref{eq: cvar_cbc_constraint} with $r=r_N(\bar{\epsilon})$, then
\begin{equation}
\label{eq: dr_safety_guarantee}
\bbP^*(\operatorname{CBC}(\bfx, \bfu^*(\bfx, t), \bfxi)) \geq 0 ) \geq (1-\epsilon)(1 - \bar{\epsilon}).     
\end{equation}
\end{lemma}
\begin{proof}
Consider the events $A \!:=\! \{\bbP^* \in \calM_N^{r_N(\bar{\epsilon})}\}$ and $B \!:=\! \{ \text{CBC}(\bfx, \bfu^*(\bfx,t), \bfxi)) \geq 0 \}$.
From \citet[Theorem 3.4]{Esfahani2018DatadrivenDR}, we have $\bbP^*(A) \geq 1- \bar{\epsilon}$. From \eqref{eq: cvar_cbc_constraint}, we have that
\begin{equation}
\label{eq: dr_condition_x}
\inf_{\mathbb{P}\in \calM_{N}^{r_N(\bar{\epsilon})}}\bbP(B) \geq 1-\epsilon.
\end{equation}
Now, consider the probability of the event $B$ under the true distribution $\bbP^*$:
\begin{align}
\bbP^*(B) &\geq \bbP^*(B \cap A) = \bbP^*(B | A)\bbP^*(A) \\
&\geq \left(\inf_{\bbP \in \calM_{N}^{r_N(\bar{\epsilon})}}\bbP(B)\right)\bbP^*(A) \geq (1-\epsilon)(1-\bar{\epsilon}). \qed \notag
\end{align}             
\end{proof}

Our previous work \citep{long2024distributionally_policy} presents a similar result but for a CLF in the context of stabilization. According to Lemma~\ref{lemma: dr_safety_guarantee}, the safety of the closed-loop system is guaranteed with high probability. However, the optimization problem in \eqref{eq: clc_cbc_cvar_approximate} is intractable \citep{Hota2019DataDrivenCC, Esfahani2018DatadrivenDR} due to the infimum over the Wasserstein ambiguity set. In Sec.~\ref{sec: dr_cbf_formulate}, we discuss our approach to identify tractable reformulations of \eqref{eq: clc_cbc_cvar_approximate} and facilitate online safe control synthesis.

\subsection{Tractable Convex Reformulation}
\label{sec: dr_cbf_formulate}

Next, we demonstrate how the samples $\{\bfxi_i\}_{i \in [N]}$ from Assumption \ref{ass: cbf_samples} can be used to obtain a tractable reformulation of \eqref{eq: clc_cbc_cvar_approximate}.


\begin{proposition}\longthmtitle{Distributionally robust safe control synthesis}
\label{proposition: dro_cbf}
Given samples $\{\bfxi_i\}_{i \in [N]}$ of $\bfxi$ as in \eqref{eq: cbc_rewrite_w}, if $(\bfu^*, \delta^*, s^*, \{\beta_i^*\}_{i \in [N]})$ is a solution to the quadratic program:
%
%
\begin{align}\label{eq: clc_cbc_dr_result}
    &\min_{\bfu \in \bbR^m,\delta \in \bbR, s \in \bbR, \beta_i \in \bbR} \| \bfu - \bfk(\bfx)\|^2 + \lambda \delta^2,  \\
    \mathrm{s.t.} \, \,  &\operatorname{CLC}(\bfx,\bfu) \leq \delta, \notag \\ 
    & r \|\ubfu\|_{\infty}  \leq s\epsilon - \frac{1}{N} \sum_{i=1}^{N} \beta_i, \notag \\
    & \beta_i \! \geq \! s \!-\! [\ubfu \; \; 1 \; \; 1]^\top \bfxi_i, \;  \beta_i \!\geq \!0, \;  \forall i \in [N] ,\notag
\end{align}
then $(\bfu^*, \delta^*)$ is also a solution to the distributionally robust chance-constrained program in \eqref{eq: clc_cbc_cvar_approximate}.
\end{proposition}

\begin{proof}
The safety constraint~\eqref{eq: cvar_cbc_constraint} is equivalent to
$\sup_{\bbP \in \calM_{N}^{r}}\mathbb{P}(-\text{CBC}(\bfx, \bfu, \bfxi) \geq 0) \leq \epsilon$. Using the CVaR approximation of the chance constraint~\eqref{eq: cvar_ccp}, we obtain a convex conservative approximation of~\eqref{eq: cvar_cbc_constraint}: 
\begin{equation}
    \sup_{\bbP \in \calM_{N}^{r}} \text{CVaR}_{1-\epsilon}^{\mathbb{P}}(-\text{CBC}(\bfx, \bfu, \bfxi)) \leq 0. 
\end{equation}
From~\eqref{eq: cvar_opti_def}, this is equivalent to 
\begin{equation}
\label{eq: sup_inf_drcc_cbc}
\sup_{\mathbb{P}\in \calM_{N}^{r}}\inf_{s \in \mathbb{R}}[\frac{1}{\epsilon}\mathbb{E}_{\bbP}[(-\text{CBC}(\bfx,\bfu,\bfxi)+s)_{+}] -s]  \leq  0. 
\end{equation}
Based on \citet[Lemma V.8]{Hota2019DataDrivenCC} and \citet[Theorem 6.3]{Esfahani2018DatadrivenDR}, with the 1-Wasserstein distance, the following inequality is a sufficient condition for~\eqref{eq: sup_inf_drcc_cbc} to hold:
\begin{equation}
\label{eq: sample_average_regular}
r L(\bfu) \!+\! \inf_{s \in \mathbb{R}}\left[\frac{1}{N} \sum_{i=1}^{N} (-\text{CBC}(\bfx,\bfu,\bfxi_i)+s)_{+} -s\epsilon \right] \! \leq \! 0
\end{equation}
where $L(\bfu)$ is the Lipschitz constant of $-\text{CBC}(\bfx,\bfu,\bfxi)$ in $\bfxi$.
Now, from~\eqref{eq: cbc_rewrite_w}, we have 
$\text{CBC}(\bfx,\bfu, \bfxi) = [\ubfu \; \;  1 \; \; 1]^\top \bfxi$. Therefore, we can define the convex function $L: \bbR^m \mapsto \mathbb{R}_{>0}$ by
\begin{equation}
\label{eq: L_h_define}
    L(\bfu) = \|[\ubfu \; \;1 \; \; 1]^\top\|_{\infty} = \max \{1, \|\ubfu\|_{\infty}\} = \|\ubfu\|_{\infty}
\end{equation}
The function $\bfxi \mapsto -\text{CBC} (\bfx,\bfu,\bfxi)$ is Lipschitz in $\bfxi$ with constant $L(\bfu)$. This is because the Lipschitz constant of a differentiable affine function equals the dual norm of its gradient, and the dual norm of the $L_1$ norm is the $L_{\infty}$ norm. Thus, 
the following is a conservative approximation of~\eqref{eq: clc_cbc_cvar_approximate},
\begin{align}\label{eq: dr_cbc_Convex_Reformulation}
&\min_{\bfu \in \bbR^m,\delta \in \bbR} \| \bfu - \bfk(\bfx)\|^2 + \lambda \delta^2,\\
\mathrm{s.t.} \, \, 
&\text{CLC}(\bfx,\bfu) \leq \delta, \notag \\ 
r L(&\bfu) \!+\! \inf_{s \in \mathbb{R}}\left[\frac{1}{N} \sum_{i=1}^{N} (-\text{CBC}(\bfx,\bfu,\bfxi_i) + s)_+ - s \epsilon\right] \!\leq \!0. \notag  
\end{align}
Lastly, as shown in \citet[Proposition~IV.1]{Long2023_acc_drccp}, the bi-level optimization in~\eqref{eq: dr_cbc_Convex_Reformulation} can be rewritten as~\eqref{eq: clc_cbc_dr_result}.
\qed
\end{proof}

Formally, the safe set $\calC(t)$ is defined by a single CBF $h(\bfx,t)$, but its value, gradient, and time derivative (defining $\bfxi(\bfx, t)$) are not directly known in practice. By leveraging observations of $\bfxi$ derived from sensor measurements and state estimates, and applying Proposition~\ref{proposition: dro_cbf}, we are able to synthesize safe controllers with distributionally robust guarantees. Due to the convex reformulation of \eqref{eq: cvar_cbc_constraint}, our approach inherently introduces some conservatism. However, in practice, this conservatism can be effectively managed by tuning the Wasserstein radius \( r \) and the safety probability \( \epsilon \). These parameters provide flexibility to adapt the formulation based on specific application requirements, sensor capabilities, and state estimation errors. 

The Lipschitz continuity and regularity of distributionally robust controllers are characterized in \citet{PM-KL-NA-JC:23-csl}. Together with Lemma~\ref{lemma: dr_safety_guarantee}, this enables safe robot control with point-wise probabilistic guarantees in unknown dynamic environments.

\begin{remark}\longthmtitle{Uncertainty in System Dynamics}
{\rm 
We have assumed that there is no uncertainty in the system dynamics \( \bfF \) to simplify the presentation in Sec. \ref{sec: dr_cbf_formulate}. However, our approach can be extended if this is not the case as long as samples of \( \bfF(\bfx) \) are available. These samples can be combined with samples of robot state and \( h(\bfx, t) \) to construct the uncertainty vector \( \{\bfxi_i\}_{i=1}^N \) and ensure the validity of Proposition~\ref{proposition: dro_cbf}. } \demo
\end{remark}

\section{Control Lyapunov Function Based Path Following}
\label{sec: clf_based_path_following}

In this section, we introduce a control strategy that accurately tracks the planned path~$\gamma$. This control strategy will serve as the basis for specifying the control Lyapunov constraint (CLC) in our distributionally robust safe control synthesis in~\eqref{eq: clc_cbc_dr_result}.

%
%
We begin by stating the following assumptions:

\begin{Assumption}
\label{ass: compact_state_space}
The state space $\calX \subset \mathbb{R}^n$ is compact. 
\end{Assumption}


\begin{Assumption}
\label{ass: clf_properties}
Given a desired reference point $\bfq \in \bbR^p$, let $\calE(\bfq) = \{\bfx \in \calX : \phi(\bfx) = \bfq\}$. For each $\bfq$, assume that there exists a continuously differentiable function $V : \calX \times \bbR^p \rightarrow \bbR_{\geq 0}$ with the following properties.
\begin{enumerate}
\item The function $V(\bfx)$ is positive definite with respect to the error $\phi(\bfx) - \bfq$, i.e., $V(\bfx) > 0$ for all $\bfx \not \in \calE(\bfq)$ and $V(\bfx) = 0$ if $\bfx \in \calE(\bfq)$.

\item There exists a continuous control law $\hat{\bfu}(\bfx, \bfq)$ such that the time derivative of $V$ along the trajectories of the system \eqref{eq: dynamic} satisfies $\dot{V}(\bfx) < 0$ for all $\bfx \not \in \calE(\bfq)$.


\item For all $\bfx \in \calE(\bfq)$, we have $\bff(\bfx) + \bfg(\bfx)\hat{\bfu}(\bfx, \bfq) = \boldsymbol{0}$.
\end{enumerate}
\end{Assumption}

A specific construction of such a function is provided in Sec.~\ref{sec: CLF_stabilize_pt}.

\subsection{Stabilization to a Goal Position}
\label{sec: stab_to_1_point}

We now establish the asymptotic stability of the goal set $\calE(\bfq)$ for the closed-loop system dynamics~\eqref{eq: dynamic} with control law $\hat{\bfu}(\bfx, \bfq)$.

\begin{lemma}\longthmtitle{Asymptotic stability of the goal set}\label{lemma: closed_loop_convergence}
Under Assumptions \ref{ass: compact_state_space} and \ref{ass: clf_properties}, the goal set $\calE(\bfq) = \{\bfx \in \calX : \phi(\bfx) = \bfq\}$ is asymptotically stable for the closed-loop dynamics \eqref{eq: dynamic} under the control law $\hat{\bfu}(\bfx, \bfq)$.
\end{lemma}
%

\begin{proof}
By Assumption \ref{ass: compact_state_space}, $\calX$ is compact. Since $\calE$ is a closed subset of $\calX$, it follows that $\calE(\bfq)$ is also compact. 
In addition, by LaSalle's Invariance Principle~\citep{khalil2002nonlinear}, the conditions in Assumption \ref{ass: clf_properties} imply that the closed-loop system trajectory converges to the largest invariant set contained in $\{\bfx \in \calX : \dot{V}(\bfx) = 0\} \subset \calE(\bfq)$. This implies that $\calE$ is asymptotically attractive too. Therefore, $\calE(\bfq)$ is asymptotically stable for the closed-loop system dynamics under the control law $\hat{\bfu}(\bfx, \bfq)$.
\qed
\end{proof}

Lemma~\ref{lemma: closed_loop_convergence} establishes that, under the control law $\hat{\bfu}(\bfx, \bfq)$ in Assumption~\ref{ass: clf_properties}, the position of the system in \eqref{eq: dynamic} satisfies $\phi(\bfx(t)) \rightarrow \bfq$ as $t \rightarrow \infty$.


\subsection{CLF-Based Path Following}
\label{sec: clf_path_follow}

Building upon the stability result in Lemma \ref{lemma: closed_loop_convergence}, we now extend the position convergence to path following. Our goal is to achieve smooth navigation by dynamically adjusting a moving goal point along the planned path $\gamma$. Inspired by reference governor control techniques \citep{garone2015explicit, li2020fast}, we consider a scalar $g(t) \in [0,1]$ with dynamics:
\begin{equation}
\label{eq: ref_pt_dynamics}
\dot{g} = \frac{k}{1 + \|\phi(\bfx) - \gamma(g)\|} (1 - g^{\zeta}),
\end{equation}
where $k \in \bbR_{>0}$ is a scaling factor, and $\zeta \in \bbN$ ensures that $g$ asymptotically approaches but never exceeds $1$. The dynamics in \eqref{eq: ref_pt_dynamics} are designed such that the reference point $\gamma(g)$ moves along the path $\gamma$ at a speed inversely proportional to the distance between the current robot position $\phi(\bfx)$ and $\gamma(g)$, facilitating a responsive path following behavior. The initial condition for $g$ is set to $g(0) = 0$, corresponding to the starting point of the path $\gamma(0)$.

\begin{lemma}\longthmtitle{Asymptotic stability of the governor dynamics}\label{lemma: governor_convergence}
The equilibrium point $g^* = 1$ is asymptotically stable for the governor dynamics in \eqref{eq: ref_pt_dynamics}.
\end{lemma}

\begin{proof}
Note that $[0,1]$ is forward invariant under \eqref{eq: ref_pt_dynamics}. Consider the candidate Lyapunov function:
\begin{equation}
\label{eq: governor_lyapunov}
V_g(g) = \frac{1}{2}(1 - g)^2.
\end{equation}
Note that $V_g$ is positive definite with respect to $g=1$. Its time derivative along the trajectories of \eqref{eq: ref_pt_dynamics} is
\begin{align*}
\dot{V}_g(g) &= -(1 - g)\dot{g} = -\frac{k(1 - g)}{1 + \|\phi(\bfx) - \gamma(g)\|} (1 - g^{\zeta})
\end{align*}
Since $\dot{V}_g(g) \leq 0$ for all $g \in [0, \infty)$ and $\dot{V}_g(g) = 0$ if and only if $g = 1$, we conclude that $g = 1$ is globally asymptotically stable (over $[0, \infty)$) for the governor dynamics \eqref{eq: ref_pt_dynamics}. \qed
\end{proof}

%
%
From Lemma~\ref{lemma: closed_loop_convergence}, for a goal position $\bfq$, the control-affine system under the control law $\hat{\bfu}(\bfx, \bfq)$ converges to the set of equilibrium points $\calE(\bfq)=\{\bfx \in \calX : \phi(\bfx) = \bfq\}$. In the path-following context, we make $\bfq$ move along the path $\gamma$, resulting in the control law $\hat{\bfu}(\bfx, \gamma(g))$ and the equilibrium set $\calE(\gamma(g))$. The following result formalizes the asymptotic convergence of the interconnected system.

\begin{theorem}\longthmtitle{Asymptotic stability of the interconnected system}\label{theorem: interconnected_convergence}
Consider the interconnected system consisting of the governor dynamics \eqref{eq: ref_pt_dynamics} and the closed-loop dynamics \eqref{eq: dynamic} with the control law $\hat{\bfu}(\bfx, \gamma(g))$. Under Assumptions \ref{ass: compact_state_space} and \ref{ass: clf_properties}, there exists a sufficiently small $k^* > 0$ such that, for all $k \in (0, k^*]$, the equilibrium set
$\calE(\gamma(1)) \times \{1\}$
is asymptotically stable for the interconnected system.
\end{theorem}
%
%
\begin{proof}
We prove the result using singular perturbation theory \citep{khalil2002nonlinear}. We view the control-affine dynamics with state $\bfx \in \calX$ as the fast subsystem and the governor dynamics with state $g \in [0, 1]$ as the slow subsystem.
First, we analyze the reduced-order model, obtained by setting $\dot{g} = 0$ in \eqref{eq: ref_pt_dynamics}:
\begin{equation}
\label{eq: reduced_order_model}
0 = \frac{k}{1 + \|\phi(\bfx) - \gamma(g)\|} (1 - g^{\zeta}).
\end{equation}
The solution to \eqref{eq: reduced_order_model} is $g = 1$, which corresponds to the endpoint of the path $\gamma(1)$. By Lemma~\ref{lemma: closed_loop_convergence}, when $g=1$, the equilibrium set $ \calE(\gamma(1))$ for the constant reference point $\gamma(1)$ is asymptotically stable for the control-affine dynamics with the control law $\hat{\bfu}(\bfx, \gamma(1))$.

Next, we consider the boundary-layer system, obtained by introducing a fast time scale $\tau = t/\nu$ and taking the limit $\nu \to 0$:
\begin{equation}
\label{eq: boundary_layer_system}
\frac{d\bfx}{d\tau} = \bff(\bfx) + \bfg(\bfx)\hat{\bfu}(\bfx, \gamma(\bar{g})),
\end{equation}
where $\bar{g} \in [0, 1]$ is treated as a fixed parameter. By Lemma~\ref{lemma: closed_loop_convergence}, for each fixed $\bar{g}$, the equilibrium set $\calE(\gamma(\bar{g}))$ is asymptotically stable for the boundary-layer system \eqref{eq: boundary_layer_system}. 

Next, we analyze the reduced slow system, obtained by substituting the quasi-steady-state solution $\bfx(g) \in \calE(\gamma(1))$ into the slow subsystem:
\begin{equation}
\label{eq: reduced_slow_system}
\dot{g} = \frac{k}{1 + \|\gamma(g) - \gamma(g)\|} (1 - g^{\zeta}) = k(1 - g^{\zeta}).
\end{equation}
This system has a unique equilibrium point $g^* = 1$, which is globally asymptotically stable on $[0,\infty)$.

Consequently, as discussed in \citet[Appendix C.3]{khalil2002nonlinear}, there exists a sufficiently small $k^* > 0$ such that, for all $k \in (0, k^*]$, the equilibrium set $\calE(\gamma(1)) \times \{1\}$ is asymptotically stable for the interconnected system. 
\qed
\end{proof}

In practical implementations, the control input is applied at a finite sampling rate. To preserve stability guarantees, we use a sampling frequency of 50 Hz in the simulations and experiments, ensuring that the discretized behavior closely approximates the continuous-time result.

\begin{remark}\longthmtitle{Practical considerations for control bounds}\label{remark: control_bounds}
{\rm
The original CLF-CBF QP formulation in \eqref{eq: clf_cbf_qp} assumes no control bounds, a condition often not met in real-world robot applications due to physical limitations, such as maximum speed and acceleration. To ensure the applicability of our approach within these practical constraints, one can tune the parameters $k$ in the governor dynamics \eqref{eq: ref_pt_dynamics} to establish a smooth path-following behavior that respects the control bounds. Specifically, increasing \(k\) slows down the progression of the reference goal \(\gamma(g)\), reducing the control effort required.
} \demo
\end{remark}

\begin{remark}\longthmtitle{Practical considerations for convergence}
\label{rm: practical_path_follow}
{\rm 
Theorem~\ref{theorem: interconnected_convergence} establishes the asymptotic stability of the equilibrium set $\bar{\calE}(\gamma(1)) \times {1}$ for the interconnected system, which implies that the robot's position $\phi(\bfx(t))$ converges to the endpoint $\gamma(1)$ of the path as $t \to \infty$. However, in practical applications, it is important to consider the finite-time convergence of the robot to its destination.
To address this, we introduce a threshold $\mu^* > 0$ and consider the robot to have effectively reached its destination when $|\phi(\bfx(t)) - \gamma(1)| \leq \mu^*$. By setting an appropriate value for $\mu^*$, we can guarantee that the robot completes its navigation task within a finite time, while still ensuring that it reaches a sufficiently close vicinity of the goal $\gamma(1)$.
} \demo
\end{remark}

\section{Application to Differential-Drive Robot}
\label{sec: unicycle}

%
%

\begin{figure}[t]
  \centering
  \subcaptionbox{Path tracking\label{fig:clf_tracking}}
  {\includegraphics[width=0.49\linewidth]{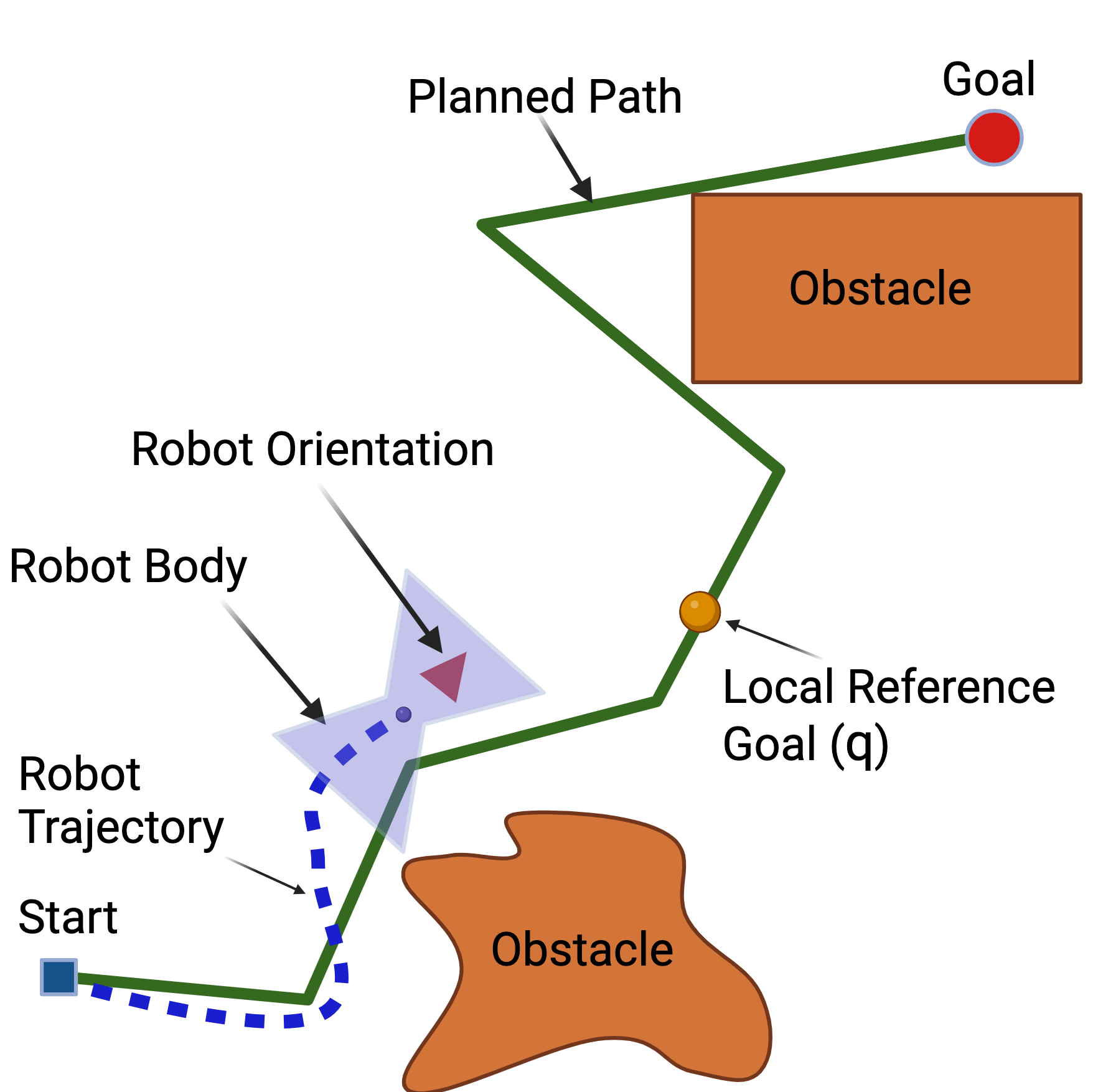}}%
  \hfill%
  \subcaptionbox{Safe navigation\label{fig:dr_safety}}{\includegraphics[width=0.49\linewidth]{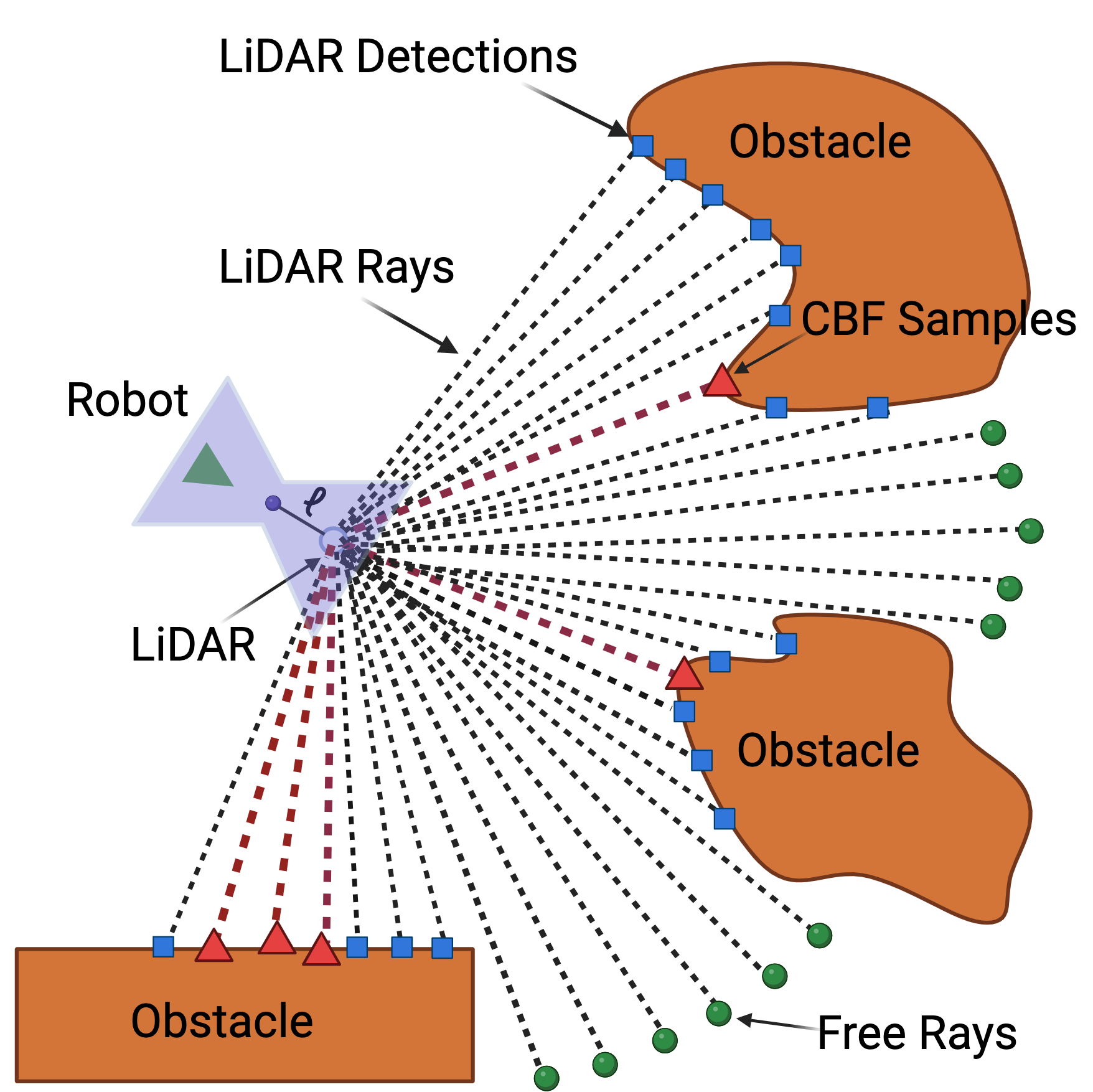}}\\
  \caption{(a) A robot is depicted following a path generated by a motion planning algorithm, and a dynamic local reference goal is highlighted in yellow. (b) The robot senses the environment with a $360$-degree LiDAR sensor mounted at $\tilde{\bfx}$. The CBF samples $\{h_{i}(\bfx)\}_{i=1}^N$ are the rays with boundary points highlighted as red triangles, and are selected based on the distance from the LiDAR detections to the robot body.}\label{fig:2}
\end{figure}

While the theory of our control method is applicable to general control-affine systems~\eqref{eq: dynamic}, our experiments focus on a wheeled differential-drive robot.

We consider a robot with state $\bfx := [x, y ,\theta]^\top \in \calX \subseteq \mathbb{R}^2 \times [-\pi,\pi)$, input $\bfu := [v,\omega]^\top \in \calU \subset \mathbb{R}^2$, and dynamics:
\begin{equation}\label{eq: unicycle_model}
\dot{\bfx} = \left[\begin{matrix} \cos(\theta) & 0 \\ \sin(\theta) & 0\\ 0 & 1 \end{matrix}\right] \left[\begin{matrix} v \\ \omega \end{matrix}\right].
\end{equation}
The function $\phi: \calX \rightarrow \mathbb{R}^2$ projecting the robot state $\bfx$ to its position is $\phi(\bfx) = \phi([x, y, \theta]^\top) = [x, y]^\top$. 

We define the state space of the robot as $\calX = \calD \times [-\pi, \pi)$, where $\calD \subset \bbR^2$ is a sufficiently large compact set containing the environment of interest, including the goal point $\bfq_G$ and the planned path $\gamma$.

In this section, we instantiate the control strategy developed in the previous sections to the case of differential-drive dynamics in \eqref{eq: unicycle_model}. We first design a CLF that enables stabilization to a desired goal point, as requested by Assumption~\ref{ass: clf_properties}. Then, we discuss the validity of using the SDF as a CBF candidate for ensuring safety. Next, we provide a discussion on how to select CBF samples based on range sensor measurements, which is crucial for the practical implementation of the proposed DR-CBF formulation. To address uncertainties in both localization and sensor data, we also introduce a unified error model that captures how localization errors propagate to LiDAR measurements, impacting the robot's perception of obstacles. By combining the unicycle-specific CLF and the data-driven CBF, we obtain a CLF-DR-CBF QP formulation tailored to the unicycle dynamics, enabling safe and efficient navigation in unknown dynamic environments.

\subsection{Unicycle Stabilization to a Goal Position}
\label{sec: CLF_stabilize_pt}

We design a CLF that enables stabilization of the unicycle dynamics \eqref{eq: unicycle_model} to a desired goal point $\bfq_G \in \bbR^2$. 

Inspired by \citet{icsleyen2023feedback}, for a local reference point $\bfq \in \bbR^2$, we define $V$ as follows:
\begin{equation}
\label{eq: clf_definition}
V(\bfx) \!=\!
\begin{cases}
\frac{1}{2}(k_v \|\bfq \!-\! \phi(\bfx)\|^2 \!+\! k_{\omega} \text{atan2}(e_v^{\perp}, e_v)^2), & \! \phi(\bfx) \!\neq \!\bfq, \\
0, & \! \phi(\bfx) \!=\! \bfq,
\end{cases}
\end{equation}
where $\phi(\bfx)$ denotes the current robot position, and $k_{v}, k_{\omega} > 0$ are user-specified control gains for linear and angular errors. 
The error terms $e_{v}$ and $e_{v}^{\perp}$ are defined as:
\begin{equation}
\label{eq: err_terms_definition}
    e_v = \begin{bmatrix}
	\cos \theta \\
	\sin \theta 
	\end{bmatrix}^\top \hspace*{-4pt} (\bfq - \phi(\bfx)), \;
	e_v^\perp =
	\begin{bmatrix}
	-\sin \theta \\
	\cos \theta 
	\end{bmatrix}^\top \hspace*{-4pt} (\bfq - \phi(\bfx)).
\end{equation}
When $\phi(\bfx) \neq \bfq$, the first part of $V(\bfx)$ represents the squared Euclidean distance between $\phi(\bfx)$ and $\bfq$, while the second part quantifies the squared angular alignment error. Fig.~\ref{fig:clf_tracking} illustrates the unicycle robot tracking a local reference point $\bfq$
on the planned path.

The time derivative of $V(\bfx)$ is:
\begin{equation}
\label{eq: point_clf_derivative}
\dot{V}(\bfx) = \mathcal{L}_{\bfg}V(\bfx) \bfu = \begin{bmatrix} -k_{v} e_{v} + k_{\omega} \dfrac{\text{atan2}(e_{v}^{\perp} , e_{v})}{\| \bfq - \phi(\bfx)\|^2} e_{v}^{\perp}  \\ - k_{\omega} \text{atan2}(e_{v}^{\perp}, e_{v}) \end{bmatrix}^{\top} \bfu. 
\end{equation}

The following result provides a control law that ensures the satisfaction of the CLC in \eqref{eq: clf}.

\begin{lemma}\longthmtitle{Control Lyapunov constraint satisfaction}\label{lemma: clf_pt}
Let $\alpha_V$ be a class $\calK$ function satisfying $\lim_{r \rightarrow 0^+} \frac{\alpha_V(r)}{r} = 0$.
%
%
For any state $\bfx \in \calX$, the following control law
\begin{equation}
\label{eq: min_norm_u_clf}
\bfu(\bfx)= 
\begin{cases}
- \alpha_{V} (V(\bfx)) \frac{\mathcal{L}_{\bfg}V(\bfx)^{\top}}{\| \mathcal{L}_{\bfg}V(\bfx)^{\top} \|^{2}}, \quad & \phi(\bfx) \neq \bfq, \\
\boldsymbol{0}, \quad & \phi(\bfx) = \bfq,
\end{cases}
\end{equation}
ensures that $\mathcal{L}_{\bff} V(\bfx) + \mathcal{L}_{\bfg} V(\bfx)\bfu(\bfx) + \alpha_V( V(\bfx)) \leq 0$ is satisfied.  Furthermore, the control law is Lipschitz continuous for $\phi(\bfx) \neq \bfq$ and continuous at $\bfq$.

%
\end{lemma}
\begin{proof}
First, note that for any $\bfx \in \calX$ such that $\phi(\bfx) = \bfq$, we have $V(\bfx) = 0$ and $\mathcal{L}_{\bfg}V(\bfx) = \boldsymbol{0}$, regardless of the orientation $\theta$. In this case, the control law $\bfu(\bfx)$ given by \eqref{eq: min_norm_u_clf} is simply $\boldsymbol{0}$, which trivially satisfies the CLC: $\mathcal{L}_{\bff} V(\bfx) + \mathcal{L}_{\bfg} V(\bfx)\bfu(\bfx) + \alpha_V( V(\bfx)) \leq 0$.

Next, we show that $\mathcal{L}_{\bfg}V(\bfx) \neq \boldsymbol{0}$ for all $\bfx \in \calX$ such that $\phi(\bfx) \neq \bfq$. Suppose, by contradiction, that $\mathcal{L}_{\bfg}V(\bfx) = \boldsymbol{0}$. This implies $k_{\omega} \text{atan2}(e_{v}^{\perp} , e_{v}) = 0$, which means $e_{v}^{\perp} = 0$. From \eqref{eq: err_terms_definition}, we have $e_{v}^{2} + e_{v}^{\perp^{2}} = ( \bfq - \phi(\bfx) )^2$. Since $e_{v}^{\perp} = 0$ and $\phi(\bfx) \neq \bfq$, we obtain $e_{v} = \pm | \bfq - \phi(\bfx) | \neq 0$, contradicting the assumption that $\mathcal{L}_{\bfg}V(\bfx) = \boldsymbol{0}$. Hence, $\mathcal{L}_{\bfg}V(\bfx) \neq \boldsymbol{0}$ for all $\bfx \in \calX$ such that $\phi(\bfx) \neq \bfq$.

Now, for all $\bfx \in \calX$ such that $\phi(\bfx) \neq \bfq$, the control law $\bfu(\bfx)$ given by \eqref{eq: min_norm_u_clf} is the closed-form solution to the optimization problem:
\begin{align}
&\min_{\bfu \in \bbR^m} \|\bfu\|^2,  \\
\mathrm{s.t.} \, \,  &\mathcal{L}_{\bfg} V(\bfx)\bfu + \alpha_{V} (V(\bfx)) \leq 0. \notag
\end{align}
Since $V(\bfx)$ is smooth with bounded derivatives for $\phi(\bfx) \neq \bfq$, it follows that $\bfu(\bfx)$ is Lipschitz continuous on the set $\{\bfx \in \calX : \phi(\bfx) \neq \bfq\}$.

Next, to show continuity on $\calX$, we prove that $\lim_{\bfx \rightarrow \bfx_G} \bfu(\bfx) = \boldsymbol{0}$ for any $\bfx_G \in \calX$ such that $\phi(\bfx_G) = \bfq$. Based on \eqref{eq: point_clf_derivative}, we can write
\begin{align*}
&\lim_{\bfx \rightarrow \bfx_G} |\bfu(\bfx)| = \lim_{\bfx \rightarrow \bfx_G} \left|\alpha_{V} (V(\bfx)) \frac{\mathcal{L}{\bfg}V(\bfx)^{\top}}{\| \mathcal{L}{\bfg}V(\bfx)^{\top} \|^{2}}\right| \\
&\leq \lim_{\bfx \rightarrow \bfx_G} \frac{\alpha_V(V(\bfx))}{V(\bfx)} \cdot \lim_{\bfx \rightarrow \bfx_G} \left|\frac{V(\bfx)}{\| \mathcal{L}{\bfg}V(\bfx)^{\top}\|} \right| = 0,
\end{align*}
where the last equality holds because the first limit is zero by the assumption on $\alpha_V$, and the second limit is bounded since $\frac{V(\bfx)}{| \mathcal{L}{\bfg}V(\bfx)^{\top} |}$ remains bounded as $\bfx \rightarrow \bfx_G$. 
Therefore, the control law $\bfu(\bfx)$ is continuous on $\calX$.
%
%
%
\qed
\end{proof}

The control law minimizes the norm of $\bfu$ while ensuring $\dot{V}(\bfx) < 0$. Therefore, the function $V$ in~\eqref{eq: clf_definition} together with the control law $\bfu(\bfx)$ in~\eqref{eq: min_norm_u_clf} satisfy Assumption~\ref{ass: clf_properties}. Therefore, Lemma~\ref{lemma: closed_loop_convergence} ensures that the system is guided towards the goal point $\bfq$. Since the control law $\bfu(\bfx) = \boldsymbol{0}$ whenever $\phi(\bfx) = \bfq$, the orientation $\theta(t)$ may converge to any value, depending on the initial conditions. Finally, by Theorem~\ref{theorem: interconnected_convergence}, the closed-loop dynamics interconnected with the governor dynamics has the unicycle robot track the planned path $\gamma$ in~$\bbR^2$.

\subsection{Signed Distance Function as Control Barrier Function}
\label{sec: sdf_to_cbf}


In unknown and dynamic environments, the precise computation of the barrier function $h(\bfx, t)$ or the construction of a probabilistic model is challenging. This section focuses on showing that the robot SDF \eqref{eq: robot_sdf_transformation} is a valid CBF for \eqref{eq: unicycle_model} under appropriate assumptions. 


We make the following assumptions: 

\begin{Assumption}\longthmtitle{Regularity of Robot Shape}
\label{ass: unicycle_shape_assumptions}
The robot shape $\calB_0$ has a smooth boundary and is compact. 
\end{Assumption}

Based on the robot SDF in \eqref{eq: robot_sdf_transformation}, we define the candidate CBF as: 
\begin{equation}
\label{eq: sdf_to_cbf}
h(\bfx, t) \!=\! d(\calB(\bfx), \calO(t)) \!:=\! \inf_{\bfq \in \calO(t)} d_{\calB_0}\left( \bfR(\bfx)^\top (\mathbf{q} - 
\phi(\bfx)) \right),
\end{equation}
where  \( \calO(t) \subset \mathbb{R}^2 \) is the obstacle set at time \( t \), and $\bfR(\bfx) \in \text{SO}(2)$ is the rotation matrix that describes the unicycle’s orientation. By Assumption~\ref{ass: unicycle_shape_assumptions}, we know that $d_{\calB_0}(\cdot)$ is 1-Lipschitz continuous \citep{fitzpatrick1980metric}. The presence of the infimum operator in \eqref{eq: sdf_to_cbf} ensures that the candidate CBF $h(\bfx, t)$ is Lipschitz continuous. However, it is important to note that $h(\bfx, t)$ is, in general, not smooth due to both the properties of $d_{\calB_0}(\cdot)$ and the infimum operation.

%
%

Therefore, we review some concepts from nonsmooth analysis to handle the potential non-differentiability of $h$.

\begin{definition}\longthmtitle{Generalized Gradient} 
\label{def: generalized_gradient}
Let $h: \calX \times \bbR_{\geq 0} \to \mathbb{R}$ be Lipschitz continuous near $\bfz := (\bfx, t) \in \mathbb{R}^4$, and suppose $\mathcal{Z}$ is any set of Lebesgue measure zero in $\bbR^4$. The \emph{Clarke generalized gradient} \citep{clarke1981generalized} of $h$ at $\bfz$ is defined as
\begin{equation}
\label{eq: generalized_gradient}
\partial h(\bfz) = \operatorname{co} \left\{ \lim_{i \to \infty} \nabla h(\bfz_i) \,\big|\, \bfz_i \to \bfz,\, \bfz_i \notin \Omega_h \cup \mathcal{Z} \right\},
\end{equation}
where $\operatorname{co}$ denotes the convex hull and $\Omega_h$ is the set of points where $h$ fails to be differentiable. 
\end{definition}

Using this definition, we extend the notion of time-varying control barrier functions to functions that are not necessarily differentiable everywhere.
%
%

\begin{definition}\longthmtitle{Nonsmooth Time-Varying Control Barrier Function}
\label{def: nonsmooth_cbf}
A Lipschitz continuous and regular function $h: \calX \times \bbR_{\geq 0} \to \mathbb{R}$ is called a \emph{nonsmooth time-varying control barrier function} on a set $\mathcal{H} \subset \mathbb{R}^4$ if there exists an extended class-$\mathcal{K}$ function $\alpha$ such that for all $\bfx \in \calX$ and all $t \geq 0$,
\begin{equation}
\label{eq: nonsmooth_cbf_condition}
\sup_{\bfu \in \mathcal{U}} \inf_{(\bfy, z) \in \partial h(\bfx, t)} \left\{ \bfy^\top \left( f(\bfx) + g(\bfx) \bfu \right) + z \right\} \geq -\alpha(h(\bfx, t)),
\end{equation}
where $\partial h(\bfx, t)$ is the generalized gradient of $h$. 
\end{definition}

Next, we show that $h(\bfx, t)$ in~\eqref{eq: sdf_to_cbf} is a valid time-varying CBF candidate for the dynamics in~\eqref{eq: unicycle_model} under the following assumptions. 

\begin{Assumption}\longthmtitle{Boundedness and Regularity of Obstacles}
\label{ass: unicycle_env_assumptions}
We assume that:
\begin{itemize}
    \item The initial obstacle set \( \mathcal{O}_0 \) is bounded, and the transformations \( A_{\mathcal{O}}(t) \) and \( \mathbf{b}_{\mathcal{O}}(t) \) in~\eqref{eq: obstacle_set_transformation} are uniformly bounded over \( t \).
    \item There exists a constant \( B > 0 \) such that for all \( (\mathbf{x}, t) \) and all \( (\bfy, z) \in \partial h(\mathbf{x}, t) \), it holds that $ |z | \leq B$.
\end{itemize}
\end{Assumption}

The following result establishes the boundedness of the obstacle set $\calO(t)$. Its proof is straightforward and therefore omitted.

\begin{lemma}\longthmtitle{Obstacle Set Boundedness}\label{lemma: obstacle_set_boundness}
Under Assumption~\ref{ass: unicycle_env_assumptions}, the obstacle set $\calO(t)$ is bounded for all $t \geq 0$.
\end{lemma}

%
%

Next, we show that $h(\bfx, t)$ is Lipschitz.

\begin{lemma}\longthmtitle{Lipschitzness of Non-smooth CBF}\label{lemma: cbf_lipschitz}
Under Assumptions~\ref{ass: unicycle_shape_assumptions} and~\ref{ass: unicycle_env_assumptions}, the function $h(\bfx, t)$ is Lipschitz in $(\bfx, t)$. 
\end{lemma}

\begin{proof}
For each fixed \( \mathbf{q} \in \mathcal{O}(t) \), we write
\begin{equation*}
    h_{\mathbf{q}}(\mathbf{x}) = d_{\mathcal{B}_0} \left( \mathbf{R}(\mathbf{x})^\top \left( \mathbf{q} - \phi(\mathbf{x}) \right) \right).
\end{equation*}
Since  \( \mathbf{R}(\mathbf{x}) \) and \( \phi(\mathbf{x}) \) are continuously differentiable, they are also Lipschitz continuous on the compact set \( \calX \). Let \( L_{\mathbf{R}} \) and \( L_{\phi} \) be their Lipschitz constants, respectively, so that 
$\left\| \mathbf{R}(\mathbf{x}_1) - \mathbf{R}(\mathbf{x}_2) \right\| \leq L_{\mathbf{R}} \left\| \mathbf{x}_1 - \mathbf{x}_2 \right\|$ and 
$\left\| \phi(\mathbf{x}_1) - \phi(\mathbf{x}_2) \right\| \leq L_{\phi} \left\| \mathbf{x}_1 - \mathbf{x}_2 \right\|$. Now, for any \( \mathbf{x}_1, \mathbf{x}_2 \in \calX \) and fixed \( \mathbf{q} \in \mathcal{O}(t) \), one has
\begin{align}
&\left| h_{\mathbf{q}}(\mathbf{x}_1) - h_{\mathbf{q}}(\mathbf{x}_2) \right| \notag \\
&= \left| d_{\mathcal{B}_0} \left( \mathbf{R}(\mathbf{x}_1)^\top \left( \mathbf{q} \!-\! \phi(\mathbf{x}_1) \right) \right) \!-\! d_{\mathcal{B}_0} \left( \mathbf{R}(\mathbf{x}_2)^\top \left( \mathbf{q} \!-\! \phi(\mathbf{x}_2) \right) \right) \right| \notag \\
&\leq \left\| \mathbf{R}(\mathbf{x}_1)^\top \left( \mathbf{q} - \phi(\mathbf{x}_1) \right) - \mathbf{R}(\mathbf{x}_2)^\top \left( \mathbf{q} - \phi(\mathbf{x}_2) \right) \right\|  \notag \\
&\leq \left\| \left[ \mathbf{R}(\mathbf{x}_1)^\top - \mathbf{R}(\mathbf{x}_2)^\top \right] \left( \mathbf{q} - \phi(\mathbf{x}_1) \right) \right\| + \notag \\
&\quad \left\| \mathbf{R}(\mathbf{x}_2)^\top \left( \phi(\mathbf{x}_2) - \phi(\mathbf{x}_1) \right) \right\|.
\label{eq: sdf_lipschitz_in_x}
\end{align}
where the first equality comes from the fact that $d_{\calB_0}$ is 1-Lipschitz \citep{fitzpatrick1980metric}.

For the first term in \eqref{eq: sdf_lipschitz_in_x}, we have 
\begin{multline*}
\left\| \left[ \mathbf{R}(\mathbf{x}_1)^\top - \mathbf{R}(\mathbf{x}_2)^\top \right] \left( \mathbf{q} - \phi(\mathbf{x}_1) \right) \right\|  \\
\leq L_{\mathbf{R}} \left\| \mathbf{x}_1 - \mathbf{x}_2 \right\| \cdot \left( \left\| \mathbf{q} \right\| + \left\| \phi(\mathbf{x}_1) \right\| \right).
\end{multline*}
Since $\calX$ is bounded by assumption and \( \mathcal{O}(t) \) is also bounded, cf. Lemma~\ref{lemma: obstacle_set_boundness}, there exist \( M_{\mathbf{q}}, M_{\phi} > 0 \) such that \( \left\| \mathbf{q} \right\| \leq M_{\mathbf{q}} \) and \( \left\| \phi(\mathbf{x}_1) \right\| \leq M_{\phi} \). Therefore:
\begin{multline*}
\left\| \left[ \mathbf{R}(\mathbf{x}_1)^\top - \mathbf{R}(\mathbf{x}_2)^\top \right] \left( \mathbf{q} - \phi(\mathbf{x}_1) \right) \right\| \\
\leq L_{\mathbf{R}} (M_{\mathbf{q}} + M_{\phi}) \left\| \mathbf{x}_1 - \mathbf{x}_2 \right\|.
\end{multline*}
For the second term in \eqref{eq: sdf_lipschitz_in_x}, since \( \left\| \mathbf{R}(\mathbf{x}) \right\| = 1 \), we have
\begin{align*}
\left\| \mathbf{R}(\mathbf{x}_2)^\top \left( \phi(\mathbf{x}_2) - \phi(\mathbf{x}_1) \right) \right\| \leq L_{\phi} \left\| \mathbf{x}_1 - \mathbf{x}_2 \right\|.
\end{align*}
Combining the bounds from the two terms, we conclude
\begin{equation}
    \left| h_{\mathbf{q}}(\mathbf{x}_1) - h_{\mathbf{q}}(\mathbf{x}_2) \right| \leq L_h \left\| \mathbf{x}_1 - \mathbf{x}_2 \right\|,
\label{eq: Lipschitz_constant_h}
\end{equation}
with $L_h = (L_{\mathbf{R}} (M_{\mathbf{q}} + M_{\phi}) + L_{\phi})$.
%
%
As established in \citet[Proposition 9.10]{rockafellar2009variational}, the pointwise infimum of a collection of Lipschitz functions is Lipschitz with a constant bounded by the maximum of the individual Lipschitz constants. Therefore, since \(h_{\mathbf{q}}(\mathbf{x})\) is uniformly Lipschitz with constant \(L_h\) and \(\mathcal{O}(t)\) is bounded, the infimum \(h(\mathbf{x}, t) = \inf_{\mathbf{q} \in \mathcal{O}(t)} h_{\mathbf{q}}(\mathbf{x})\) is Lipschitz with the same constant \(L_h\).

Furthermore, by Assumption~\ref{ass: unicycle_env_assumptions}, $h(\bfx, t)$ is Lipschitz in~$t$ with constant $B$. 
\qed
\end{proof}

Next, our goal is to show that for all \( (\bfx, t) \in \calX \times \bbR_{\geq 0}\), there exists a control input \( \mathbf{u} \in \mathcal{U} \) such that the CBF condition \eqref{eq: nonsmooth_cbf_condition} holds.

\begin{proposition}\longthmtitle{Sufficient Conditions for Signed Distance Function as a Nonsmooth Time-Varying Control Barrier Function}
\label{prop: sdf_as_cbf_general}
Under Assumptions~\ref{ass: unicycle_shape_assumptions} and~\ref{ass: unicycle_env_assumptions}, the function \( h(\mathbf{x}, t) \) defined in~\eqref{eq: sdf_to_cbf} satisfies the nonsmooth CBF condition~\eqref{eq: nonsmooth_cbf_condition} for the unicycle dynamics~\eqref{eq: unicycle_model}, provided that $\inf_{(\mathbf{0}, z) \in \partial h(\mathbf{x}, t)} z \ge -\alpha_h(h(\bfx, t))$. 
\end{proposition}

\begin{proof}
Using the unicycle dynamics~\eqref{eq: unicycle_model}, for any \( (\bfy,z) \in \partial h(\mathbf{x}, t) \), the expression inside the infimum in \eqref{eq: nonsmooth_cbf_condition} becomes
\begin{align}
\label{eq: unicycle_nonsmooth_cbf_conditions}
&\bfy^\top \left( f(\mathbf{x}) + g(\mathbf{x}) \mathbf{u} \right) + z =  \\
&\bfy^\top \left( \begin{bmatrix} \cos\theta \\ \sin\theta \\ 0 \end{bmatrix} v + \begin{bmatrix} 0 \\ 0 \\ 1 \end{bmatrix} \omega \right) + z \nonumber = a_v v + a_\omega \omega + z, \label{eq: h_dot_expression}
\end{align}
where
\begin{align*}
a_v &= \bfy^\top \begin{bmatrix} \cos\theta \\ \sin\theta \\ 0 \end{bmatrix}, \quad a_\omega = \bfy^\top \begin{bmatrix} 0 \\ 0 \\ 1 \end{bmatrix}.
\end{align*}
From Lemma~\ref{lemma: cbf_lipschitz}, \( \| y \| \leq L_h \). It follows that $| a_v |, \; | a_\omega | \leq L_h$. We now consider two cases:

\textbf{Case 1:} $(\boldsymbol{0}, z) \notin \partial h(\bfx, t)$.
In this case, \( \bfy \neq \boldsymbol{0} \), so at least one of \( a_v \) or \( a_\omega \) is non-zero. Define $\varepsilon = a_v^2 + a_\omega^2 > 0$ and choose
\begin{equation}
\label{eq: control_inputs_case1}
v = \frac{a_v}{\varepsilon} \beta, \quad \omega = \frac{a_\omega}{\varepsilon} \beta,
\end{equation}
where \( \beta > 0 \) is a scaling factor. Then, since \( z \in [ -B, B ] \) (cf. Assumption~\ref{ass: unicycle_env_assumptions}), we get 
\begin{align}
\label{eq: non_smooth_cbf_bound}
a_v v + a_\omega \omega + z &= \left( \frac{a_v^2 + a_\omega^2}{\varepsilon} \right) \beta + z \geq \beta - B.
\end{align}
The scaling factor $\beta$ can be chosen sufficiently large to ensure that the nonsmooth CBF condition~\eqref{eq: nonsmooth_cbf_condition} is satisfied. 

\textbf{Case 2:} $(\boldsymbol{0}, z) \in \partial h(\bfx, t)$. In this case, the spatial gradient of \( h \) vanishes, and the expression in \eqref{eq: unicycle_nonsmooth_cbf_conditions} simplifies to:
\[
\bfy^\top \left( f(\mathbf{x}) + g(\mathbf{x}) \mathbf{u} \right) + z = z.
\]
The nonsmooth CBF condition~\eqref{eq: nonsmooth_cbf_condition} holds in this case because of the hypotheses of the statement. 
%
%
\qed
\end{proof}

Regarding the requirement of Proposition~\ref{prop: sdf_as_cbf_general}, intuitively, when \( \bfy = \boldsymbol{0} \), the spatial gradient does not provide directional guidance for the robot, and the safety condition depends solely on the temporal evolution of \( h \), i.e., the behavior of the obstacles. The requirement $\inf_{(\mathbf{0}, z) \in \partial h(\mathbf{x}, t)} z \ge -\alpha_h(h(\bfx, t))$ ensures that the nonsmooth CBF condition is satisfied in this case.

In practice, we need to ensure that the control inputs \( v \) and \( \omega \) satisfy the control constraints \( | v | \leq v_{\max} \) and \( | \omega | \leq \omega_{\max} \). Based on \eqref{eq: control_inputs_case1} and \eqref{eq: non_smooth_cbf_bound}, this means that the bound $B$ on the partial time derivative of the distance function must be compatible with the control bounds to ensure the inequality \eqref{eq: nonsmooth_cbf_condition} holds. In other words, to guarantee safety, the obstacles must move at speeds compatible with the motion capabilities of the robot.

\subsection{Sensor-Based CBF Sample Selection}
\label{sec: unicycle_dr_cbf_samples}

As the robot navigates through unknown and dynamic environments relying on noisy LiDAR measurements, precisely determining the obstacle set \( \mathcal{O}(t) \) and thus computing \( h(\mathbf{x}, t) \) is not feasible. Therefore, we leverage the distributionally robust CBF formulation from Sec.~\ref{sec: dro_cbf_navigation} and obtain samples \( \{ \boldsymbol{\xi}_i \}_{i=1}^N \), where recall
\begin{align*}
    \bfxi = [\bfF^\top(\bfx)\nabla_{\bfx} h(\bfx, t), \; \alpha_h(h(\bfx, t)), \; \frac{\partial h(\bfx, t)}{\partial t}], 
\end{align*}
directly from the distance measurements \( \boldsymbol{\eta}(\phi(\mathbf{x}), \mathbf{R}(\bfx)) = [\eta_1(\phi(\mathbf{x}), \mathbf{R}(\bfx)), \dots, \eta_K(\phi(\mathbf{x}), \mathbf{R}(\bfx))]^\top \). Note that the system dynamics $\bfF(\bfx)$ and $\calK_{\infty}$ function $\alpha_h$ are assumed to be known and deterministic.

In practical robotic systems, sensor measurements, such as LiDAR readings, are collected at discrete intervals determined by the sensor's frequency. These measurements must be processed, along with state estimation and control synthesis, within the time constraints of the control loop to ensure real-time applicability.

To account for state estimation uncertainty, we consider \( M \) samples of the robot's estimated pose, denoted as \( \{ \mathbf{x}^{(j)} \}_{j=1}^M \), where each sample represents a possible true pose of the differential-drive robot given its localization error distribution. For each sample \( \mathbf{x}^{(j)} \), we have \( K \) corresponding LiDAR measurements in the robot's local frame, \( \{ \boldsymbol{\eta}_i^{(j)} \}_{i=1}^K \). These measurements are transformed into the global frame using the estimated pose, yielding LiDAR hit $\mathbf{q}_i^{(j)}(t) \in \bbR^2$ with $i$-th LiDAR measurement of pose $\bfx^j$ at time $t$.
By aggregating these transformed measurements, we construct a comprehensive set of LiDAR points in the global frame:
\begin{equation}
\label{eq: aggregated_lidar_points}
\mathcal{P}(t) = \left\{ \mathbf{q}_i^{(j)}(t) \mid i = 1, \dots, K;\; j = 1, \dots, N \right\} \subset \mathbb{R}^2,
\end{equation}
resulting in a total of \( N \times K \) points. This set effectively captures the combined uncertainties in both state estimation (localization) and sensor measurements. To account for dynamic obstacles, we may obtain estimates of the time derivatives \( \{ \frac{\partial \mathbf{q}_{i}^j}{\partial t}\} \)
%
%
using a radar sensor, Doppler LiDAR, or a LiDAR velocity estimation algorithm \citep{Lidar_velocity_estimate}.

Since we cannot compute \( h(\mathbf{x}, t) \) exactly, we use the available LiDAR points set $\calP(t)$ to approximate the barrier function and its gradients. From this aggregated data, we select \( N \) samples that minimize the following criterion:
\begin{align}
\label{eq: sample_selection_criterion}
\frac{\partial}{\partial t} &d_{\mathcal{B}_0}\left( \mathbf{R}(\mathbf{x})^\top \left( \mathbf{q}_i(t) - \phi(\mathbf{x}) \right) \right) + \notag \\
&\alpha_h\left( d_{\mathcal{B}_0}\left( \mathbf{R}(\mathbf{x})^\top \left( \mathbf{q}_i(t) - \phi(\mathbf{x}) \right) \right) \right),
\end{align}
%
%
for \( i = 1, \dots, K \). Note that the generalized gradient $\partial h_i(\mathbf{x}, t)$ can be computed correspondingly by utilizing the robot SDF $d_{\calB_0}$. This criterion effectively identifies obstacle points where the combined effect of the barrier function's rate of change and its current value is most critical, highlighting the samples where the safety constraint, as defined in \eqref{eq:tvcbc_define}, is closest to being violated.

With this procedure, given each $\bfx \in \calX$, we have available samples $\{\bfxi_i\}_{i \in [N]}$ for the vector $\bfxi$.

\begin{remark}\longthmtitle{Sample selection for static environments}
{\rm
If $\frac{\partial h}{\partial t} = 0$, the selection of the $N$ samples reduces to finding the minimum $N$ values of $h_i(\bfx)$, representing the distance of $N$ closest detected obstacle points to the robot. This is illustrated in Fig.~\ref{fig:dr_safety}.}  \demo
\end{remark}

With the sample collection strategy for $\{\bfxi_i\}_{i \in [N]}$ and the CLF design in~\eqref{eq: clf_definition}, we formulate the CLF-DR-CBF optimization problem in \eqref{eq: clc_cbc_dr_result} to synthesize control inputs for the unicycle robot.

The differential-drive specific CLF is combined with the DR-CBF constraint, which is based on the CBF samples obtained from the range sensor measurements and estimated robot poses. This approach enables safe robot navigation in unknown dynamic environments, while directly utilizing the sensor measurements to ensure safety.

\subsection{Application to Other Robotic Systems}


Our CLF-DR-CBF formulation can be applied to other control-affine robotic systems described by \eqref{eq: dynamic}.
To apply this formulation to a different system, two key components are required: a CLF 
\(V(\bfx)\) and a CBF $h(\bfx,t)$:
\begin{itemize}
    \item \textbf{CLF:} $V(\bfx)$ should encode the stability with respect to a desired equilibirum. In many cases, it is more practical to design a CLF for a subset of the state space, such as position, rather than the full state, particularly for underactuated systems like unicycles;
    \item \textbf{CBF:} $h(\bfx, t)$ defines the safety specifications based on the robot's geometry and its interaction with the environment. The CBF samples can be obtained, for example, from onboard sensory data.
\end{itemize}

Note that systematic approaches exist for constructing CLFs, including sum-of-squares programming~\citep{jarvis_clf_sos,HD-CJ-HZ-AC:24} and neural network-based approaches~\citep{Chang2019NeuralLC, boffi2021learning}.  Additionally, specific classes of systems, such as linear systems and differentially flat systems, are more tractable in this context, as discussed in~\citep{mestres2024safe}. Once these components are defined, the CLF-DR-CBF formulation can be directly applied to synthesize a distributionally robust safe and stable controller.

%
%

For instance, a table-top manipulator evolves in a high-dimensional configuration space, where safety constraints must prevent both obstacle collisions and self-collisions. A CLF can be designed to stabilize the manipulator's end-effector toward a target pose in task space. This typically involves defining a Lyapunov function over task-space error (e.g., pose error in SE(3)) and mapping its gradient back to joint space via the manipulator's Jacobian. Alternatively, inverse kinematics can be used to convert the task-space target into a reference joint configuration, around which a CLF is constructed. The CBF, on the other hand, can be constructed using signed distance functions \citep{koptev_neural_jsdf_2022, li2024representing}, which encode collision avoidance constraints. CBF samples can similarly be obtained from onboard sensory data (e.g., images) and robot state estimations.

\section{Evaluation}
\label{sec: eva}

\begin{figure}[t]
  \centering
  \subcaptionbox{Static environment\label{fig:3a}}{\includegraphics[width=0.49\linewidth]{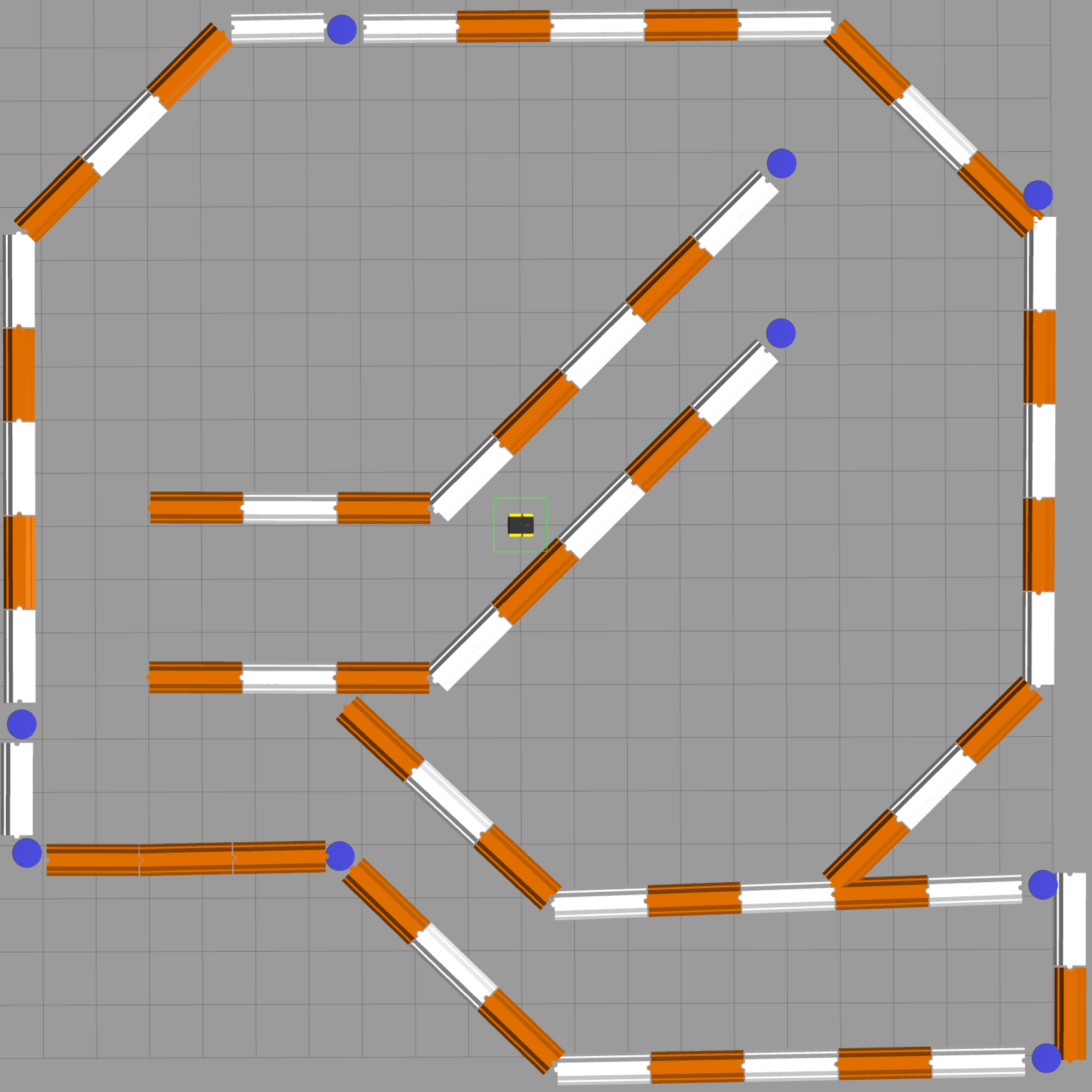}}%
  \hfill%
  \subcaptionbox{Dynamic environment\label{fig:3b}}{\includegraphics[width=0.49\linewidth]{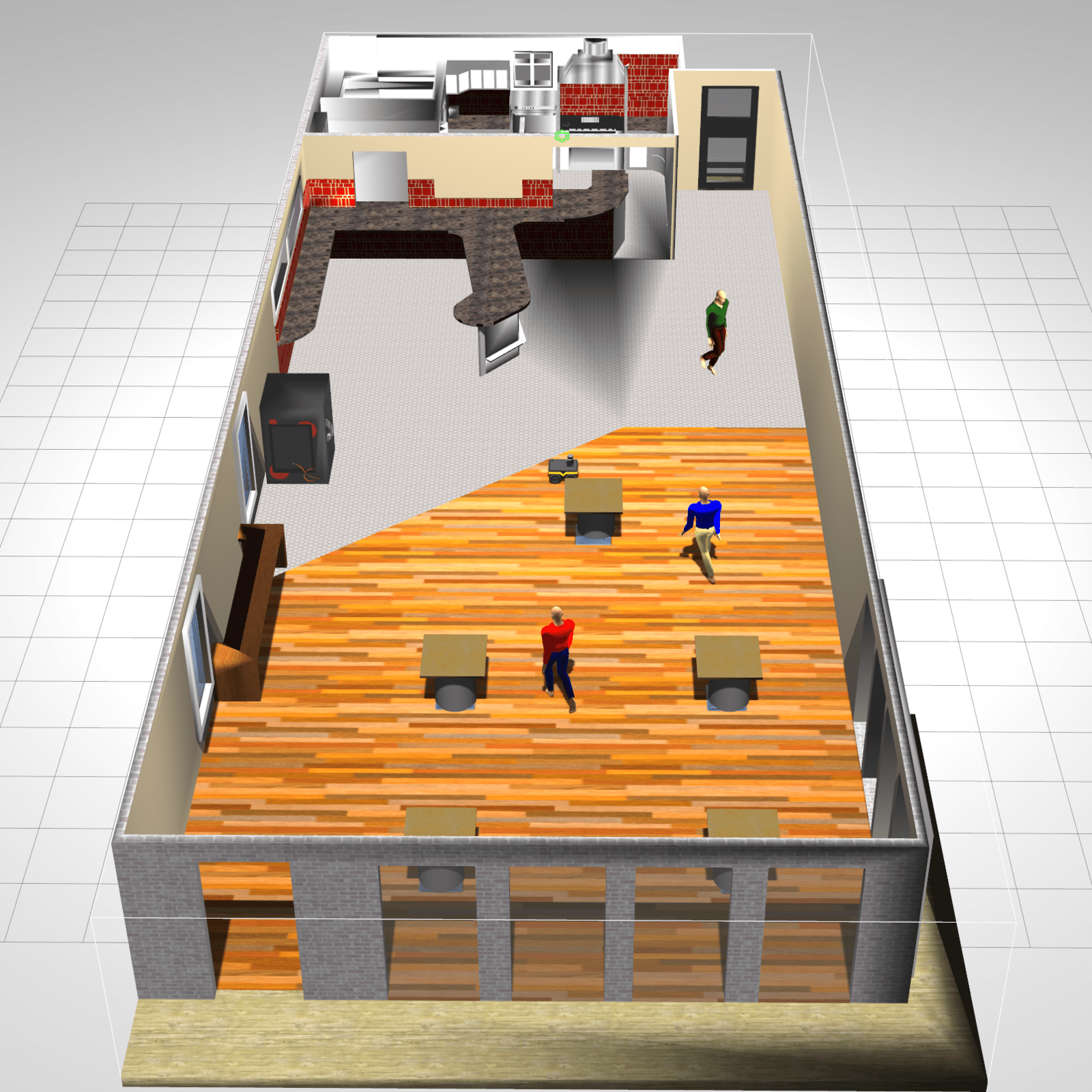}}\\
  \caption{Simulated environments in Gazebo.}
\end{figure}

In this section, we evaluate our CLF-DR-CBF QP formulation through several simulation and real-world experiments.

We compare our approach with two other safe control strategies, the nominal CLF-CBF QP in \eqref{eq: clf_cbf_qp} and a CLF-Gaussian Process (GP)-CBF second-order cone program (SOCP) \citep{Long2022RAL}. The nominal CLF-CBF QP approach utilizes the closest LiDAR point to define a single CBF $h(\bfx, t)$ and its gradients at each time step.
%
%
In the CLF-GP-CBF SOCP method, a real-time GP-SDF model \citep{log-gpis} of the unknown environment is constructed using LiDAR data, from which the CBF, its gradient, and uncertainty information are determined.  While the GP-SDF mapping process contributes to safety by continuously updating the environment representation, it also incurs computational overhead due to the real-time update of the GP-SDF model. For a fair evaluation, we solve each of the optimization programs to generate control signals using the Interior Point Optimizer through the CasADi framework \citep{Andersson2019}.

In the following simulations and experiments, a consistent set of parameters is utilized to ensure reproducibility of the results. The linear velocity is constrained in $[-1.2, 1.2]$ m/s and the angular velocity is limited within $[-1, 1]$ rad/s. The nominal control input $\bfk(\bfx)$ is set to $[1.2, 0]^\top$, directing the robot to move forward at $1.2$ m/s. While we use a constant nominal controller, note that our formulation supports more complex, state-dependent nominal controllers. The scaling factor is $\lambda = 50$. Table~\ref{table: parameters} summarizes other parameter values. 

\begin{table}[ht]
\caption{Simulation and experiment parameters. The class $\calK$ function $\alpha_V$ for CLF and the class $\calK_{\infty}$ for CBF are assumed to be linear. The parameters $k_v$ and $k_{\omega}$ are control gains for linear and angular velocities, respectively, $\epsilon$ the risk tolerance of the CLF-DR-CBF QP formulation, and $N$ the DR-CBF sample size.}
\label{table: parameters}
\centering
\begin{tabular}{|l|c|c|c|c|c|c|}
\hline
\textbf{Parameters} & $\alpha_V$ & $\alpha_h$ &  $k_v$ & $k_{\omega}$  & $\epsilon$ & $N$  \\ \hline
\textbf{Value}  & 1.0    & 1.5    & 0.05    & 0.4    & 0.1    & 5   \\ \hline
\end{tabular}
\end{table}

The layout of this section is as follows. Sec.~\ref{sec: eva_static_sim} presents simulation results and compares with the two baseline approaches in static environments (e.g., Fig.~\ref{fig:3a}). In Sec.~\ref{sec: eva_dynamic_sim}, we evaluate our approach in dynamic Gazebo environments \citep{gazebo} with pedestrians, shown in Fig.~\ref{fig:3b}. Finally, in Sect.~\ref{sec: eva_real_world}, we test our CLF-DR-CBF QP formulation in dynamic cluttered real-world environments. In all results, the $A^*$ algorithm is employed for path planning, operating at a frequency of $5$ Hz. Our CLF-DR-CBF QP formulation is used for real-time safe navigation, running at $50$ Hz.


\subsection{Simulated Static Environments}
\label{sec: eva_static_sim}

\begin{figure*}[ht]
\centering
\subfloat[Robot Shape]{\includegraphics[width=0.23\linewidth]{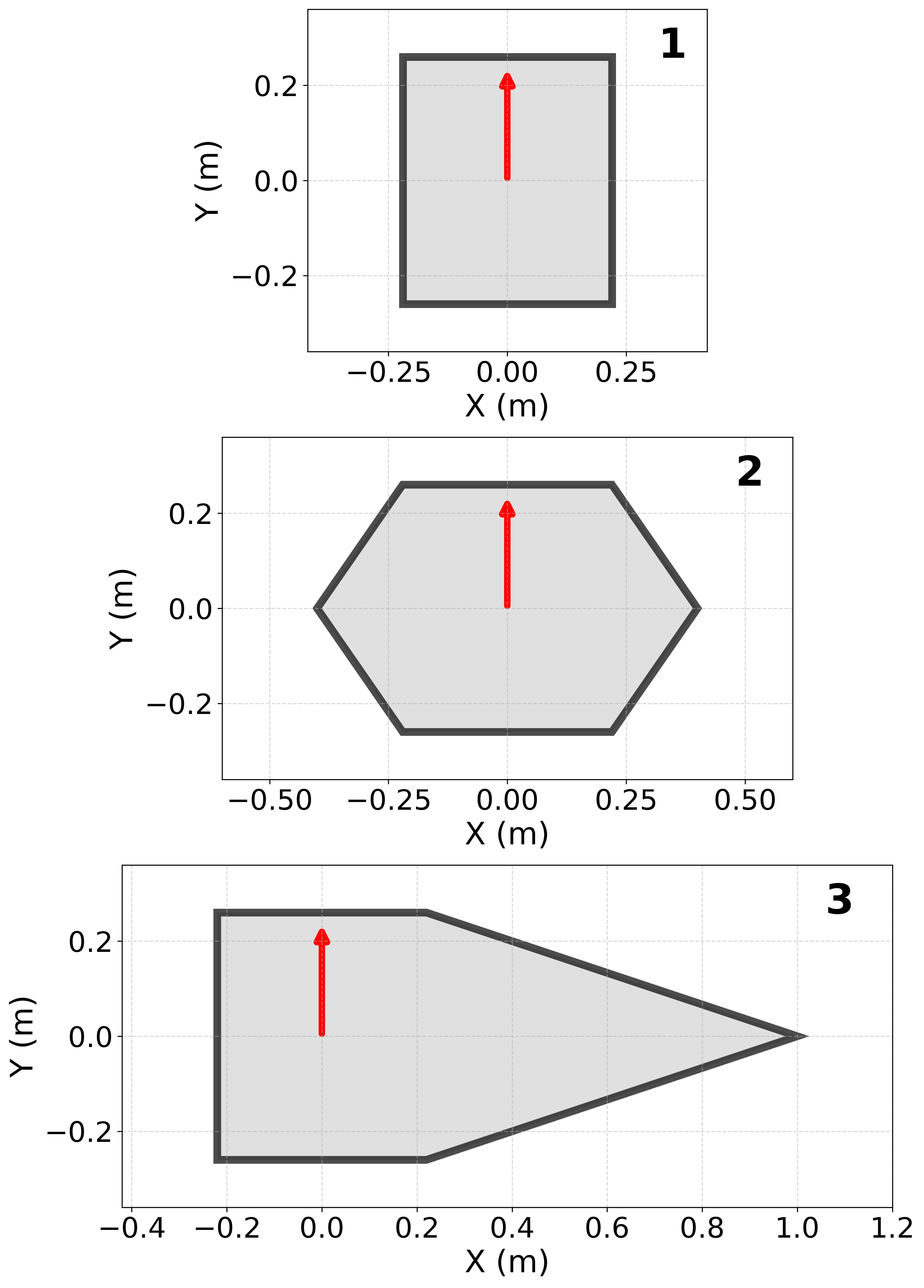}\label{fig:robot_shape_visual}}
\hfill
\subfloat[Trajectory under different robot shapes]{\includegraphics[width=0.24\linewidth]{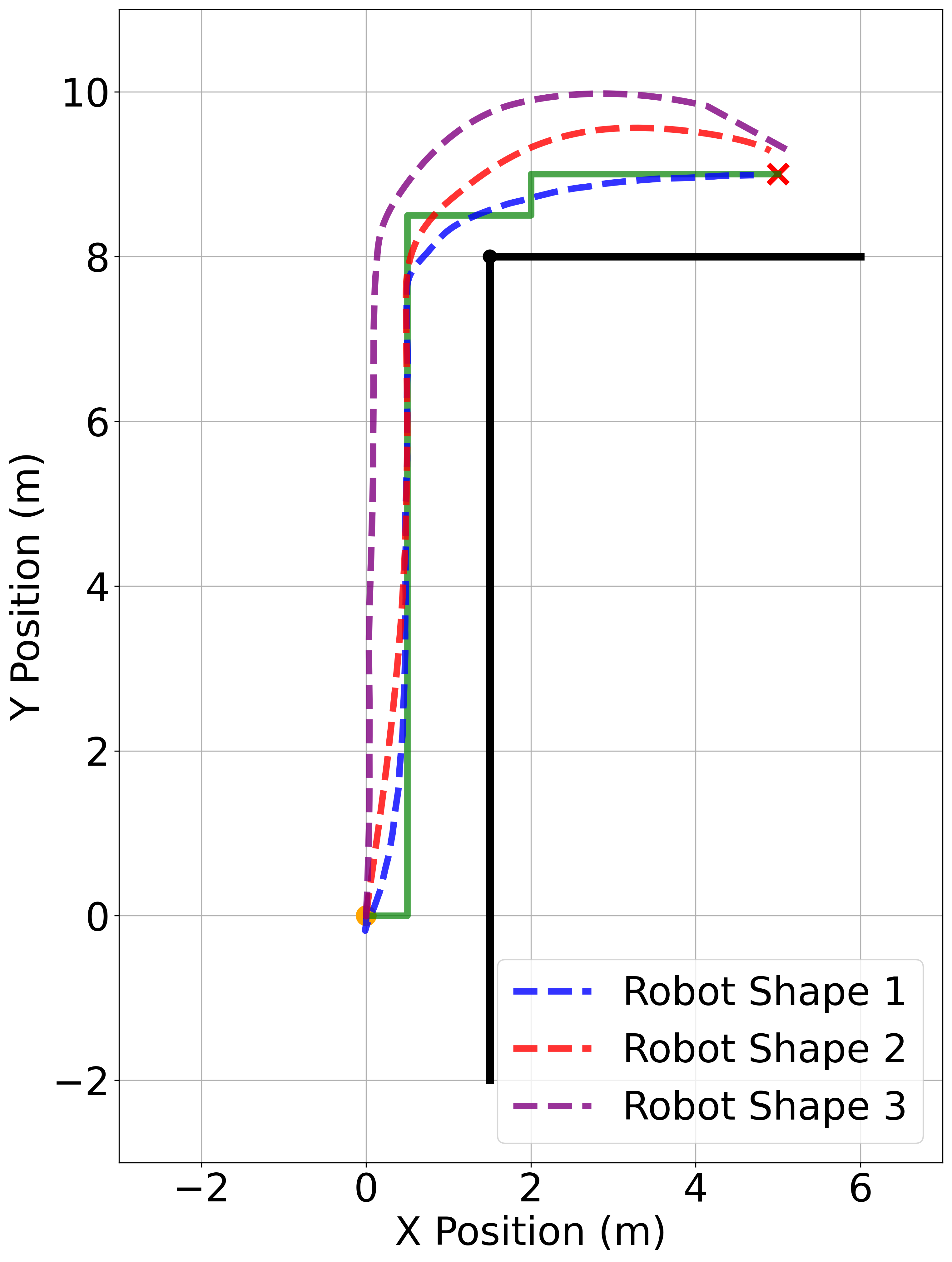}\label{fig:robot_shape_compare}}
\hfill
\subfloat[Trajectory under varying LiDAR Noise]{\includegraphics[width=0.24\linewidth]{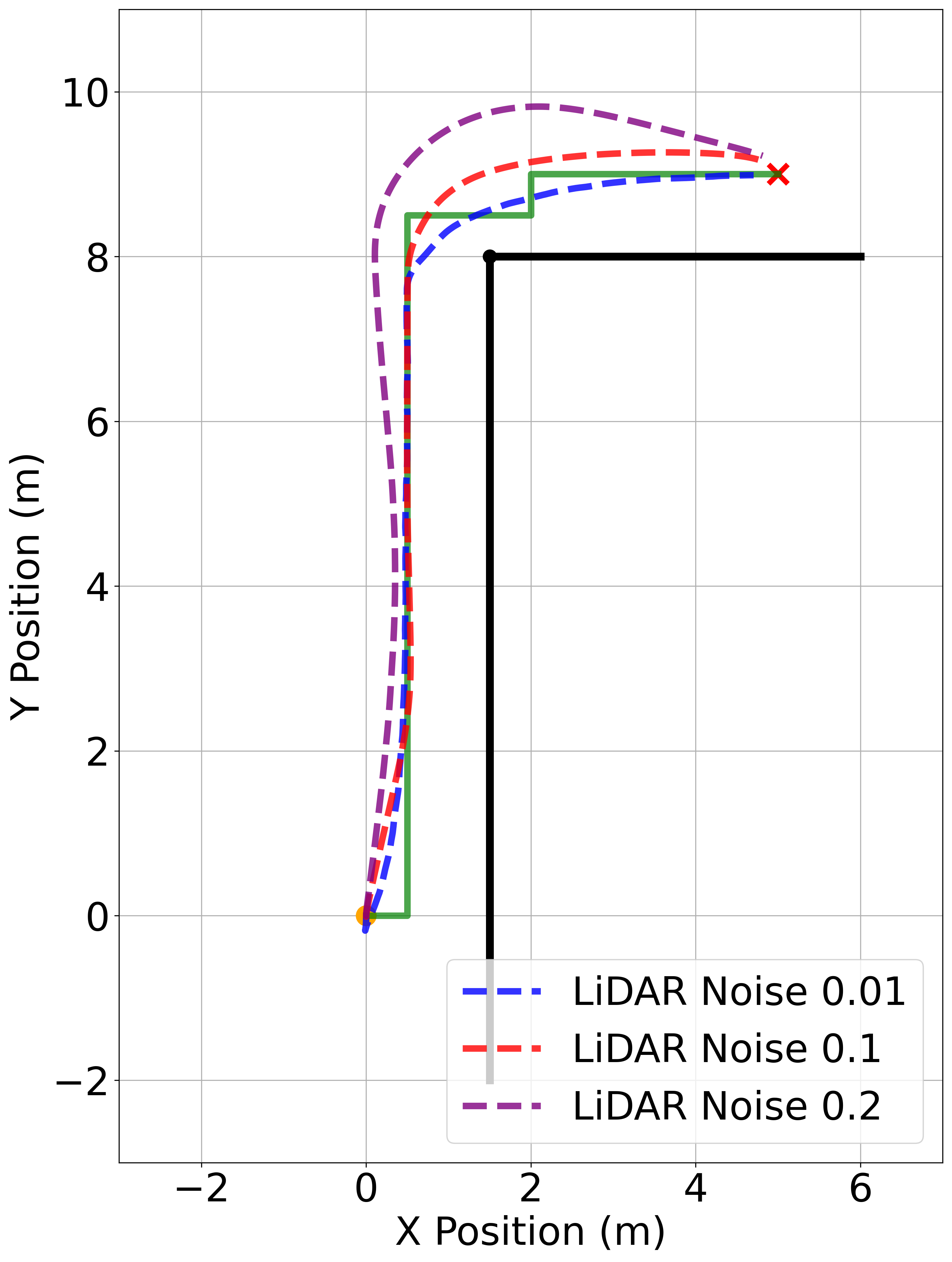}\label{fig:lidar_error_compare}}
\hfill
\subfloat[Trajectory under varying Localization Noise]{\includegraphics[width=0.24\linewidth]{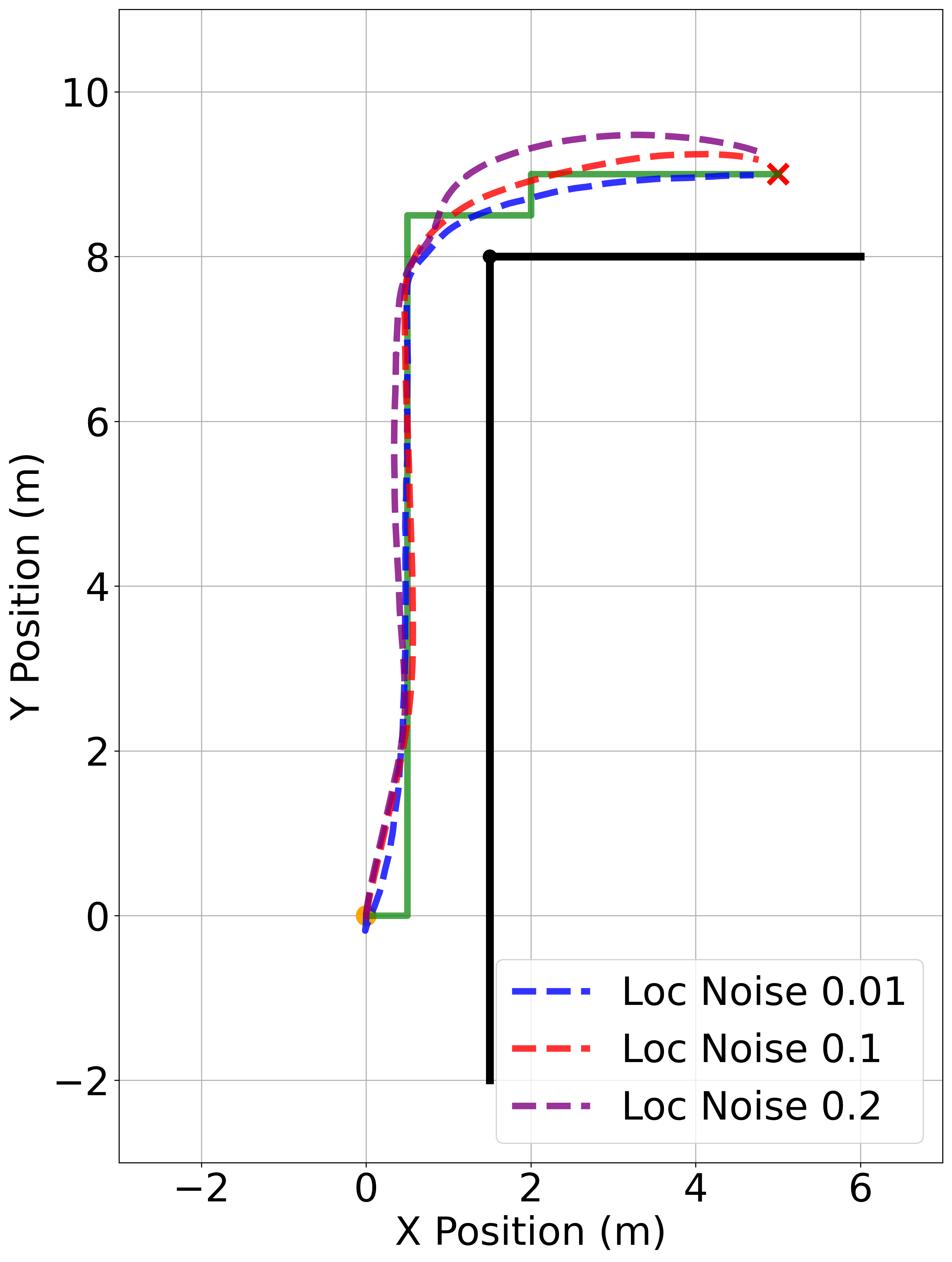}\label{fig:localization_error_compare}}\\
\includegraphics[width=0.75\linewidth,,trim={0 8mm 0 0mm},clip]{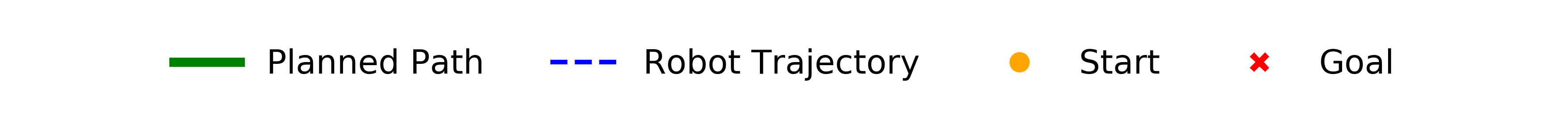}
\caption{Evaluation of robot navigation performance under varying conditions. The simulations demonstrate the system's behavior across different robot shapes and noise conditions, with (a) showing the tested robot shapes, (b) illustrating the impact of shape variation on navigation, and (c-d) analyzing the effects of sensor and localization uncertainties on trajectory execution. For (c) and (d), Shape 1 from (a) is used as the test case.}
\label{fig:sim_robot_trajectories}
\vspace{-1ex}
\end{figure*}

\begin{table*}[htbp]
\caption{Performance metrics for static environments under varying LiDAR and localization noise. The metrics include stuck rate and collision rate out of 1000 trials and tracking error (m) (mean $\pm$ std). The stuck rate reflects cases of infeasibility or the robot being trapped in local optima. The robot starts from the origin, and goals are placed at least 10 meters away in Figure~\ref{fig:3a}.}
\label{tab:static_env_metrics}
\centering
\resizebox{\linewidth}{!}{%
\begin{tabular}{|c|c|c|c|c|c|}
\hline
\textbf{LiDAR Noise} & \textbf{Localization Noise} & \textbf{Method} & \textbf{Stuck Rate} & \textbf{Collision Rate} & \textbf{Tracking Error (m)} \\ \hline
\multirow{6}{*}{$\sigma = 0.001$} & \multirow{3}{*}{$\sigma = 0.01$} & CLF-DR-CBF QP & 0.0 & 0.0 & 1.33 $\pm$ 0.19 \\ 
 &  & CLF-GP-CBF SOCP & 0.2 & 0.0 & 1.38 $\pm$ 0.31 \\ 
 &  & Nominal CLF-CBF QP & 0.0 & 0.2 & 1.26 $\pm$ 0.20 \\ \cline{2-6}
 & \multirow{3}{*}{$\sigma = 0.05$} & CLF-DR-CBF QP & 0.1 & 0.0 & 1.88 $\pm$ 0.35 \\ 
 &  & CLF-GP-CBF SOCP & 0.3 & 15.8 & 2.16 $\pm$ 0.47 \\ 
 &  & Nominal CLF-CBF QP & 0.1 & 15.5 & 1.92 $\pm$ 0.34 \\ \hline
\multirow{6}{*}{$\sigma = 0.05$} & \multirow{3}{*}{$\sigma = 0.01$} & CLF-DR-CBF QP & 0.0 & 0.0 & 1.77 $\pm$ 0.38 \\ 
 &  & CLF-GP-CBF SOCP & 25.1 & 0.2 & 1.97 $\pm$ 0.52 \\ 
 &  & Nominal CLF-CBF QP & 1.7 & 4.4 & 1.82 $\pm$ 0.41 \\ \cline{2-6}
 & \multirow{3}{*}{$\sigma = 0.05$} & CLF-DR-CBF QP & 1.8 & 0.7 & 2.22 $\pm$ 0.57 \\ 
 &  & CLF-GP-CBF SOCP & 31.3 & 13.9  & 2.58 $\pm$ 0.78 \\ 
 &  & Nominal CLF-CBF QP & 3.3 & 25.8 & 2.23 $\pm$ 0.59 \\ \hline
\multirow{6}{*}{$\sigma = 0.1$} & \multirow{3}{*}{$\sigma = 0.01$} & CLF-DR-CBF QP & 1.2 & 0.0 & 2.31 $\pm$ 0.66 \\ 
 &  & CLF-GP-CBF SOCP & 60.7 & 0.0 & 2.55 $\pm$ 0.79 \\ 
 &  & Nominal CLF-CBF QP & 4.7 & 7.7 & 2.52 $\pm$ 0.72 \\ \cline{2-6}
 & \multirow{3}{*}{$\sigma = 0.05$} & CLF-DR-CBF QP & 6.6 & 1.8 & 2.61 $\pm$ 0.75 \\ 
 &  & CLF-GP-CBF SOCP & 65.4 & 22.0 & 2.78 $\pm$ 0.91 \\ 
 &  & Nominal CLF-CBF QP & 8.9 & 58.3 & 2.55 $\pm$ 0.69 \\ \hline
\end{tabular}}
\end{table*}


The first set of simulations aims to validate the robustness and adaptability of the proposed CLF-DR-CBF QP formulation in static environments.

\textbf{Hypothesis:}
Our CLF-DR-CBF QP method can ensure safe navigation by dynamically adjusting to varying robot shapes and compensating for sensor and localization noise.

\textbf{Setup:}  
The robot is tasked to follow the planned path while avoiding obstacles. Gaussian noise with varying standard deviations (\( \sigma \)) is added to the LiDAR measurements, while the localization error is modeled as a Gaussian random vector with varying standard deviation levels as well. Diverse robot shape geometries (Figure~\ref{fig:robot_shape_visual}) are considered to evaluate the method's adaptability. Unless otherwise noted, the default robot shape is shape 1, representing the original Jackal robot. The trajectories are analyzed under different conditions to assess safety, adaptability, and smooth navigation.

\textbf{Results and discussion:}
Figure~\ref{fig:robot_shape_compare} demonstrates the trajectories for each robot shape. The results show that our approach dynamically adjusts the robot's path based on its geometry, ensuring safety by maintaining adequate clearance from obstacles while achieving smooth navigation. For symmetric shapes like robot shape 1 (original Jackal robot), the executed path is closer to the obstacle boundary. In contrast, for highly asymmetric shapes like robot shape 3, the executed path deviates more significantly to account for the robot's geometry to ensure safety. Unless otherwise noted, the simulations presented in the following are performed using robot shape 1.

Figures~\ref{fig:lidar_error_compare} and~\ref{fig:localization_error_compare} explore the effects of sensor and localization noise on navigation performance. In Figure~\ref{fig:lidar_error_compare}, our CLF-DR-CBF QP method successfully compensates for sensor inaccuracies, maintaining safety even at higher noise levels (\( \sigma = 0.2 \)). Similarly, Figure~\ref{fig:localization_error_compare} evaluates the impact of localization noise on trajectory execution, where the robot effectively adapts to inaccuracies in state estimation, showcasing the robustness of our distributionally robust formulation.

We next evaluate the performance of our CLF-DR-CBF QP formulation in static environments simulated in Gazebo (Fig.~\ref{fig:3a}), comparing it with two baseline approaches: Nominal CLF-CBF QP and CLF-GP-CBF SOCP \citep{Long2022RAL}.

\textbf{Hypothesis:}  Our CLF-DR-CBF QP formulation ensures robust and safe navigation under varying levels of sensor and localization noise. Compared to baseline methods, our approach should achieve lower failure rates (stuck and collision) and demonstrate superior adaptability to noise conditions without compromising computational efficiency.

\textbf{Setup:}  
The simulations are conducted in a static Gazebo environment where the robot is tasked to achieve a predefined goal while avoiding obstacles. Gaussian noise with standard deviation \( \sigma \) (ranging from 0.001 to 0.1) is added to the LiDAR measurements, and localization noise with \( \sigma \) values up to 0.05 is introduced. For each noise level, 1000 trials are conducted with randomly placed goal locations at least 10 meters away from the robot's starting position. The evaluation metrics include stuck rate, collision rate, and average tracking error (mean $\pm$ std), as summarized in Table~\ref{tab:static_env_metrics}. 

The stuck rate and collision rate together determine the success rate, representing the percentage of trials where the robot successfully reaches the goal without safety violations. The stuck rate captures two failure modes. First, the optimization program may become infeasible due to large uncertainties, particularly in the GP-CBF method, where the GP-SDF map's high variance at large LiDAR noise levels often renders the SOCP problem infeasible. Second, the robot may become trapped in a corner (local minimum), unable to make further progress toward the goal.

\begin{table*}[htbp]
\caption{Computation time comparison between different control approaches (in seconds). The values represent the mean $\pm$ standard deviation of the computation time along the robot trajectory. The total computation time for each method is the sum of the GP map training time (if applicable), inference time, and control synthesis solver time. The CLF-DR-CBF QP and Nominal CLF-CBF QP methods have similar total computation time, as they do not require map updates. For these two methods, the inference time refers to processing the LiDAR data as CBF samples and corresponding gradients. The CLF-GP-CBF SOCP method has the highest total computation time due to the additional overhead of GP map training.}
\label{tab:computation_time}
\centering
\begin{tabular}{|l|c|c|c|c|}
\hline
Method & Map Training & Inference & Controller Solver & Total Computation Time \\
\hline
CLF-DR-CBF QP & 0 & 0.0002 & 0.0071 $\pm$ 0.0022 & 0.0073 $\pm$ 0.0022 \\
\hline
CLF-GP-CBF SOCP & 0.0086 $\pm$ 0.0031 & 0.0003 & 0.0096 $\pm$ 0.0028 & 0.0185 $\pm$ 0.0059 \\
\hline
Nominal CLF-CBF QP& 0 & 0.0002 & 0.0064 $\pm$ 0.0023 & 0.0066 $\pm$ 0.0023 \\
\hline
\end{tabular}
\vspace{-2ex}
\end{table*}

\textbf{Results and discussion:}  
The results in Table~\ref{tab:static_env_metrics} highlight the robustness of the CLF-DR-CBF QP method, which consistently achieves low stuck and collision rates, even under high noise levels. In contrast, the GP-CBF and Nominal CLF-CBF QP methods exhibit significant performance degradation in challenging noise conditions.

As localization noise increases, the two baseline methods are more prone to collisions. This is because their formulations do not explicitly account for uncertainties in the robot state estimation. Our CLF-DR-CBF QP method remains robust by explicitly addressing localization uncertainties in its formulation. Similarly, at higher LiDAR noise levels, the GP-CBF method struggles due to significant variance in the GP-SDF estimation, often rendering its optimization program infeasible. The CLF-DR-CBF QP method, by directly using LiDAR measurements and incorporating distributionally robust constraints, avoids reliance on explicit map construction, exhibiting a higher probability of reliable performance at higher noise levels. Table~\ref{tab:computation_time} highlights the computational efficiency of the proposed method. The CLF-DR-CBF QP formulation achieves computation times comparable to the Nominal CLF-CBF QP method while significantly outperforming the GP-CBF method. This advantage stems from the CLF-DR-CBF QP method's direct use of LiDAR measurements without requiring computationally expensive GP map construction.



Overall, these results demonstrate the robustness and efficiency of the proposed CLF-DR-CBF QP formulation. By directly handling noisy sensor data and avoiding reliance on explicit map reconstruction, the method effectively balances computational efficiency and robust safety, making it suitable for real-world applications where sensing noise, localization noise, and computational constraints are significant challenges.

\subsection{Simulated Dynamic Environments}
\label{sec: eva_dynamic_sim}

\begin{table*}[htbp]
\caption{Performance metrics for dynamic environments over 1000 trials. The metrics include success rate, stuck rate, collision rate, and task completion time (mean ± std). The CLF-DR-CBF QP outperforms both baselines in terms of success rate and collision avoidance, demonstrating its robustness in dynamic settings.}
\label{tab:dynamic_metrics}
\centering
\begin{tabular}{|l|c|c|c|c|}
\hline
\textbf{Method} & \textbf{Success Rate (\%)} & \textbf{Stuck Rate (\%)} & \textbf{Collision Rate (\%)} & \textbf{Time (s)} \\ \hline
CLF-DR-CBF QP & \textbf{93.2} & \textbf{5.1} & \textbf{1.7} & $10.7 \pm 2.2$ \\ \hline
CLF-GP-CBF SOCP & 60.5 & 36.3 & 3.2 & $13.6 \pm 2.9$ \\ \hline
Nominal CLF-CBF QP & 61.7 & 8.5 & 29.8 & $10.1 \pm 2.1$ \\ \hline
\end{tabular}
\end{table*}

\begin{figure*}[ht]
\centering
\subfloat[Defensive Maneuver]{\includegraphics[width=0.195\linewidth]{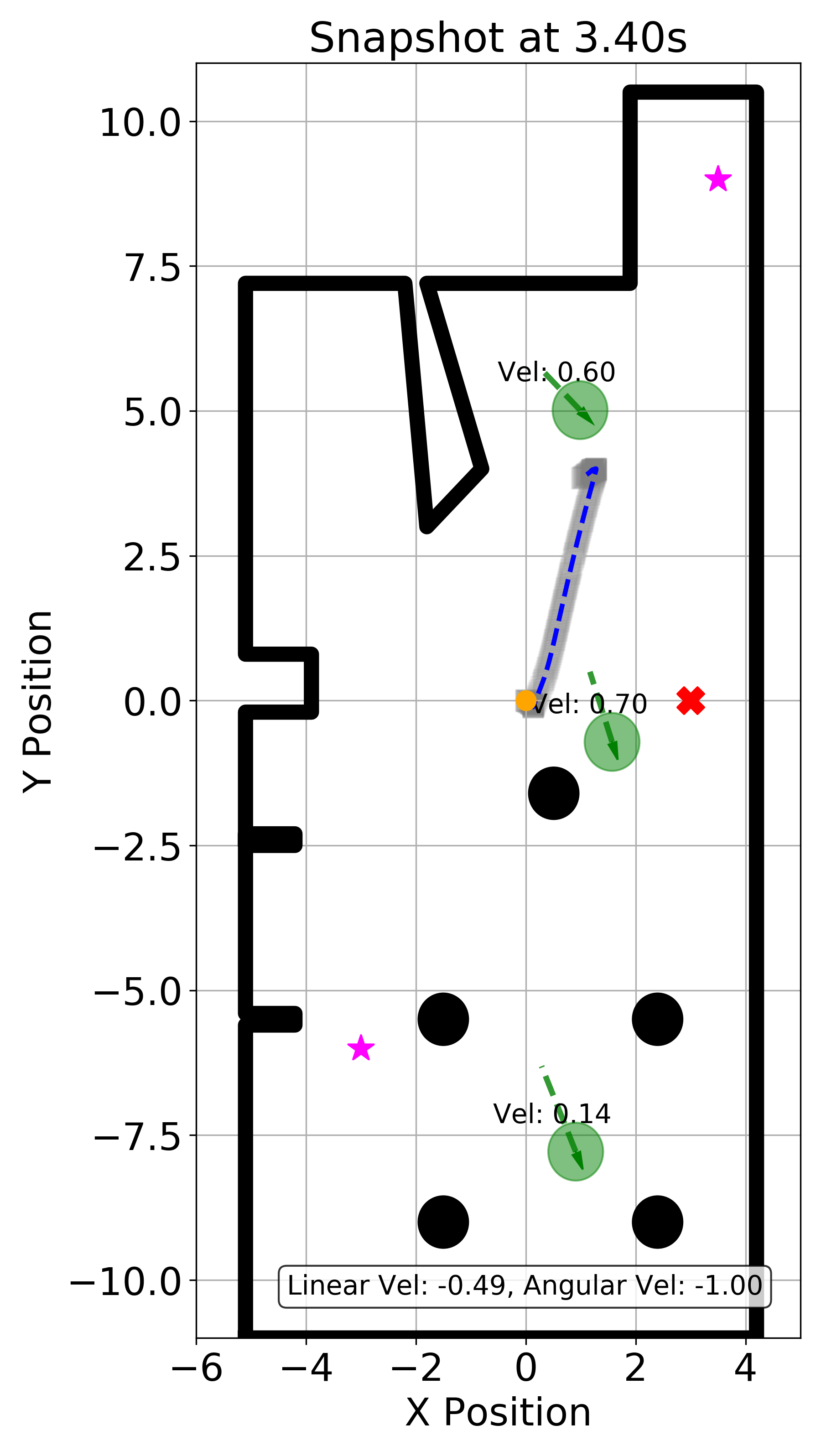}\label{fig:snapshot_170}}
\hfill
\subfloat[Resuming Tracking]{\includegraphics[width=0.195\linewidth]{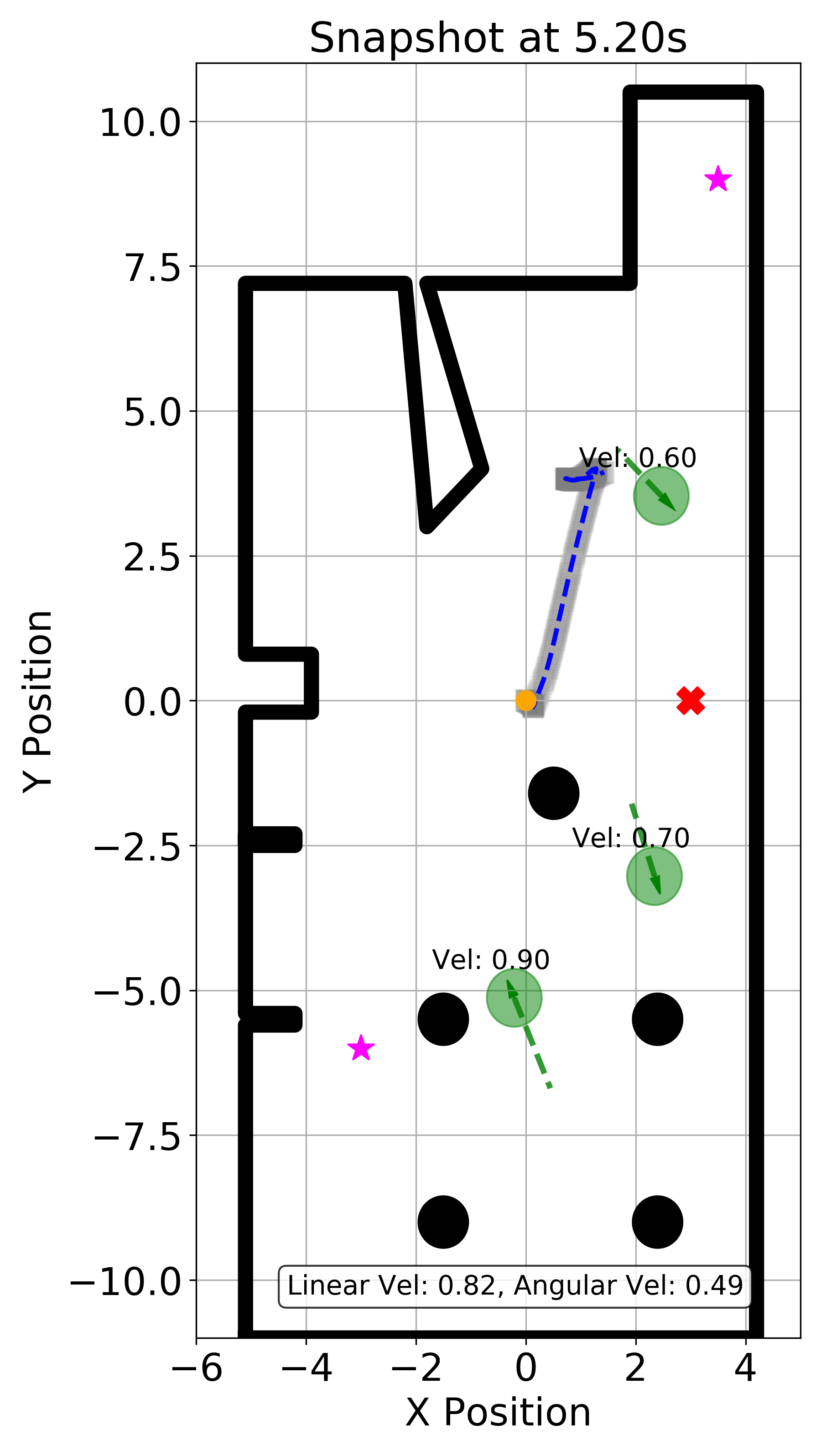}\label{fig:snapshot_260}}
\hfill
\subfloat[Wait Pedestrian]{\includegraphics[width=0.195\linewidth]{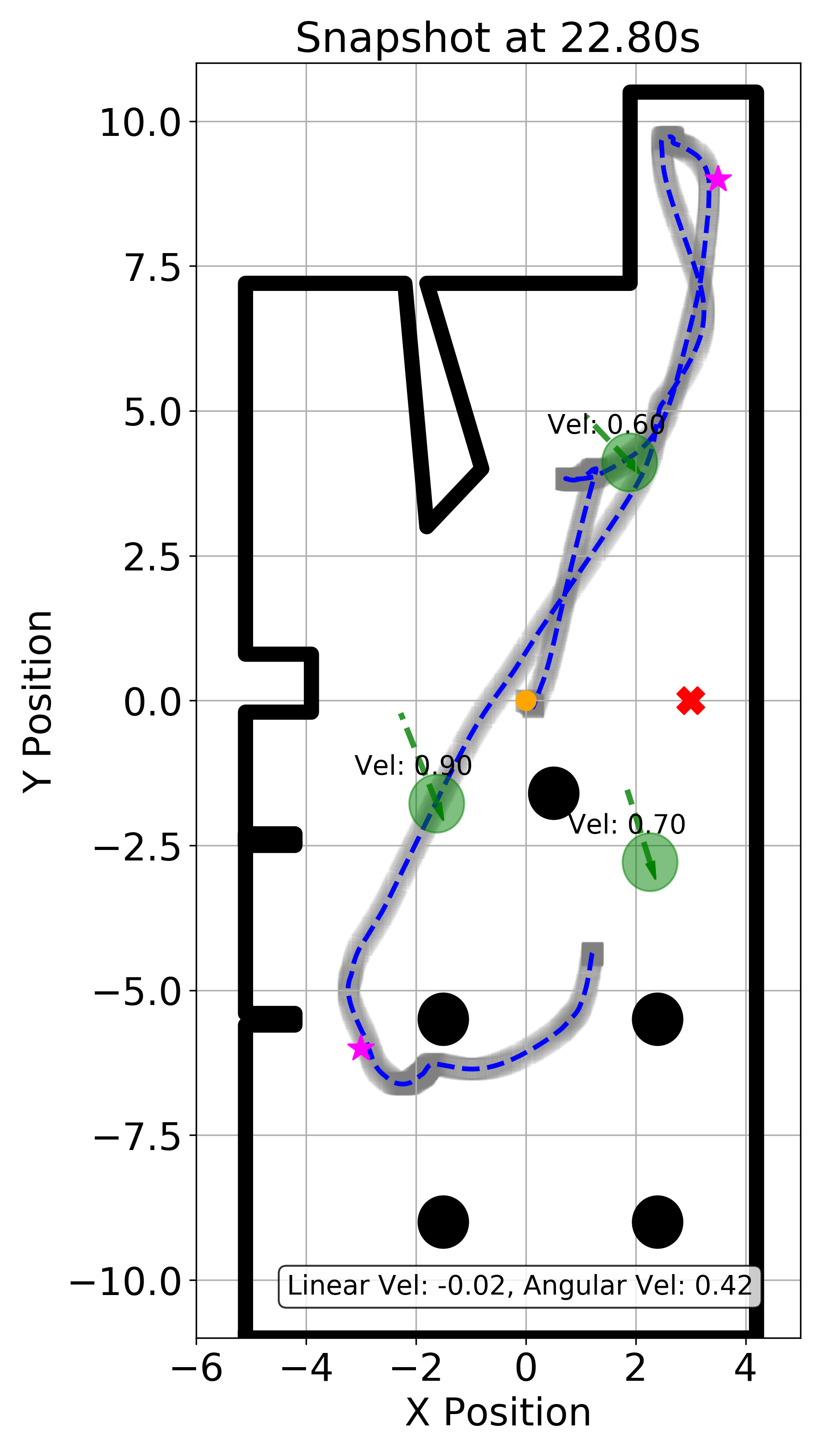}\label{fig:snapshot_1140}}
\hfill
\subfloat[Resuming Tracking]{\includegraphics[width=0.195\linewidth]{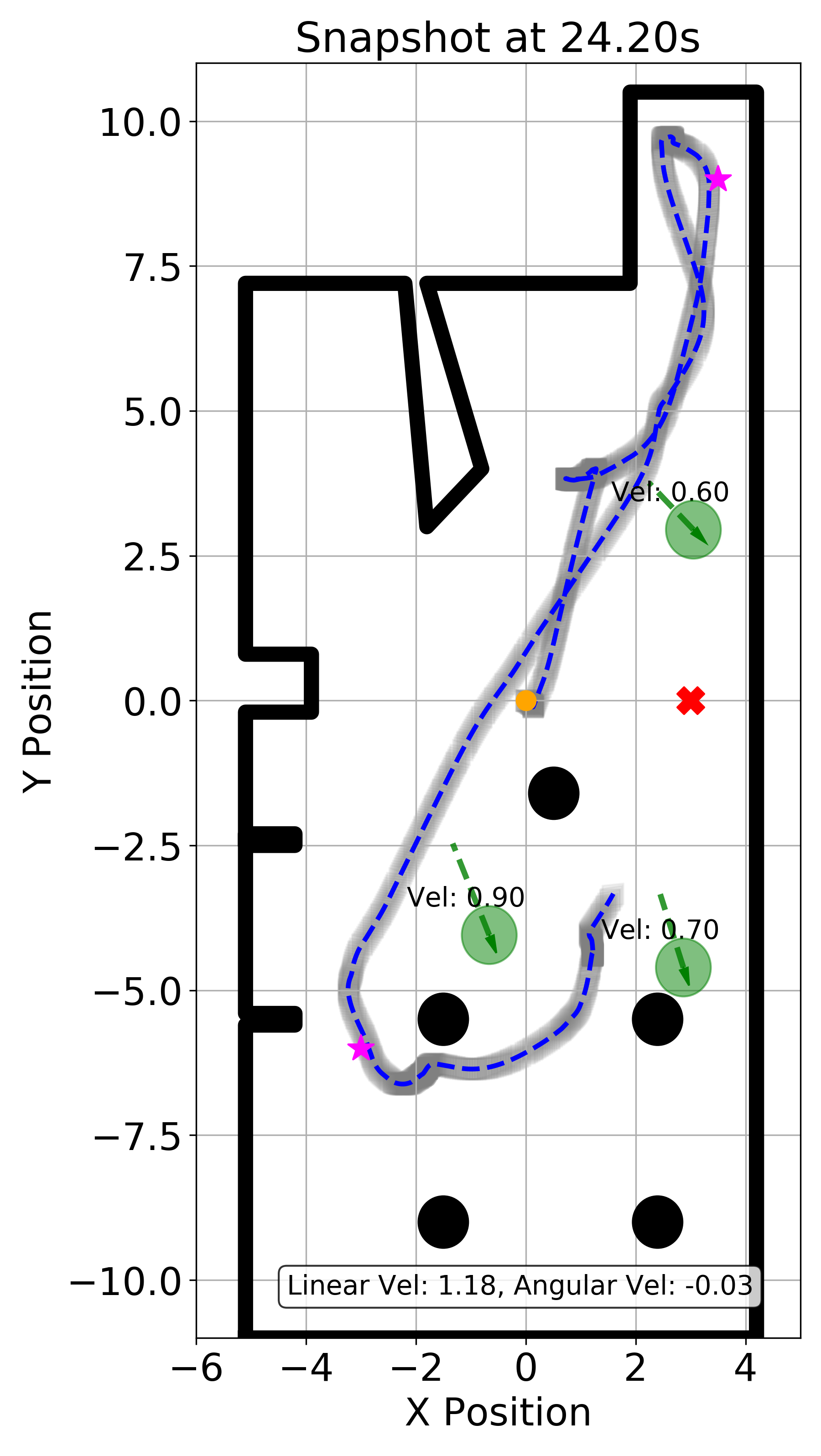}\label{fig:snapshot_1210}}
\hfill
\subfloat[Task Completion]{\includegraphics[width=0.195\linewidth]{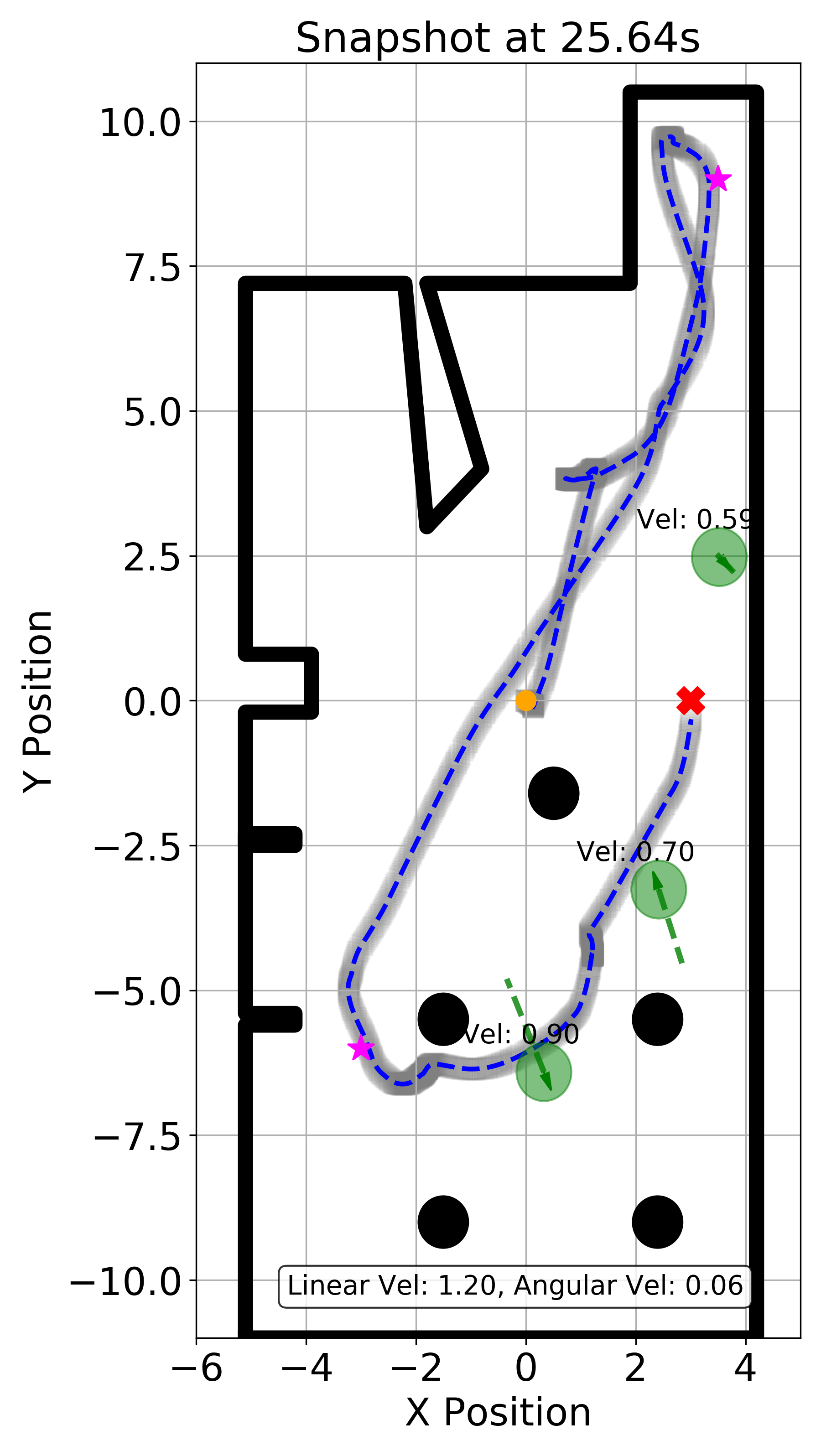}\label{fig:snapshot_1282}}\\
\includegraphics[width=0.75\linewidth,trim={0 10mm 0 5mm},clip]{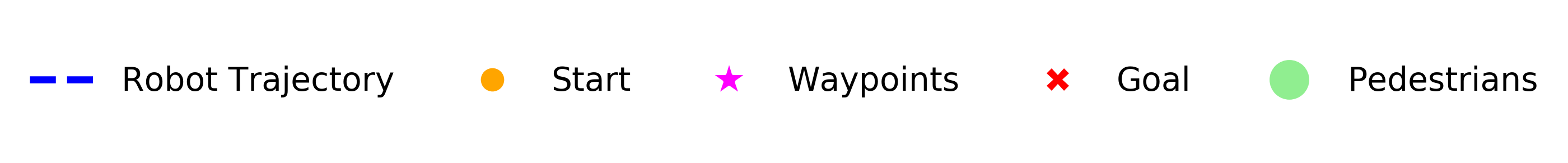}
\caption{Snapshots showing safe robot navigation in a simulated dynamic environment with three pedestrians, as depicted in Fig.~\ref{fig:3b}. The ground-truth static environment (e.g., walls, table base) is plotted in black. Each pedestrian is represented by a light green circle, with trajectory over the past second and current velocity also displayed. (a) At $t=3.4$s, the robot adjusts its trajectory due to an approaching pedestrian, adopting a defensive maneuver by rotating left ($-1$ rad/s) and moving backwards ($-0.49$m/s). (b) By $t=5.2$s, as the pedestrian clears, the robot accelerates forward ($0.89$m/s) to track its planned path towards the first waypoint. (c) At $t=22.8$s, facing another pedestrian crossing its planned path, the robot stops ($-0.02$m/s) to allow the pedestrian to pass. (d) At $t=24.2$s, the pedestrian has moved away, enabling the robot to resume its course towards the goal. (e) The complete trajectory at $t=25.6$s shows the robot successfully navigated to two waypoints and the final goal, ensuring safety in a dynamically changing environment.}
\label{fig:mov_robot_trajectories}
\vspace{-2ex}
\end{figure*}


The dynamic environment simulations are conducted in Gazebo (cf. Fig.~\ref{fig:3b}), designed to mimic real-world scenarios with pedestrians modeled using the social force model \citep{helbing1995social, moussaid2010walking}.

\textbf{Hypothesis:}  
Our CLF-DR-CBF QP approach will outperform baseline methods in handling time-varying constraints under noisy conditions. Specifically, we expect our method to achieve higher success rates, lower collision rates, and efficient task completion times, due to its ability to incorporate sensor noise and localization uncertainty directly into the control formulation.

\textbf{Setup:}  
The robot starts at $(0,0)$ with an initial orientation of $0^\circ$, and the goal locations are randomly placed at least $6$ meters away. Pedestrians are also randomly positioned in the environment, with velocities bounded by $B = 1$ m/s. Both the static and dynamic elements of the environment are unknown, and the robot relies on noisy LiDAR measurements (Gaussian noise with $\sigma = 0.05$) for collision avoidance. In all simulations below, the $A^*$ planning algorithm operates independently of the pedestrian motion, and the real-time pedestrian avoidance relies on our CLF-DR-CBF QP formulation (or the two baseline approaches). We conducted 1000 trials for each method, measuring metrics such as success rate, stuck rate, collision rate, and average task completion time. Success is defined as reaching the goal while maintaining safety, while a trial is considered stuck if the robot fails to reach the goal within 20 seconds.


\textbf{Results and discussion:}  
Table~\ref{tab:dynamic_metrics} summarizes the quantitative results. The proposed CLF-DR-CBF QP approach achieves the highest success rate (93.2\%) and the lowest collision rate (1.7\%) among the three methods.
By directly handling sensor noise through its distributionally robust formulation, the CLF-DR-CBF QP method ensures safety while maintaining efficient task completion times. The GP-CBF method exhibits a lower success rate and higher stuck rate due to its reliance on the GP-SDF map, which becomes computationally expensive and less reliable in dynamic environments. The Nominal CLF-CBF QP approach suffers from the highest collision rate (29.8\%), highlighting its limitations in handling dynamic obstacles with sensor noise. 


%
%

We next present some qualitative results in Fig.~\ref{fig:mov_robot_trajectories} in the same dynamic environment. The robot is tasked to sequentially visit two waypoints before reaching a designated goal at $(3,0)$. In Fig.~\ref{fig:snapshot_170}, at $t=3.4$s, the robot encounters a pedestrian on a collision course with its planned path to the first waypoint at the top right. With our CLF-DR-CBF QP controller, the robot employs a defensive maneuver. This adjustment shows the methodology's capability to anticipate potential hazards and react accordingly.

As the pedestrian clears the immediate area, the robot resumes its path tracking towards the first waypoint by $5.2$s (Fig.~\ref{fig:snapshot_260}). This behavior highlights the efficiency of our approach in balancing mission objectives with the need for safety.

The challenge intensifies at $t=22.8$s when another pedestrian intersects the robot's planned route (Fig.~\ref{fig:snapshot_1140}). In response, the robot stops to allow the pedestrian to pass safely. Once the pedestrian has passed, the robot continues its journey towards the goal, as observed at $t=24.2$s (Fig.~\ref{fig:snapshot_1210}). 

The successful completion of the task is shown in Fig.~\ref{fig:snapshot_1282}, where the robot reaches its final goal after safely navigating past all dynamic obstacles at $t=25.64$s. This simulation shows the CLF-DR-CBF QP controller's ability for robust path tracking and obstacle avoidance in a dynamic environment. 


\subsection{Real-World Experiments}
\label{sec: eva_real_world}

\begin{figure}[t]
  \centering
  \subcaptionbox{Robot Shape 1\label{fig:experiment_shape1}}{\includegraphics[width=0.49\linewidth]{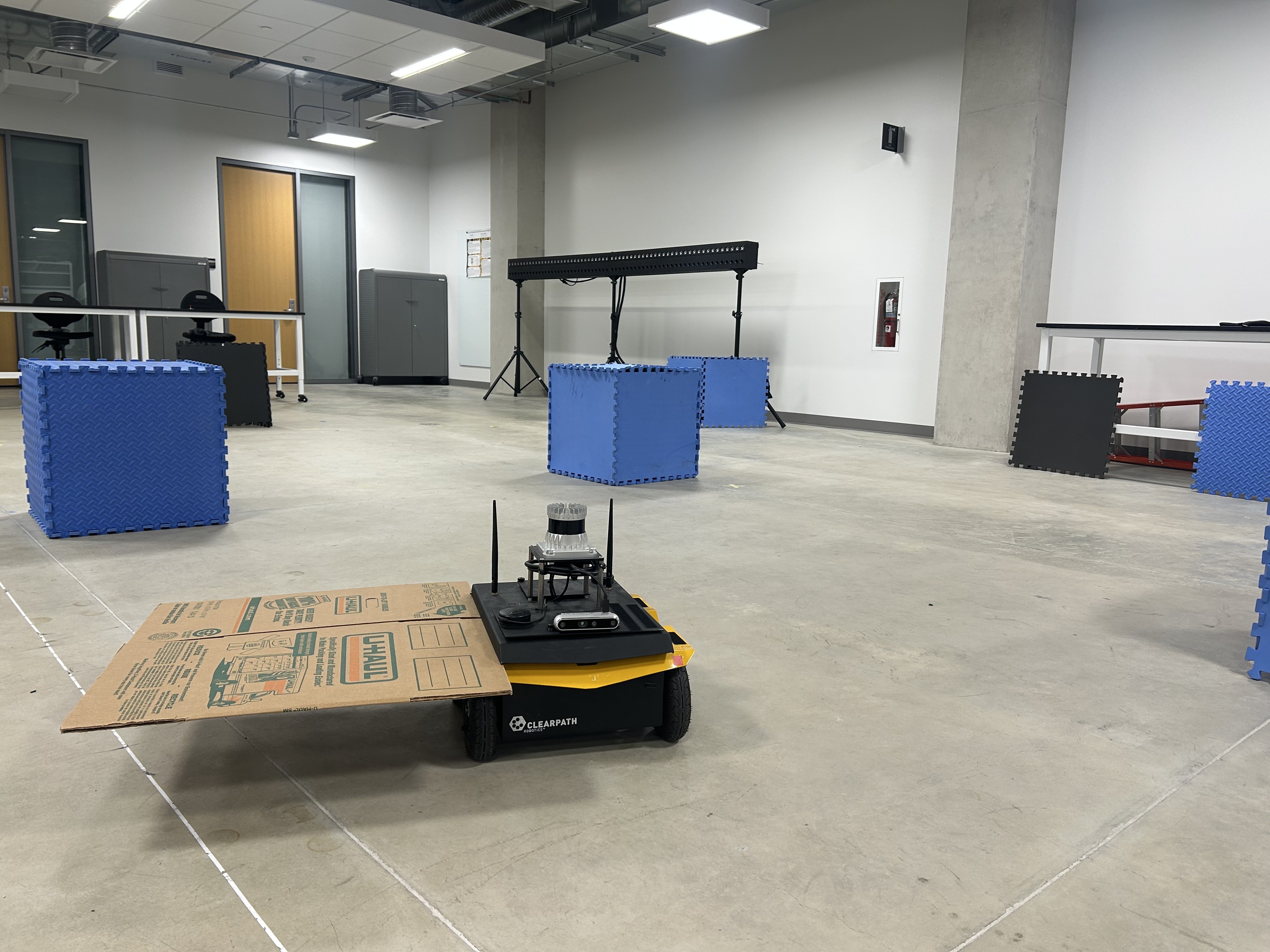}}%
  \hfill%
  \subcaptionbox{Robot Shape 2\label{fig:experiment_shape2}}{\includegraphics[width=0.49\linewidth]{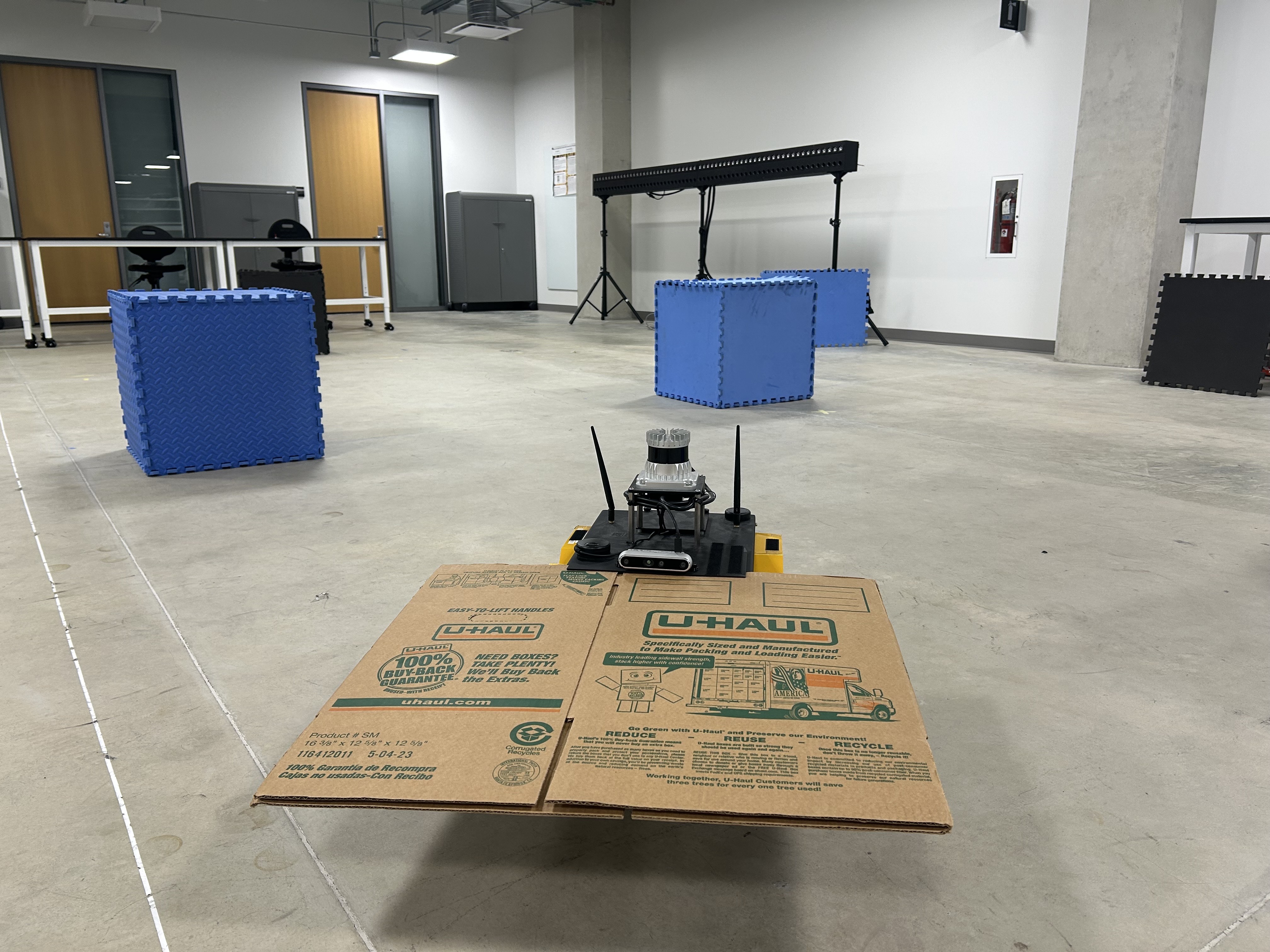}}\\
  \caption{Robot shapes used in real-world experiments.}
  \vspace{-2ex}
\end{figure}

\begin{figure}[t]
  \centering
  \includegraphics[width=0.8\linewidth]{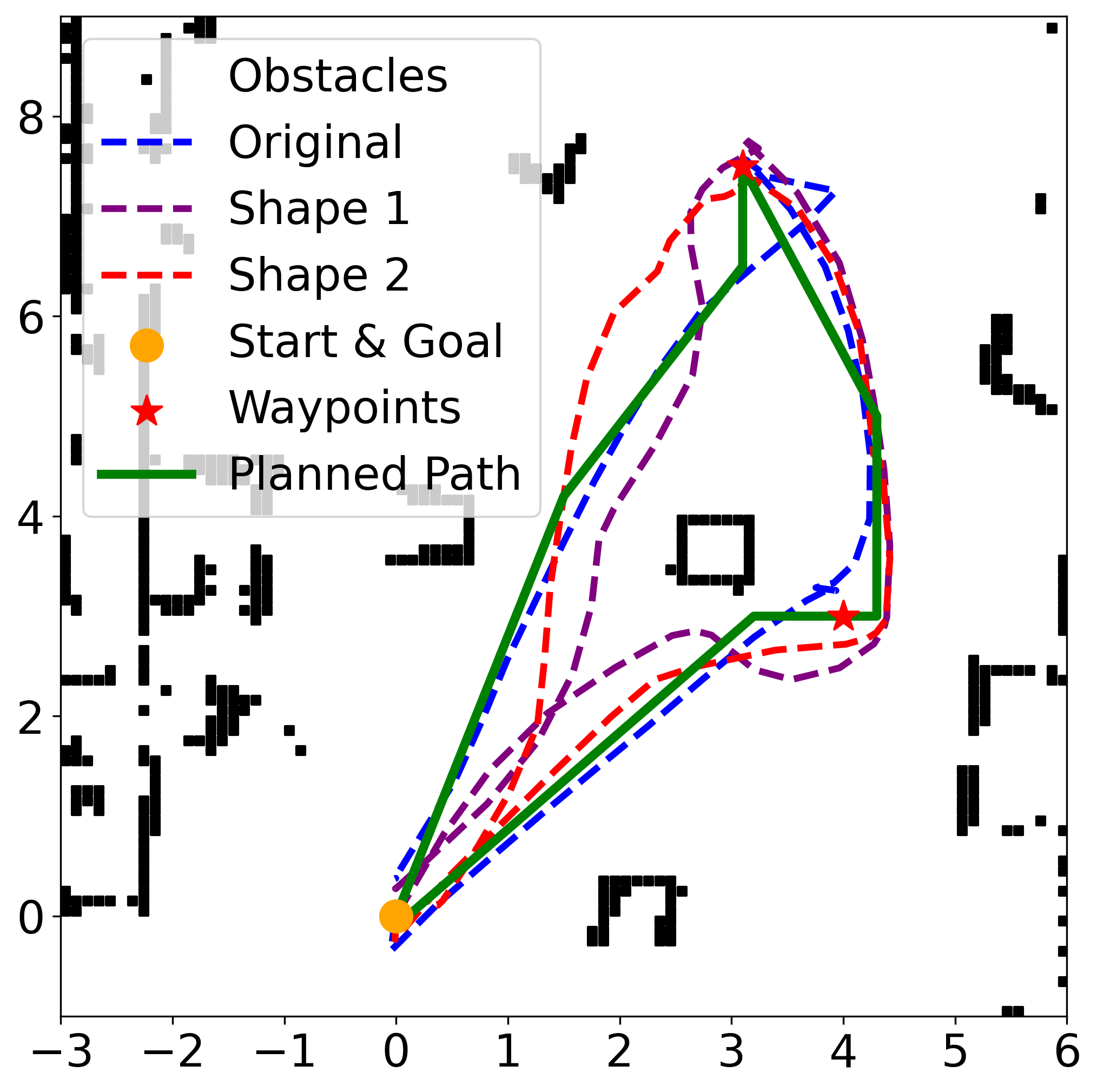}\\
  \caption{Comparison of trajectories for the three robot shapes.}
  \label{fig: exp_traj_compare}
  \vspace{-1ex}
\end{figure}


\begin{figure*}[ht]
\centering
\subfloat[Thin chair legs]{\includegraphics[width=0.33\textwidth]{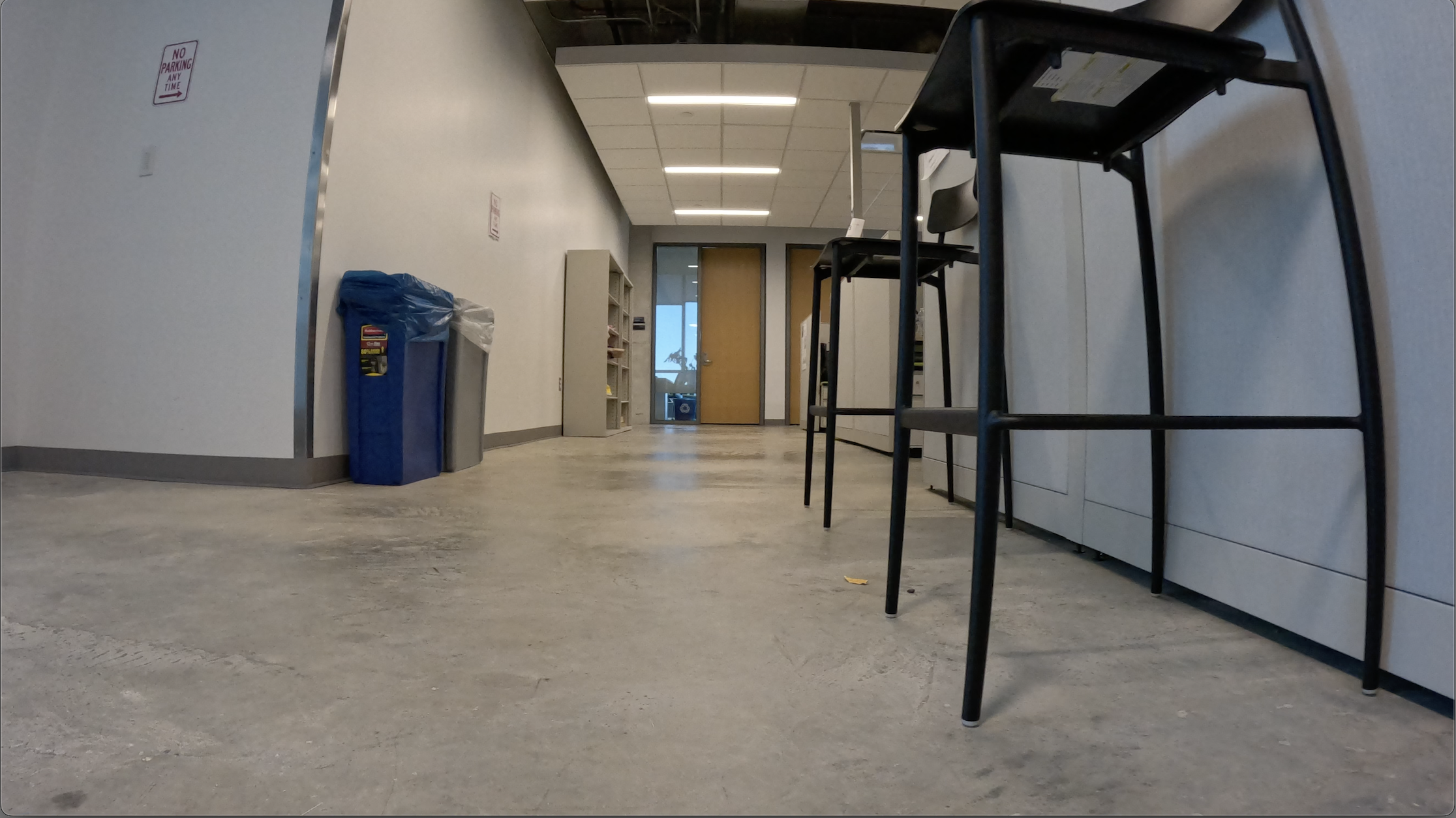}\label{fig:lab_example_1}}
\hfill
\subfloat[Narrow passage with pedestrians]{\includegraphics[width=0.33\textwidth]{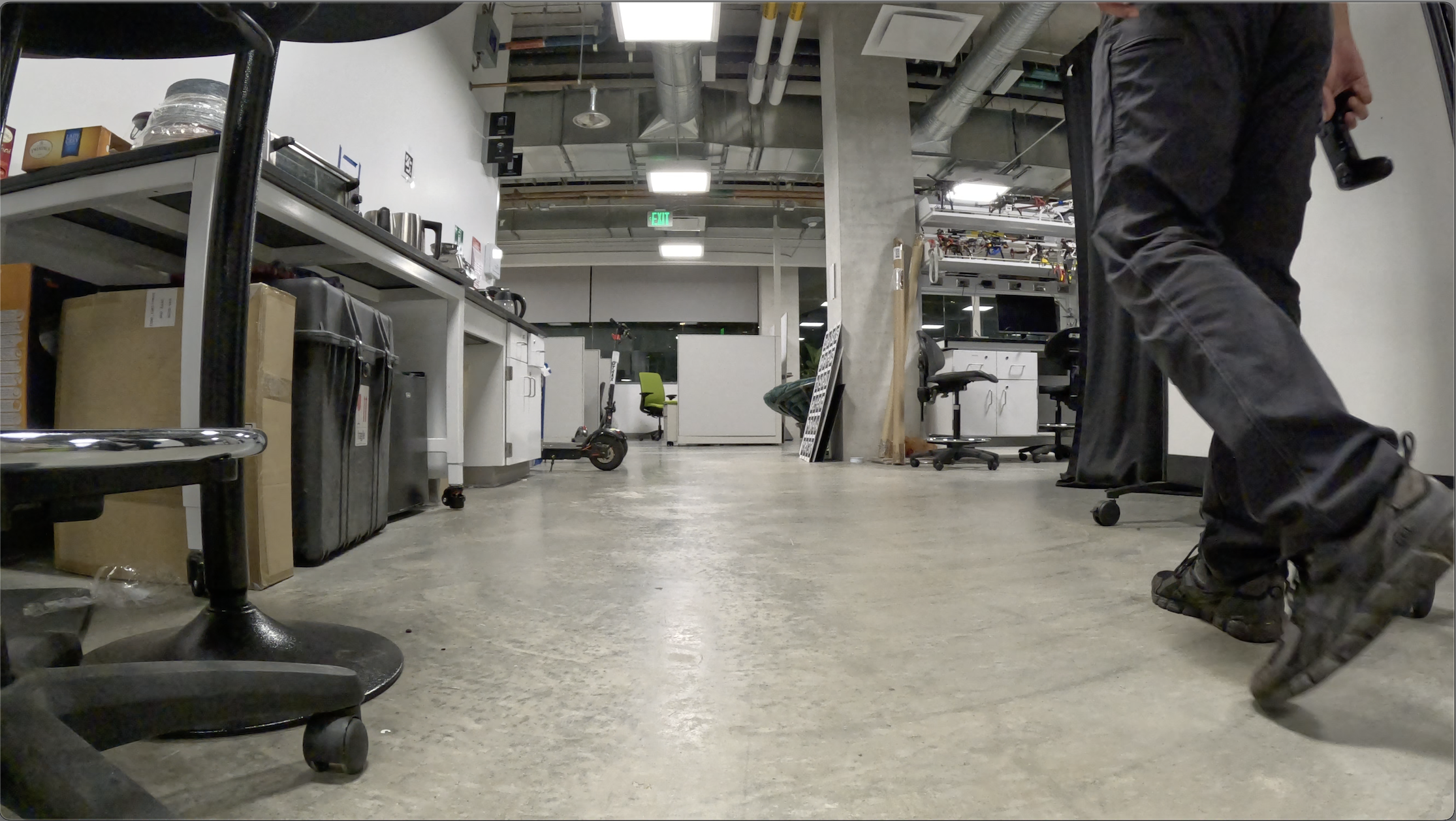}\label{fig:lab_example_2}}
\hfill
\subfloat[Pedestrian approaching]{\includegraphics[width=0.33\textwidth]{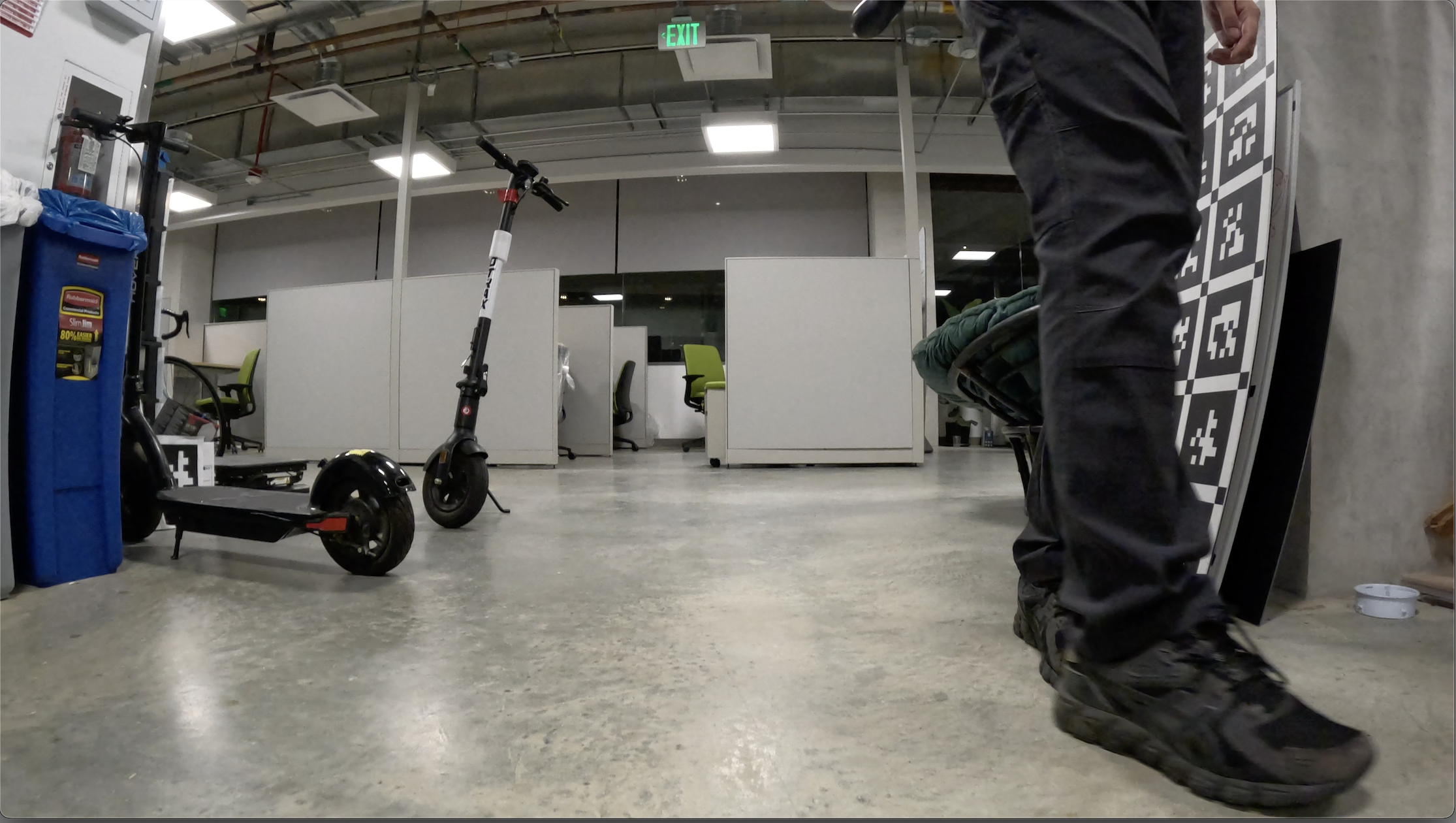}\label{fig:lab_example_3}}\\[1ex]
\includegraphics[width=\linewidth,trim={0 5mm 0 0},clip]{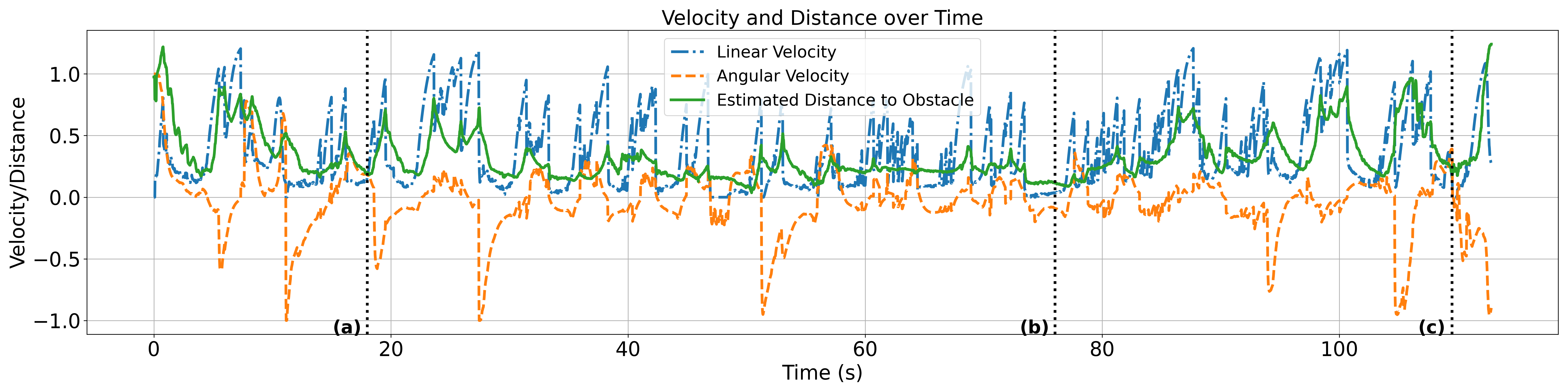}
\caption{Evaluation of our CLF-DR-CBF QP approach in a real lab environment. The top row illustrates challenging scenarios, including (a) navigating around thin chair legs, (b) passing through a narrow passage with pedestrians, and (c) handling an approaching pedestrian. The bottom plot shows the robot's velocity profile and distance to obstacles over time, with 3 vertical dotted lines marking the specific time instances corresponding to the challenging scenarios in the top row.}
\label{fig:experiment_velo_dis}
\end{figure*}

\begin{figure}[t]
\centering
\includegraphics[width=\linewidth]{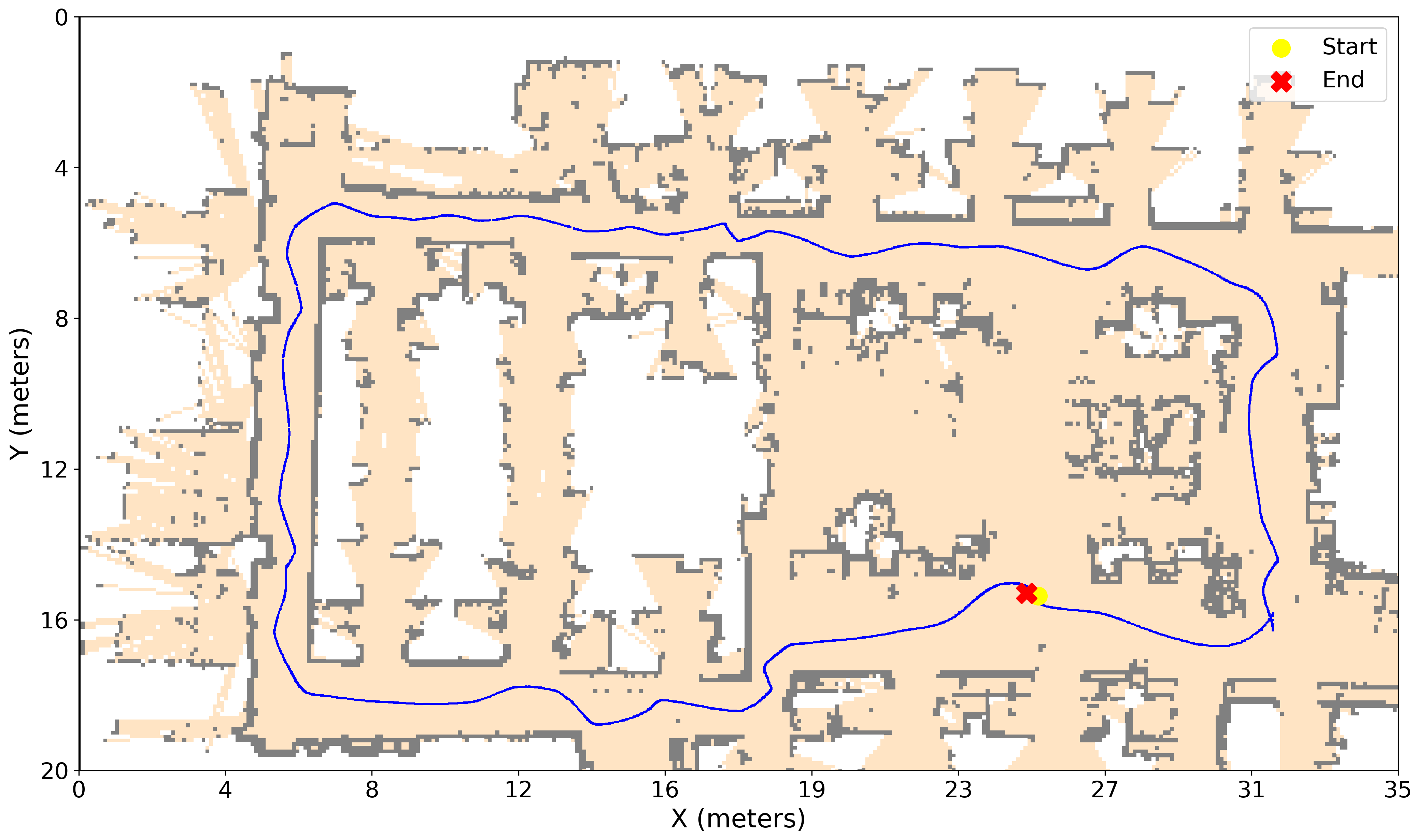}
\caption{Robot trajectory (blue) and estimated occupancy map (yellow and gray) of the lab environment.}
\label{fig:experiment_map_traj}
\end{figure}

We carried out real-world experiments using a differential-drive ClearPath Jackal robot (Fig.~\ref{fig:jackal_robot}). The robot was equipped with an Intel i7-9700TE CPU with 32GB RAM, an Ouster OS1-32 LiDAR, and a UM7 9-axis IMU, and a velocity controller accepting linear and angular velocity.

\textbf{Setup:}
The experiments took place in a lab environment, designed to test various challenging scenarios. The robot relied solely on noisy LiDAR measurements for navigation. We first tested our approach in an area of the lab with static obstacles, using three different robot shapes: the original shape (Fig.~\ref{fig:1a}), Shape 1 (Fig.~\ref{fig:experiment_shape1}), and Shape 2 (Fig.~\ref{fig:experiment_shape2}).  For quantitative evaluation, we ran 50 trials per shape in environments populated with randomly placed obstacles (e.g., cubes and pyramids) and three randomly walking pedestrians. 

\textbf{Results and discussion:}
Figure~\ref{fig: exp_traj_compare} shows the trajectories of the robot with the three shapes navigating the same environment by following a pre-planned path. Notably, the planned path was generated assuming the nominal robot shape and did not account for the differences introduced by Shape 1 and Shape 2. Consequently, the original shape demonstrates a minimal deviation from the planned path, whereas Shape 1 exhibits the largest deviations due to its wider, asymmetrical design.

Table~\ref{tab:dynamic_metrics_shapes} summarizes the results, showing the success, stuck, and collision rates for the three robot shapes under the proposed CLF-DR-CBF QP formulation. Our approach achieved high success rates across all robot shapes, demonstrating its robustness to variations in geometry and dynamic obstacles in the environment. In our evaluation, a robot was considered “stuck” if it did not reach its goal within 60 seconds.

When the robot shape becomes larger (Shapes 1 and 2), it is more likely to get stuck due to limited maneuverability in tight spaces. This is particularly evident for Shape 1, where the asymmetry makes the robot significantly wider on one side, further complicating its ability to bypass obstacles. Additionally, the larger size of Shapes 1 and 2 reduces their ability to navigate around obstacles within the required time, leading to reduced success rates. Despite these challenges, all three shapes achieved a 0\% collision rate, demonstrating the effectiveness of our CLF-DR-CBF QP formulation in maintaining safe navigation.

\begin{table}[t]
\caption{Performance metrics for real-world experiments with different robot shapes. Metrics include success rate, stuck rate, and collision rate over 50 trials per shape.}
\label{tab:dynamic_metrics_shapes}
\centering
\begin{tabular}{|l|c|c|c|}
\hline
\textbf{Shape} & \textbf{Success(\%)} & \textbf{Stuck (\%)} & \textbf{Collision(\%)} \\ \hline
Original      & 100             & 0                      & 0                         \\ \hline
Shape 1              & 84             & 16                      & 0               \\ \hline
Shape 2              & 90             & 10                      & 0               \\ \hline
\end{tabular}
\vspace{-2ex}
\end{table}

Next, we demonstrated the performance of our CLF-DR-CBF QP formulation using the original robot shape in a full-lab navigation task. The robot successfully handled various real-world challenges, such as thin chair legs, narrow passages with pedestrians, and approaching pedestrians (Fig.~\ref{fig:experiment_velo_dis}). In contrast to the CLF-GP-CBF SOCP formulation, which relies on GP regression to construct CBFs and cannot handle dynamic environments effectively, our CLF-DR-CBF QP formulation maintains safety while solely depending on noisy LiDAR measurements. The bottom plot in Fig.~\ref{fig:experiment_velo_dis} presents the distance to the obstacles and the robot's velocity profile over time, highlighting the robot's ability to maintain a safe distance while efficiently navigating towards its goal. For more details, please refer to the accompanying videos on the project webpage\footnote{\href{https://existentialrobotics.org/DRO_Safe_Navigation/}{https://existentialrobotics.org/DRO\_Safe\_Navigation/}}.

Fig.~\ref{fig:experiment_map_traj} depicts the estimated occupancy map and the executed trajectory using our CLF-DR-CBF QP controller. The robot successfully navigates through the cluttered environment, avoiding both static and dynamic obstacles, and reaches its desired goal position.

\section{Conclusion and Future Work}

We introduced a novel strategy for ensuring safety of mobile robots navigating autonomously in unknown dynamically changing environments. Our distributionally robust control barrier function formulation leverages sensor measurements and state estimates directly, eliminating the need for precise knowledge of CBFs, which may be slow and inaccurate to obtain in dynamic environments. By combining the DR-CBF with a control Lyapunov function for path tracking, we developed a CLF-DR-CBF quadratic program that enables safe autonomous navigation for robots with control-affine dynamics and arbitrary shape. Our approach underscores the efficiency and effectiveness of employing sensor-based DR-CBF constraints to handle uncertainty in measurements and state estimates. The simulation and experiment results suggest that further exploration into integrating sensor measurements and perception estimates directly into the robot planning and control methods may lead to significant progress in deploying reliable autonomous robot systems in the real world.

While our methodology demonstrates robustness and effectiveness, it has certain limitations. Sections \ref{sec: dro_cbf_navigation} present a general framework for control-affine systems, relying on the assumption that valid CBFs with relative degree 1 exist for these systems. However, identifying such functions can be nontrivial, particularly for complex or high-dimensional systems. To provide a practical demonstration, we employed a signed distance function as a CBF for a differential-drive robot in Section \ref{sec: unicycle} and conducted our evaluations on this specific model.

Future work will aim to extend this methodology to more complex robot systems, such as mobile manipulators and humanoids. Additionally, we plan to explore high-order control barrier functions to address challenges associated with designing valid CBFs, further bridging the gap between theoretical safety assurances and practical deployments in real-world scenarios.

\begin{funding}
We gratefully acknowledge support from ONR Award N00014-23-1-2353, NSF Award CMMI-2044900, and NSF Award CCF-2112665 (TILOS).
\end{funding}


%
%
\bibliographystyle{SageH}
\bibliography{ref.bib}

\begin{thebibliography}{111}
\providecommand{\natexlab}[1]{#1}
\providecommand{\url}[1]{\texttt{#1}}
\providecommand{\urlprefix}{URL }
\expandafter\ifx\csname urlstyle\endcsname\relax
  \providecommand{\doi}[1]{DOI:\discretionary{}{}{}#1}\else
  \providecommand{\doi}{DOI:\discretionary{}{}{}\begingroup \urlstyle{rm}\Url}\fi

\bibitem[{Abdi et~al.(2023)Abdi, Raja and Ghabcheloo}]{abdi2023safe}
Abdi H, Raja G and Ghabcheloo R (2023) Safe control using vision-based control barrier function ({V-CBF}).
\newblock In: \emph{IEEE International Conference on Robotics and Automation (ICRA)}. pp. 782--788.

\bibitem[{Abuaish et~al.(2023)Abuaish, Srinivasan and Vela}]{Abuaish_2023_radial_NN_ACC}
Abuaish A, Srinivasan M and Vela PA (2023) Geometry of radial basis neural networks for safety biased approximation of unsafe regions.
\newblock In: \emph{2023 American Control Conference (ACC)}. pp. 1459--1466.

\bibitem[{Agrawal and Panagou(2021)}]{Agrawal_2021_cdc}
Agrawal DR and Panagou D (2021) Safe control synthesis via input constrained control barrier functions.
\newblock In: \emph{2021 60th IEEE Conference on Decision and Control (CDC)}. pp. 6113--6118.

\bibitem[{Ames et~al.(2019)Ames, Coogan, Egerstedt, Notomista, Sreenath and Tabuada}]{cbf}
Ames A, Coogan S, Egerstedt M, Notomista G, Sreenath K and Tabuada P (2019) Control barrier functions: Theory and applications.
\newblock In: \emph{European Control Conference (ECC)}. pp. 3420--3431.

\bibitem[{Ames et~al.(2017)Ames, Xu, Grizzle and Tabuada}]{ames2016control}
Ames AD, Xu X, Grizzle JW and Tabuada P (2017) Control barrier function based quadratic programs for safety critical systems.
\newblock \emph{IEEE Transactions on Automatic Control} 62(8): 3861--3876.

\bibitem[{Andersson et~al.(2019)Andersson, Gillis, Horn, Rawlings and Diehl}]{Andersson2019}
Andersson JAE, Gillis J, Horn G, Rawlings JB and Diehl M (2019) {CasADi} -- {A} software framework for nonlinear optimization and optimal control.
\newblock \emph{Mathematical Programming Computation} 11(1): 1--36.
\newblock \doi{10.1007/s12532-018-0139-4}.

\bibitem[{Aolaritei et~al.(2023)Aolaritei, Fochesato, Lygeros and D{\"o}rfler}]{aolaritei2023wasserstein}
Aolaritei L, Fochesato M, Lygeros J and D{\"o}rfler F (2023) Wasserstein tube {MPC} with exact uncertainty propagation.
\newblock In: \emph{IEEE Conference on Decision and Control (CDC)}. pp. 2036--2041.

\bibitem[{Arslan and Koditschek(2019)}]{arslan2019sensor}
Arslan O and Koditschek DE (2019) Sensor-based reactive navigation in unknown convex sphere worlds.
\newblock \emph{The International Journal of Robotics Research} 38(2-3): 196--223.

\bibitem[{Artstein(1983)}]{Artstein1983StabilizationWR}
Artstein Z (1983) Stabilization with relaxed controls.
\newblock \emph{Nonlinear Analysis: Theory, Methods \& Applications} 7: 1163--1173.

\bibitem[{Bajcsy et~al.(2019)Bajcsy, Bansal, Bronstein, Tolani and Tomlin}]{bajcsy_2019_cdc_safety_static}
Bajcsy A, Bansal S, Bronstein E, Tolani V and Tomlin CJ (2019) An efficient reachability-based framework for provably safe autonomous navigation in unknown environments.
\newblock In: \emph{2019 IEEE 58th Conference on Decision and Control (CDC)}. pp. 1758--1765.

\bibitem[{Boffi et~al.(2021)Boffi, Tu, Matni, Slotine and Sindhwani}]{boffi2021learning}
Boffi N, Tu S, Matni N, Slotine JJ and Sindhwani V (2021) Learning stability certificates from data.
\newblock In: \emph{Conference on Robot Learning}. PMLR, pp. 1341--1350.

\bibitem[{Borenstein et~al.(1991)Borenstein, Koren et~al.}]{borenstein1991vector}
Borenstein J, Koren Y et~al. (1991) The vector field histogram-fast obstacle avoidance for mobile robots.
\newblock \emph{IEEE Transactions on Robotics and Automation} 7(3): 278--288.

\bibitem[{Boskos et~al.(2023)Boskos, Cort{\'e}s and Mart{\'\i}nez}]{DB-JC-SM:23-cdc}
Boskos D, Cort{\'e}s J and Mart{\'\i}nez S (2023) Data-driven distributionally robust coverage control by mobile robots.
\newblock In: \emph{IEEE Conference on Decision and Control (CDC)}. pp. 2030--2035.

\bibitem[{Boskos et~al.(2024)Boskos, Cort{\'e}s and Mart{\'i}nez}]{DB-JC-SM:24-tac}
Boskos D, Cort{\'e}s J and Mart{\'i}nez S (2024) High-confidence data-driven ambiguity sets for time-varying linear systems.
\newblock \emph{IEEE Transactions on Automatic Control} 69(2): 797--812.

\bibitem[{Breeden and Panagou(2023)}]{breeden2023robust}
Breeden J and Panagou D (2023) Robust control barrier functions under high relative degree and input constraints for satellite trajectories.
\newblock \emph{Automatica} 155: 111109.

\bibitem[{Brito et~al.(2019)Brito, Floor, Ferranti and Alonso-Mora}]{brito2019model}
Brito B, Floor B, Ferranti L and Alonso-Mora J (2019) Model predictive contouring control for collision avoidance in unstructured dynamic environments.
\newblock \emph{IEEE Robotics and Automation Letters} 4(4): 4459--4466.

\bibitem[{Chang et~al.(2019)Chang, Roohi and Gao}]{Chang2019NeuralLC}
Chang YC, Roohi N and Gao S (2019) Neural {L}yapunov control.
\newblock In: \emph{Advances in Neural Information Processing Systems}, volume~32.

\bibitem[{Chatila and Laumond(1985)}]{chatila1985position}
Chatila R and Laumond J (1985) Position referencing and consistent world modeling for mobile robots.
\newblock In: \emph{IEEE International Conference on Robotics and Automation (ICRA)}, volume~2. pp. 138--145.

\bibitem[{Chen et~al.(2022)Chen, Pei, Lu and Li}]{chen2022deep_rl}
Chen P, Pei J, Lu W and Li M (2022) A deep reinforcement learning based method for real-time path planning and dynamic obstacle avoidance.
\newblock \emph{Neurocomputing} 497: 64--75.

\bibitem[{Choi et~al.(2023)Choi, Casta{\~n}eda, Jung, Zhang, Tomlin and Sreenath}]{choi2023constraint}
Choi JJ, Casta{\~n}eda F, Jung W, Zhang B, Tomlin CJ and Sreenath K (2023) Constraint-guided online data selection for scalable data-driven safety filters in uncertain robotic systems.
\newblock \emph{arXiv preprint arXiv:2311.13824} .

\bibitem[{Chriat and Sun(2024)}]{chriat2023wasserstein}
Chriat A and Sun C (2024) Wasserstein distributionally robust control barrier function using conditional value-at-risk with differentiable convex programming.
\newblock In: \emph{AIAA SCITECH 2024 Forum}. p. 0725.

\bibitem[{Clark(2019)}]{clark_2019_acc_robust_cbf}
Clark A (2019) Control barrier functions for complete and incomplete information stochastic systems.
\newblock In: \emph{2019 American Control Conference (ACC)}. pp. 2928--2935.
\newblock \doi{10.23919/ACC.2019.8814901}.

\bibitem[{Clarke(1981)}]{clarke1981generalized}
Clarke FH (1981) Generalized gradients of lipschitz functionals.
\newblock \emph{Advances in Mathematics} 40(1): 52--67.

\bibitem[{Corke et~al.(2000)Corke, Trevelyan, Brock and Khatib}]{corke2000elastic}
Corke P, Trevelyan J, Brock O and Khatib O (2000) Elastic strips: A framework for integrated planning and execution.
\newblock In: \emph{Experimental Robotics VI}. Springer, pp. 329--338.

\bibitem[{Coulson et~al.(2021)Coulson, Lygeros and D{\"o}rfler}]{coulson2021distributionally}
Coulson J, Lygeros J and D{\"o}rfler F (2021) Distributionally robust chance constrained data-enabled predictive control.
\newblock \emph{IEEE Transactions on Automatic Control} 67(7): 3289--3304.

\bibitem[{Dai et~al.(2024)Dai, Jiang, Zhang and Clark}]{HD-CJ-HZ-AC:24}
Dai H, Jiang C, Zhang H and Clark A (2024) Verification and synthesis of compatible control lyapunov and control barrier functions.
\newblock In: \emph{2024 IEEE 63rd Conference on Decision and Control (CDC)}. pp. 8178--8185.
\newblock \doi{10.1109/CDC56724.2024.10885943}.

\bibitem[{Das and Burdick(2024)}]{das2024robust}
Das E and Burdick JW (2024) Robust control barrier functions using uncertainty estimation with application to mobile robots.
\newblock \emph{arXiv preprint arXiv:2401.01881} .

\bibitem[{Dawson et~al.(2022)Dawson, Lowenkamp, Goff and Fan}]{dawson2022learning}
Dawson C, Lowenkamp B, Goff D and Fan C (2022) Learning safe, generalizable perception-based hybrid control with certificates.
\newblock \emph{IEEE Robotics and Automation Letters} 7(2): 1904--1911.

\bibitem[{Daş and Murray(2022)}]{Das_2022_cdc_robust}
Daş E and Murray RM (2022) Robust safe control synthesis with disturbance observer-based control barrier functions.
\newblock In: \emph{IEEE Conference on Decision and Control (CDC)}. pp. 5566--5573.

\bibitem[{De~Lima and Pereira(2013)}]{de2013navigation}
De~Lima DA and Pereira GAS (2013) Navigation of an autonomous car using vector fields and the dynamic window approach.
\newblock \emph{Journal of Control, Automation and Electrical Systems} 24: 106--116.

\bibitem[{Desai and Ghaffari(2022)}]{Manavendra_2022_navigate}
Desai M and Ghaffari A (2022) {CLF-CBF} based quadratic programs for safe motion control of nonholonomic mobile robots in presence of moving obstacles.
\newblock In: \emph{IEEE/ASME International Conference on Advanced Intelligent Mechatronics (AIM)}. pp. 16--21.

\bibitem[{Dhiman et~al.(2023)Dhiman, Khojasteh, Franceschetti and Atanasov}]{dhiman_2023_tac_probabilistic}
Dhiman V, Khojasteh MJ, Franceschetti M and Atanasov N (2023) Control barriers in bayesian learning of system dynamics.
\newblock \emph{IEEE Transactions on Automatic Control} 68(1): 214--229.

\bibitem[{Dijkstra(1959)}]{DIJKSTRA1959}
Dijkstra EW (1959) A note on two problems in connexion with graphs.
\newblock \emph{Numerische Mathematik} 1: 269--271.

\bibitem[{Dixit et~al.(2023)Dixit, Lindemann, Wei, Cleaveland, Pappas and Burdick}]{dixit2023adaptive}
Dixit A, Lindemann L, Wei SX, Cleaveland M, Pappas GJ and Burdick JW (2023) Adaptive conformal prediction for motion planning among dynamic agents.
\newblock In: \emph{Learning for Dynamics and Control Conference}. PMLR, pp. 300--314.

\bibitem[{Djuric et~al.(2003)Djuric, Kotecha, Zhang, Huang, Ghirmai, Bugallo and Miguez}]{djuric2003particle}
Djuric PM, Kotecha JH, Zhang J, Huang Y, Ghirmai T, Bugallo MF and Miguez J (2003) Particle filtering.
\newblock \emph{IEEE signal processing magazine} 20(5): 19--38.

\bibitem[{Dudek et~al.(1978)Dudek, Jenkin, Milios and Wilkes}]{dudek1978robotic}
Dudek G, Jenkin M, Milios E and Wilkes D (1978) Robotic exploration as graph construction.
\newblock \emph{J. Comput., vol} 7(3).

\bibitem[{Emam et~al.(2022)Emam, Glotfelter, Wilson, Notomista and Egerstedt}]{yousef_2021_tro}
Emam Y, Glotfelter P, Wilson S, Notomista G and Egerstedt M (2022) Data-driven robust barrier functions for safe, long-term operation.
\newblock \emph{IEEE Transactions on Robotics} 38(3): 1671--1685.

\bibitem[{Esfahani and Kuhn(2018)}]{Esfahani2018DatadrivenDR}
Esfahani PM and Kuhn D (2018) Data-driven distributionally robust optimization using the {W}asserstein metric: performance guarantees and tractable reformulations.
\newblock \emph{Mathematical Programming} 171: 115--166.

\bibitem[{Everett et~al.(2021)Everett, Chen and How}]{everett_2021_rl}
Everett M, Chen YF and How JP (2021) Collision avoidance in pedestrian-rich environments with deep reinforcement learning.
\newblock \emph{IEEE Access} 9: 10357--10377.

\bibitem[{Fitzpatrick(1980)}]{fitzpatrick1980metric}
Fitzpatrick S (1980) Metric projections and the differentiability of distance functions.
\newblock \emph{Bulletin of the Australian Mathematical Society} 22(2): 291--312.

\bibitem[{Fox et~al.(1997)Fox, Burgard and Thrun}]{Fox1997TheDW}
Fox D, Burgard W and Thrun S (1997) The dynamic window approach to collision avoidance.
\newblock \emph{IEEE Robotics Autom. Mag.} 4: 23--33.

\bibitem[{Frazzoli et~al.(2002)Frazzoli, Dahleh and Feron}]{frazzoli2002real}
Frazzoli E, Dahleh MA and Feron E (2002) Real-time motion planning for agile autonomous vehicles.
\newblock \emph{Journal of guidance, control, and dynamics} 25(1): 116--129.

\bibitem[{Garone and Nicotra(2015)}]{garone2015explicit}
Garone E and Nicotra MM (2015) Explicit reference governor for constrained nonlinear systems.
\newblock \emph{IEEE Transactions on Automatic Control} 61(5): 1379--1384.

\bibitem[{Grandia et~al.(2021)Grandia, Taylor, Ames and Hutter}]{grandia_2021_legged}
Grandia R, Taylor AJ, Ames AD and Hutter M (2021) Multi-layered safety for legged robots via control barrier functions and model predictive control.
\newblock In: \emph{IEEE International Conference on Robotics and Automation (ICRA)}. pp. 8352--8358.

\bibitem[{Grisetti et~al.(2010)Grisetti, K{\"u}mmerle, Stachniss and Burgard}]{grisetti2010tutorial}
Grisetti G, K{\"u}mmerle R, Stachniss C and Burgard W (2010) A tutorial on graph-based slam.
\newblock \emph{IEEE Intelligent Transportation Systems Magazine} 2(4): 31--43.

\bibitem[{Hakobyan and Yang(2022)}]{hakobyan2022distributionally}
Hakobyan A and Yang I (2022) Distributionally robust risk map for learning-based motion planning and control: A semidefinite programming approach.
\newblock \emph{IEEE Transactions on Robotics} 39(1): 718--737.

\bibitem[{Hamdipoor et~al.(2023)Hamdipoor, Meskin and Cassandras}]{hamdipoor2023safe}
Hamdipoor V, Meskin N and Cassandras CG (2023) Safe control synthesis using environmentally robust control barrier functions.
\newblock \emph{European Journal of Control} 74: 100840.
\newblock 2023 European Control Conference Special Issue.

\bibitem[{Han et~al.(2019)Han, Gao, Zhou and Shen}]{luxin2019fiesta}
Han L, Gao F, Zhou B and Shen S (2019) Fiesta: Fast incremental {e}uclidean distance fields for online motion planning of aerial robots.
\newblock In: \emph{IEEE/RSJ International Conference on Intelligent Robots and Systems (IROS)}. pp. 4423--4430.

\bibitem[{Hart et~al.(1968)Hart, Nilsson and Raphael}]{A_star_planning}
Hart PE, Nilsson NJ and Raphael B (1968) A formal basis for the heuristic determination of minimum cost paths.
\newblock \emph{IEEE Transactions on Systems Science and Cybernetics} 4(2): 100--107.

\bibitem[{Helbing and Molnar(1995)}]{helbing1995social}
Helbing D and Molnar P (1995) Social force model for pedestrian dynamics.
\newblock \emph{Physical Review E} 51(5): 4282.

\bibitem[{Herbert et~al.(2017)Herbert, Chen, Han, Bansal, Fisac and Tomlin}]{herbert2017fastrack}
Herbert SL, Chen M, Han S, Bansal S, Fisac JF and Tomlin CJ (2017) Fa{ST}rack: A modular framework for fast and guaranteed safe motion planning.
\newblock In: \emph{IEEE Conference on Decision and Control (CDC)}. pp. 1517--1522.

\bibitem[{Hota et~al.(2019)Hota, Cherukuri and Lygeros}]{Hota2019DataDrivenCC}
Hota AR, Cherukuri AK and Lygeros J (2019) Data-driven chance constrained optimization under {W}asserstein ambiguity sets.
\newblock In: \emph{American Control Conference (ACC)}. pp. 1501--1506.

\bibitem[{Huang and Grizzle(2023)}]{Huang2023_TRO_planning_control}
Huang JK and Grizzle JW (2023) Efficient anytime clf reactive planning system for a bipedal robot on undulating terrain.
\newblock \emph{IEEE Transactions on Robotics} 39(3): 2093--2110.

\bibitem[{{\.I}{\c{s}}leyen et~al.(2023){\.I}{\c{s}}leyen, van~de Wouw and Arslan}]{icsleyen2023feedback}
{\.I}{\c{s}}leyen A, van~de Wouw N and Arslan {\"O} (2023) Feedback motion prediction for safe unicycle robot navigation.
\newblock In: \emph{IEEE/RSJ International Conference on Intelligent Robots and Systems (IROS)}. pp. 10511--10518.

\bibitem[{Jarvis-Wloszek et~al.(2003)Jarvis-Wloszek, Feeley, Tan, Sun and Packard}]{jarvis_clf_sos}
Jarvis-Wloszek Z, Feeley R, Tan W, Sun K and Packard A (2003) Some controls applications of sum of squares programming.
\newblock In: \emph{42nd IEEE International Conference on Decision and Control}, volume~5. pp. 4676--4681.
\newblock \doi{10.1109/CDC.2003.1272309}.

\bibitem[{Jiang and Guan(2016)}]{Jiang2016DatadrivenCC}
Jiang R and Guan Y (2016) Data-driven chance constrained stochastic program.
\newblock \emph{Mathematical Programming} 158: 291--327.

\bibitem[{Kalman(1960)}]{kalman1960new}
Kalman RE (1960) A new approach to linear filtering and prediction problems.
\newblock \emph{ASME Journal of Basic Engineering} 82: 35--45.

\bibitem[{Keyumarsi et~al.(2024)Keyumarsi, Atman and Gusrialdi}]{keyumarsi_LiDAR_CBF}
Keyumarsi S, Atman MWS and Gusrialdi A (2024) Lidar-based online control barrier function synthesis for safe navigation in unknown environments.
\newblock \emph{IEEE Robotics and Automation Letters} 9(2): 1043--1050.

\bibitem[{Khalil(2002)}]{khalil2002nonlinear}
Khalil HK (2002) \emph{Control of nonlinear systems}.
\newblock Prentice Hall, New York, NY.

\bibitem[{Khatib(1986)}]{potential-field}
Khatib O (1986) Real-time obstacle avoidance for manipulators and mobile robots.
\newblock \emph{The International Journal of Robotics Research} 5(1): 90--98.

\bibitem[{Khazoom et~al.(2022)Khazoom, Gonzalez-Diaz, Ding and Kim}]{khazoom_humanoid_2022}
Khazoom C, Gonzalez-Diaz D, Ding Y and Kim S (2022) Humanoid self-collision avoidance using whole-body control with control barrier functions.
\newblock In: \emph{IEEE International Conference on Humanoid Robots (Humanoids)}. pp. 558--565.

\bibitem[{Koenig and Howard(2004)}]{gazebo}
Koenig N and Howard A (2004) Design and use paradigms for gazebo, an open-source multi-robot simulator.
\newblock In: \emph{IEEE/RSJ International Conference on Intelligent Robots and Systems (IROS)}, volume~3. pp. 2149--2154.

\bibitem[{Kohlbrecher et~al.(2011)Kohlbrecher, von Stryk, Meyer and Klingauf}]{hector_slam_2011}
Kohlbrecher S, von Stryk O, Meyer J and Klingauf U (2011) A flexible and scalable slam system with full 3d motion estimation.
\newblock In: \emph{IEEE International Symposium on Safety, Security, and Rescue Robotics}. pp. 155--160.

\bibitem[{Kondo et~al.(2023)Kondo, Figueroa, Rached, Tordesillas, Lusk and How}]{kondo2023robust}
Kondo K, Figueroa R, Rached J, Tordesillas J, Lusk PC and How JP (2023) Robust mader: Decentralized multiagent trajectory planner robust to communication delay in dynamic environments.
\newblock \emph{IEEE Robotics and Automation Letters} .

\bibitem[{Koptev et~al.(2023)Koptev, Figueroa and Billard}]{koptev_neural_jsdf_2022}
Koptev M, Figueroa N and Billard A (2023) Neural joint space implicit signed distance functions for reactive robot manipulator control.
\newblock \emph{IEEE Robotics and Automation Letters} 8(2): 480--487.
\newblock \doi{10.1109/LRA.2022.3227860}.

\bibitem[{Lathrop et~al.(2021)Lathrop, Boardman and Mart{\'\i}nez}]{lathrop2021distributionally}
Lathrop P, Boardman B and Mart{\'\i}nez S (2021) Distributionally safe path planning: {W}asserstein safe {RRT}.
\newblock \emph{IEEE Robotics and Automation Letters} 7(1): 430--437.

\bibitem[{Li et~al.(2023)Li, Liu, Jin, Qin and Hirche}]{Li_2023_RAL}
Li J, Liu Q, Jin W, Qin J and Hirche S (2023) Robust safe learning and control in an unknown environment: An uncertainty-separated control barrier function approach.
\newblock \emph{IEEE Robotics and Automation Letters} 8(10): 6539--6546.

\bibitem[{Li et~al.(2024)Li, Zhang, Razmjoo and Calinon}]{li2024representing}
Li Y, Zhang Y, Razmjoo A and Calinon S (2024) Representing robot geometry as distance fields: Applications to whole-body manipulation.
\newblock In: \emph{Proc. IEEE Intl Conf. on Robotics and Automation (ICRA)}. pp. 15351--15357.

\bibitem[{Li et~al.(2020)Li, Arslan and Atanasov}]{li2020fast}
Li Z, Arslan {\"O} and Atanasov N (2020) Fast and safe path-following control using a state-dependent directional metric.
\newblock In: \emph{IEEE International Conference on Robotics and Automation (ICRA)}. pp. 6176--6182.

\bibitem[{Lindemann et~al.(2023)Lindemann, Cleaveland, Shim and Pappas}]{Lindemann_2022_conformal_mpc}
Lindemann L, Cleaveland M, Shim G and Pappas GJ (2023) Safe planning in dynamic environments using conformal prediction.
\newblock \emph{IEEE Robotics and Automation Letters} 8(8): 5116--5123.

\bibitem[{Liu et~al.(2023{\natexlab{a}})Liu, Adu, Lymburner, Kaushik, Trang and Vasudevan}]{liu2023radius}
Liu J, Adu CE, Lymburner L, Kaushik V, Trang L and Vasudevan R (2023{\natexlab{a}}) Radius: Risk-aware, real-time, reachability-based motion planning.
\newblock In: \emph{Robotics: Science and Systems (RSS)}.

\bibitem[{Liu et~al.(2023{\natexlab{b}})Liu, Li, Huang and Grizzle}]{liu2023realtime}
Liu J, Li M, Huang JK and Grizzle JW (2023{\natexlab{b}}) Realtime safety control for bipedal robots to avoid multiple obstacles via {CLF}-{CBF} constraints.
\newblock \emph{arXiv preprint arXiv:2301.01906} .

\bibitem[{Liu et~al.(2023{\natexlab{c}})Liu, Liu and Dolan}]{liu2023safe}
Liu S, Liu C and Dolan J (2023{\natexlab{c}}) Safe control under input limits with neural control barrier functions.
\newblock In: \emph{Conference on Robot Learning}. PMLR, pp. 1970--1980.

\bibitem[{Long et~al.(2024)Long, Cortés and Atanasov}]{long2024distributionally_policy}
Long K, Cortés J and Atanasov N (2024) Distributionally robust policy and lyapunov-certificate learning.
\newblock \emph{IEEE Open Journal of Control Systems} 3: 375--388.

\bibitem[{Long et~al.(2022)Long, Dhiman, Leok, Cortés and Atanasov}]{Long2022RAL}
Long K, Dhiman V, Leok M, Cortés J and Atanasov N (2022) Safe control synthesis with uncertain dynamics and constraints.
\newblock \emph{IEEE Robotics and Automation Letters} 7(3): 7295--7302.

\bibitem[{Long et~al.(2021)Long, Qian, Cortés and Atanasov}]{Long_learningcbf_ral21}
Long K, Qian C, Cortés J and Atanasov N (2021) Learning barrier functions with memory for robust safe navigation.
\newblock \emph{IEEE Robotics and Automation Letters} 6(3): 4931--4938.

\bibitem[{Long et~al.(2023{\natexlab{a}})Long, Yi, Cortes and Atanasov}]{long2023dro_lf}
Long K, Yi Y, Cortes J and Atanasov N (2023{\natexlab{a}}) Distributionally robust {L}yapunov function search under uncertainty.
\newblock In: \emph{Learning for Dynamics and Control Conference}. PMLR, pp. 864--877.

\bibitem[{Long et~al.(2023{\natexlab{b}})Long, Yi, Cortés and Atanasov}]{Long2023_acc_drccp}
Long K, Yi Y, Cortés J and Atanasov N (2023{\natexlab{b}}) Safe and stable control synthesis for uncertain system models via distributionally robust optimization.
\newblock In: \emph{American Control Conference (ACC)}. pp. 4651--4658.

\bibitem[{Lozano-Perez(1983)}]{lozano_1983_planning}
Lozano-Perez (1983) Spatial planning: A configuration space approach.
\newblock \emph{IEEE Transactions on Computers} C-32(2): 108--120.

\bibitem[{Majd et~al.(2021)Majd, Yaghoubi, Yamaguchi, Hoxha, Prokhorov and Fainekos}]{majd_rrt_cbf_iros}
Majd K, Yaghoubi S, Yamaguchi T, Hoxha B, Prokhorov D and Fainekos G (2021) Safe navigation in human occupied environments using sampling and control barrier functions.
\newblock In: \emph{IEEE/RSJ International Conference on Intelligent Robots and Systems (IROS)}. pp. 5794--5800.

\bibitem[{Mestres et~al.(2023)Mestres, Allibhoy and Cort\'es}]{PM-AA-JC:23-scl}
Mestres P, Allibhoy A and Cort\'es J (2023) Robinson’s counterexample and regularity properties of optimization-based controllers.
\newblock \emph{arXiv preprint arXiv:2311.13167} .

\bibitem[{Mestres et~al.(2024{\natexlab{a}})Mestres, Long, Atanasov and Cortés}]{PM-KL-NA-JC:23-csl}
Mestres P, Long K, Atanasov N and Cortés J (2024{\natexlab{a}}) Feasibility analysis and regularity characterization of distributionally robust safe stabilizing controllers.
\newblock \emph{IEEE Control Systems Letters} 8: 91--96.

\bibitem[{Mestres et~al.(2024{\natexlab{b}})Mestres, Nieto-Granda and Cort{\'e}s}]{mestres2024safe}
Mestres P, Nieto-Granda C and Cort{\'e}s J (2024{\natexlab{b}}) Safe and dynamically-feasible motion planning using control lyapunov and barrier functions.
\newblock \emph{arXiv preprint arXiv:2410.08364} .

\bibitem[{Moussa{\"\i}d et~al.(2010)Moussa{\"\i}d, Perozo, Garnier, Helbing and Theraulaz}]{moussaid2010walking}
Moussa{\"\i}d M, Perozo N, Garnier S, Helbing D and Theraulaz G (2010) The walking behaviour of pedestrian social groups and its impact on crowd dynamics.
\newblock \emph{PloS one} 5(4): e10047.

\bibitem[{Nemirovski and Shapiro(2006)}]{Nemirovski2006ConvexAO}
Nemirovski A and Shapiro A (2006) Convex approximations of chance constrained programs.
\newblock \emph{SIAM J. Optim.} 17: 969--996.

\bibitem[{Oleynikova et~al.(2017)Oleynikova, Taylor, Fehr, Siegwart and Nieto}]{voxblox_Oleynikova}
Oleynikova H, Taylor Z, Fehr M, Siegwart R and Nieto J (2017) Voxblox: Incremental 3d {E}uclidean signed distance fields for on-board {MAV} planning.
\newblock In: \emph{2017 IEEE/RSJ International Conference on Intelligent Robots and Systems (IROS)}. pp. 1366--1373.

\bibitem[{Park et~al.(2019)Park, Florence, Straub, Newcombe and Lovegrove}]{deepsdf}
Park JJ, Florence P, Straub J, Newcombe R and Lovegrove S (2019) {Deep{SDF}: Learning Continuous Signed Distance Functions for Shape Representation}.
\newblock In: \emph{Proceedings of the IEEE/CVF conference on computer vision and pattern recognition}. pp. 165--174.

\bibitem[{Parwana et~al.(2022)Parwana, Mustafa and Panagou}]{hardik_2022_rate_cbf}
Parwana H, Mustafa A and Panagou D (2022) Trust-based rate-tunable control barrier functions for non-cooperative multi-agent systems.
\newblock In: \emph{2022 IEEE 61st Conference on Decision and Control (CDC)}. pp. 2222--2229.

\bibitem[{Pfeiffer et~al.(2018)Pfeiffer, Shukla, Turchetta, Cadena, Krause, Siegwart and Nieto}]{pfeiffer2018reinforced}
Pfeiffer M, Shukla S, Turchetta M, Cadena C, Krause A, Siegwart R and Nieto J (2018) Reinforced imitation: Sample efficient deep reinforcement learning for mapless navigation by leveraging prior demonstrations.
\newblock \emph{IEEE Robotics and Automation Letters} 3(4): 4423--4430.

\bibitem[{Prajna and Jadbabaie(2004)}]{prajna2004safety}
Prajna S and Jadbabaie A (2004) Safety verification of hybrid systems using barrier certificates.
\newblock In: \emph{International Workshop on Hybrid Systems: Computation and Control}. Springer, pp. 477--492.

\bibitem[{Ren and Majumdar(2022)}]{ren2022distributionally_ral}
Ren AZ and Majumdar A (2022) Distributionally robust policy learning via adversarial environment generation.
\newblock \emph{IEEE Robotics and Automation Letters} 7(2): 1379--1386.

\bibitem[{{Rimon} and {Koditschek}(1992)}]{navigation-function}
{Rimon} E and {Koditschek} DE (1992) Exact robot navigation using artificial potential functions.
\newblock \emph{IEEE Transactions on Robotics and Automation} 8(5): 501--518.

\bibitem[{Rockafellar and Uryasev(2000)}]{Rockafellar00optimizationof}
Rockafellar RT and Uryasev S (2000) Optimization of conditional value-at-risk.
\newblock \emph{Journal of Risk} 2: 21--41.

\bibitem[{Rockafellar and Wets(2009)}]{rockafellar2009variational}
Rockafellar RT and Wets RJB (2009) \emph{Variational analysis}, volume 317.
\newblock Springer Science \& Business Media.

\bibitem[{Ryu and Mehr(2024)}]{Ryu_2024_icra_mpc_dro}
Ryu K and Mehr N (2024) Integrating predictive motion uncertainties with distributionally robust risk-aware control for safe robot navigation in crowds.
\newblock In: \emph{2024 IEEE International Conference on Robotics and Automation (ICRA)}. pp. 2410--2417.

\bibitem[{Schaefer et~al.(2021)Schaefer, Leung, Ivanovic and Pavone}]{schaefer2021leveraging}
Schaefer S, Leung K, Ivanovic B and Pavone M (2021) Leveraging neural network gradients within trajectory optimization for proactive human-robot interactions.
\newblock In: \emph{2021 IEEE International Conference on Robotics and Automation (ICRA)}. IEEE, pp. 9673--9679.

\bibitem[{Shafer and Vovk(2008)}]{shafer2008tutorial}
Shafer G and Vovk V (2008) A tutorial on conformal prediction.
\newblock \emph{Journal of Machine Learning Research} 9(3).

\bibitem[{Sontag(1989)}]{SONTAG1989117}
Sontag ED (1989) {A ‘universal’ construction of {A}rtstein's theorem on nonlinear stabilization}.
\newblock \emph{Systems \& Control Letters} 13(2): 117--123.

\bibitem[{Summers(2018)}]{summers2018_dr_rrt}
Summers T (2018) Distributionally robust sampling-based motion planning under uncertainty.
\newblock In: \emph{IEEE/RSJ International Conference on Intelligent Robots and Systems (IROS)}. pp. 6518--6523.

\bibitem[{Van~Parys et~al.(2016)Van~Parys, Kuhn, Goulart and Morari}]{Parys2015monent}
Van~Parys BPG, Kuhn D, Goulart PJ and Morari M (2016) Distributionally robust control of constrained stochastic systems.
\newblock \emph{IEEE Transactions on Automatic Control} 61(2): 430--442.

\bibitem[{Wang and Xu(2023)}]{Wang_2023_acc_dob}
Wang Y and Xu X (2023) Disturbance observer-based robust control barrier functions.
\newblock In: \emph{American Control Conference (ACC)}. pp. 3681--3687.

\bibitem[{Wu et~al.(2023)Wu, Lee, Le~Gentil and Vidal-Calleja}]{wu_2023_tro}
Wu L, Lee KMB, Le~Gentil C and Vidal-Calleja T (2023) Log-gpis-mop: A unified representation for mapping, odometry, and planning.
\newblock \emph{IEEE Transactions on Robotics} 39(5): 4078--4094.

\bibitem[{Wu et~al.(2021)Wu, Lee, Liu and Vidal-Calleja}]{log-gpis}
Wu L, Lee KMB, Liu L and Vidal-Calleja T (2021) Faithful {E}uclidean distance field from log-gaussian process implicit surfaces.
\newblock \emph{IEEE Robotics and Automation Letters} 6(2): 2461--2468.

\bibitem[{Xiao et~al.(2022)Xiao, Wang, Chahine, Amini, Hasani and Rus}]{xiao2022differentiable}
Xiao W, Wang TH, Chahine M, Amini A, Hasani R and Rus D (2022) Differentiable control barrier functions for vision-based end-to-end autonomous driving.
\newblock \emph{arXiv preprint arXiv:2203.02401} .

\bibitem[{Xie(2021)}]{Xie2021OnDR}
Xie W (2021) On distributionally robust chance constrained programs with {W}asserstein distance.
\newblock \emph{Math. Program.} 186: 115--155.

\bibitem[{Yang et~al.(2023)Yang, Pappas, Mangharam and Lindemann}]{yang2023safe}
Yang S, Pappas GJ, Mangharam R and Lindemann L (2023) Safe perception-based control under stochastic sensor uncertainty using conformal prediction.
\newblock In: \emph{IEEE Conference on Decision and Control (CDC)}. pp. 6072--6078.

\bibitem[{Yang et~al.(2022)Yang, Gong, Huang and Hong}]{Lidar_velocity_estimate}
Yang W, Gong Z, Huang B and Hong X (2022) Lidar with velocity: Correcting moving objects point cloud distortion from oscillating scanning lidars by fusion with camera.
\newblock \emph{IEEE Robotics and Automation Letters} 7(3): 8241--8248.

\bibitem[{Yu et~al.(2023)Yu, Hirayama, Yu, Herbert and Gao}]{Yu_sequential_CBF_2023}
Yu H, Hirayama C, Yu C, Herbert S and Gao S (2023) Sequential neural barriers for scalable dynamic obstacle avoidance.
\newblock In: \emph{IEEE/RSJ International Conference on Intelligent Robots and Systems (IROS)}. pp. 11241--11248.

\bibitem[{Zhang et~al.(2023)Zhang, Garg and Fan}]{zhang2023neural}
Zhang S, Garg K and Fan C (2023) Neural graph control barrier functions guided distributed collision-avoidance multi-agent control.
\newblock In: \emph{Conference on Robot Learning (CoRL)}. PMLR, pp. 2373--2392.

\bibitem[{Zhang et~al.(2024)Zhang, Tian, Wen, Yao, Zhang, Bing, He and Knoll}]{zhang2024online}
Zhang Y, Tian G, Wen L, Yao X, Zhang L, Bing Z, He W and Knoll A (2024) Online efficient safety-critical control for mobile robots in unknown dynamic multi-obstacle environments.
\newblock \emph{arXiv preprint arXiv:2402.16449} .

\bibitem[{Zhao et~al.(2024)Zhao, Yu, Deshmukh and Lindemann}]{zhao2024conformal}
Zhao Y, Yu X, Deshmukh JV and Lindemann L (2024) Conformal predictive programming for chance constrained optimization.
\newblock \emph{arXiv preprint arXiv:2402.07407} .

\end{thebibliography}

\end{document}